\documentclass[11pt]{article} %
\pdfoutput=1
\usepackage[T1]{fontenc}

\usepackage{lmodern} \normalfont %
\usepackage{anyfontsize}
\DeclareFontShape{T1}{lmr}{bx}{sc} { <-> ssub * cmr/bx/sc }{}
\DeclareFontShape{T1}{lmr}{m}{scit}{ <-> ssub * cmr/m/sc }{}
\DeclareFontShape{T1}{lmr}{bx}{scit}{ <-> ssub * cmr/bx/sc }{}

\usepackage{etoolbox}

\usepackage{url}            %
\usepackage{booktabs}       %
\usepackage{amsfonts}       %
\usepackage{nicefrac}  
\usepackage{microtype}
\usepackage{graphicx}
\usepackage{booktabs} %
\usepackage{subcaption}
\usepackage{lipsum}
\usepackage{tkz-fct}
\pdfoutput=1
\usepackage{amsmath,amsthm,wrapfig}
	\usepackage[utf8]{inputenc} 
\usepackage{amsmath, amssymb,bm, cases, mathtools, thmtools}
\usepackage{verbatim}
\usepackage{graphicx}\graphicspath{{figures/}}
\usepackage{multicol}
\usepackage{tabularx}
\usepackage{mathrsfs} 
\usepackage{algorithm}
\usepackage{algorithmicx}
\usepackage[noend]{algpseudocode}
\usepackage{array,booktabs,arydshln}
\usepackage{xcolor}
\PassOptionsToPackage{table,xcdraw}{xcolor}
\usepackage[%
		minnames=4,maxnames=5,maxbibnames=5,minalphanames=4,maxalphanames=5,
		style=alphabetic,
		sorting=anyt,
		sortcites=false, %
		doi=false,url=false,
		uniquename=init,
		giveninits=true,
		hyperref,natbib,
		backend=biber]{biblatex}
\renewbibmacro{in:}{%
	\ifentrytype{article}{}{\printtext{\bibstring{in}\intitlepunct}}}

\usepackage[T1]{fontenc}
\usepackage{inconsolata}
\usepackage{dsfont}
\usepackage{hyperref}
\usepackage{listings} %

\usepackage{enumitem}
\usepackage{booktabs}       %
\usepackage{nicefrac}       %

\usepackage{lineno}
\usepackage{bbm}

\usepackage{datetime}

\DeclareMathAlphabet\EuRoman{U}{eur}{m}{n}
\SetMathAlphabet\EuRoman{bold}{U}{eur}{b}{n}

\declaretheorem[style=plain,numberwithin=section,name=Theorem]{theorem}
\declaretheorem[style=plain,sibling=theorem,name=Lemma]{lemma}

\declaretheorem[style=definition,sibling=theorem,name=Definition]{definition}

\declaretheorem[style=remark,qed=$\triangleleft$,sibling=theorem,name=Remark]{remark}
\numberwithin{theorem}{section}

\usepackage{xparse}
\usepackage{xstring}
\usepackage{xspace}

\def\[#1\]{\begin{align}#1\end{align}}
\def\*[#1\]{\begin{align*}#1\end{align*}}

\newcommand{\trainset}{S_n}

\newcommand{\parspace}{\mathcal{W}}

\newcommand{\dataspace}{\mathcal Z}
\newcommand\optparen[1]{\ifthenelse{\equal{#1}{}}{}{(#1)}}

\newcommand{\Naturals}{\mathbb{N}}

\newcommand{\Reals}{\mathbb{R}}

\newcommand{\grad}{\nabla}
\newcommand{\imp}{\Rightarrow}

\DeclareMathOperator*{\newlim}{\mathrm{lim}\vphantom{\mathrm{infsup}}}
\DeclareMathOperator*{\newmin}{\mathrm{min}\vphantom{\mathrm{infsup}}}
\DeclareMathOperator*{\newmax}{\mathrm{max}\vphantom{\mathrm{infsup}}}

\DeclareMathOperator*{\newsup}{\mathrm{sup}\vphantom{\mathrm{infsup}}}
\renewcommand{\lim}{\newlim}
\renewcommand{\min}{\newmin}
\renewcommand{\max}{\newmax}

\renewcommand{\sup}{\newsup}

\newcommand{\ProbMeasures}[1]{\mathcal{M}_1(#1)}

\renewcommand{\Pr}{\mathbb{P}}
\def\EE{\mathbb{E}}

\newcommand{\norm}[1]{\left\lVert #1 \right\rVert}

\newcommand{\iid}{i.i.d.}

\newcommand{\Normal}{\mathcal N}

\newcommand{\loss}{\ell}
\newcommand{\e}{\mathrm{e}}

\newcommand{\Alg}{\mathcal{A}}

\newcommand{\lcrx}[4][{-1}]{
	\IfEq{#1}{-1}{\left #2 {{{{#3}}}} \right #4}{
   	\IfEq{#1}{0}{#2 {{{{#3}}}} #4}{
	\IfEq{#1}{1}{\bigl #2 {{{{#3}}}} \bigr #4}{
	\IfEq{#1}{2}{\Bigl #2 {{{{#3}}}} \Bigr #4}{
	\IfEq{#1}{3}{\biggl #2 {{{{#3}}}} \biggr #4}{
	\IfEq{#1}{4}{\Biggl #2 {{{{#3}}}} \Biggr #4}{
    \GenericWarning{"4th argument to lcrx must be -1, 0, 1, 2, 3, or 4"}
    }}}}}}}
\newcommand{\inner}[3][{-1}]{\lcrx[#1] < {{#2},{#3}} >}
\newcommand{\indep}{\mathrel{\perp\mkern-9mu\perp}}

\newcommand{\range}[1]{ [#1] }

\DeclareMathOperator*{\argmin}{arg\,min}

\newcommand{\dataset}{S_n}
\newcommand{\gradl}{\nabla \ell}

\newcommand{\proj}[0]{\Pi_{\parspace}} %
\newcommand{\ww}{w}

\newcommand{\lipf}{\mathtt{L}_0}
\newcommand{\lipg}{\mathtt{L}_1}
\newcommand{\liph}{\mathtt{L}_2}
\newcommand{\logf}{f_{\mathrm{LL}}}
\newcommand{\soi}{SOI\xspace}
\newcommand{\losslog}{\loss_{\mathrm{LL}}}
\newcommand{\diam}{D} %

\newcommand{\rank}{\xspace\mathsf{rank}\xspace}

\newcommand{\hessc}{\ensuremath{\mathsf{Hess}\text{-}\mathsf{clip}}\xspace}
\newcommand{\hessa}{\ensuremath{\mathsf{Hess}\text{-}\mathsf{add}}\xspace}
\newcommand{\quc}{\ensuremath{\mathsf{QU}\text{-}\mathsf{clip}}\xspace}
\newcommand{\qua}{\ensuremath{\mathsf{QU}\text{-}\mathsf{add}}\xspace}

\usepackage{makecell}
\usepackage{arydshln}
\usepackage[capitalize,noabbrev]{cleveref}
\usepackage{wrapfig}

\title{Faster Differentially Private Convex Optimization \\via Second-Order Methods}
\author{Arun Ganesh\thanks{Google,  \href{mailto:arunganesh@google.com}{arunganesh@google.com}} 
\and 
Mahdi Haghifam\thanks{University of Toronto, Vector Institute; Part of this work was done while the author was an intern at Google Research--Brain Team,  \href{mailto:mahdi.haghifam@mail.utoronto.ca}{mahdi.haghifam@mail.utoronto.ca}} 
\and
Thomas Steinke\thanks{Google, \href{mailto:steinke@google.com}{steinke@google.com}}
\and 
Abhradeep  Thakurta\thanks{Google, \href{mailto:athakurta@google.com}{athakurta@google.com}}}
\date{}

\usepackage{titlesec}
\titlespacing{\section}{0pt}{\parskip}{0pt}
\titlespacing{\subsection}{0pt}{\parskip}{0pt}
\titlespacing{\subsubsection}{0pt}{\parskip}{0pt}
\usepackage[letterpaper, margin=1in]{geometry}

\renewcommand{\epsilon}{\varepsilon}

\begin{document}

\maketitle

\footnotetext{Authors ordered alphabetically.}

\begin{abstract}
    Differentially private (stochastic) gradient descent is the workhorse of differentially private machine learning in both the convex and non-convex settings. Without privacy constraints, second-order methods, like Newton's method, converge faster than first-order methods like gradient descent.   
    In this work, we investigate the prospect of using the second-order information of loss function to accelerate differentially private convex optimization. 
    We first develop a private variant of the regularized cubic Newton method of \citet{nesterov2006cubic} for the class of strongly convex loss functions. We show that our algorithm achieves the optimal excess loss and attains the same (optimal) rate of convergence as its non-private counterparts.  We then design a practical second-order DP algorithm for the unconstrained logistic regression problem. %
    We empirically study the performance of our algorithm. We show that our algorithm almost always achieves the best excess loss for a wide range of $\varepsilon \in [0.01,10]$ on many challenging datasets. Furthermore, the run-time of our algorithm is $10\times$-$40\times$ faster than DPGD. 
\end{abstract}

\section{Introduction}
\label{sec:intro}

Many machine learning tasks reduce to a convex optimization problem. More precisely, given a dataset $\trainset=(z_1,\dots,z_n)\in\dataspace^n$, a closed,  convex
set $\parspace$,  and a loss function $f: \parspace \times \dataspace \to \Reals$ such that, for every $z\in \dataspace$, $f(\ww,z)$ is a convex function in $\ww$, our goal is to compute an approximation to
$\argmin_{\ww \in \parspace}~\left(\loss(\ww,\trainset)\triangleq \frac1n \sum_{i \in \range{n}}f(\ww,z_i)\right).$
In this paper, we are interested in the problem of designing optimization algorithms in the scenario that the dataset $\trainset$ contains private information. Differential privacy (DP) \citep{dwork2006calibrating} is a formal standard for privacy-preserving data analysis that provides a framework for ensuring that the output of an analysis on the data does not leak this private information. This problem is known as \emph{private convex optimization}: We want an algorithm $\Alg : \dataspace^n \to \parspace$ that is both DP and ensures low \emph{excess loss} $\triangleq \loss(\Alg(\trainset),\trainset) - \min_{\ww \in \parspace}\loss(\ww,\trainset)$. 

The predominant algorithm for private convex optimization is DP (stochastic) gradient descent (DP-GD/DP-SGD). This is a \emph{first-order} iterative method. I.e., we start with an initial value $\ww_0$ and iteratively update it using the gradient of the loss $ \nabla_{\ww_t} \loss(\ww_t,\trainset) $ following the update rule $\ww_{t+1} \!=\! \ww_t \!-\! \eta \!\cdot\! \left( \nabla_{\ww_t} \loss(\ww_t,\!\trainset) \!+\! \xi_t \right)\!,\label{eq:gd}$ where $\eta>0$ is a constant and $\xi_t$ is Gaussian noise to ensure privacy. The number of iterations $T$ also determines the amount of noise at each iteration, i.e., the scale of $\xi_t$ is proportional to $\sqrt{T}$ due to the composition of DP. %
Note that we assume $\|\nabla_{\ww_t} \loss(\ww_t,\trainset)\|\le1$.

One of the major drawbacks of DP-(S)GD is \emph{slow convergence}. We argue that the main reason for this is the difficulty of choosing the hyperparameters $(\eta,T)$.
The choice of $(\eta,T)$ exhibits a tradeoff in terms of the excess loss: if $\eta \cdot T$ is small, the algorithm cannot reach the optimal solution; on the other hand, the magnitude of noise at each iteration is $\eta \cdot \sqrt{T}$, which cannot be too large. %
Therefore, to maximize $\eta \cdot T$ and minimize $\eta \cdot \sqrt{T}$, implementations of DP-(S)GD err on the side of large $T$ and small $\eta$, which results in a long, slow path to convergence. This slowness is exacerbated by the facts that (1) DP-SGD requires large batch sizes for good performance \cite{ponomareva2023dp} and (2) the hyperparameter tuning of DP-(S)GD, and generally DP algorithms, is a challenging task \citep{papernot2022hyperparameter}. 
\emph{Can we design a DP optimization algorithm which accelerates DP-(S)GD by choosing the step size dynamically?}

We draw inspiration from the non-private optimization literature: To address the slow convergence of GD and of first-order methods in general, a class of algorithms based on \emph{preconditioning} the gradient using second-order information has been developed \citep{nesterov1998introductory,nocedal1999numerical}. This class of algorithms is based on successively minimizing a quadratic \emph{approximation} of the function, i.e.,  $\ww_{t+1} = \ww_t + \Delta_t$ where $\Delta_t = \argmin_{\Delta} \{ \loss(\ww_t,\trainset) + \inner{\nabla \loss(\ww_t,\trainset)}{\Delta} + \frac{1}{2}\inner{H_t\cdot \Delta}{\Delta} \} = - \left(H_t\right)^{-1}\nabla \loss(\ww_t,\trainset)$. Here, $H_t$ is a scaling matrix which provides curvature information about the loss $\loss(\cdot,\trainset)$ at $\ww_t$. For instance, Newton's method uses the Hessian $H_t = \nabla^2 \loss(\ww_t,\trainset)$. Second-order algorithms significantly improve over the convergence speed of GD, and key to their success is that at each step they \emph{automatically} tune the stepsize along each dimension based on the local curvature. 

In this paper, our goal is to accelerate DP convex optimization. In particular, the current paper revolves around the following questions: Can the second-order information \emph{accelerate} private convex optimization while achieving \emph{optimal excess error}? What is the best way to \emph{privatize second-order information}, e.g., the Hessian matrix? How does the achievable \emph{privacy-utility-runtime tradeoff} compare with first-order methods such as DP-GD? 
We show that second-order information can accelerate DP optimization while achieving excess loss that matches or improves on DP-GD. 
Our main contributions are both theoretical and empirical:

\subsection{Provably Optimal Algorithm for Strongly Convex Functions}
Newton's method is a second-order optimization technique that is well-known for its rapid convergence for strongly convex and smooth functions in non-private optimization. Specifically, to achieve an excess loss of $\alpha$, the method only requires $O(\log\log(1/\alpha))$ iterations, which is provably faster than the convergence rate of \emph{any} first-order method. %
One natural question is whether it is possible to design a second-order DP convex optimization algorithm that can achieve the \emph{optimal minmax} excess error $\mathrm{err}^\text{opt}$ in $O(\log\log(1/\mathrm{err}^\text{opt}))$ iterations?  We provide an affirmative answer to this question in \cref{sec:cubic-newton}  by designing a second-order DP algorithm based on the cubic regularized Newton's method of \citet{nesterov2006cubic}.
At each step $t$, we compute a \emph{cubic} upper bound %
$
\!\loss(\ww\!+\!\Delta,\!S_n) \! \le \!  \loss(\ww,S_n) \!+\! \left\langle \nabla_{\ww} \loss(\ww,S_n) , \Delta \right\rangle  + \frac12 \!\left\langle \nabla_{\ww}^2 \loss(\ww,S_n) \!\cdot\! \Delta ,\! \Delta \right\rangle \!+\! O\!\left(\|\Delta\|^3\right).
$
We can minimize this cubic upper bound using \emph{any} DP convex optimization subroutine; the minimizer becomes the next iterate $\ww_{t+1}$. 
Since the cubic is a universal upper bound, our algorithm converges globally and the second-order information ensures that it does not require any stepsize tuning.%

\subsection{Fast Practical Algorithms for DP Logistic Regression}

DP logistic regression is a popular approach for private classification, with DP-GD/DP-SGD being the predominant class of algorithms for this task. As we numerically show, DP-GD/DP-SGD exhibit slow convergence for this task (See \cref{fig:intro-comparison}).
In \cref{sec:practical-alg}, we develop a practical algorithm that injects carefully designed noise into Newton's update rule as follows:
\[
\ww_{\!t\!+\!1} \! = \! \ww_t \! - \! \Psi\!\left(\nabla_{\ww_t}^2 \!\loss(\ww_t,\!\trainset)\right)^{\!-1\!} \!\!\cdot\! \left( \nabla_{\ww_t} \! \loss(\ww_t,\!\trainset) \!+\! \xi_{t,1}\!\right) \!+\! \xi_{t,2} . \label{eq:ours-hessian}
\]
\begin{wrapfigure}{r}{0.4\textwidth}
\includegraphics[width=1\linewidth]{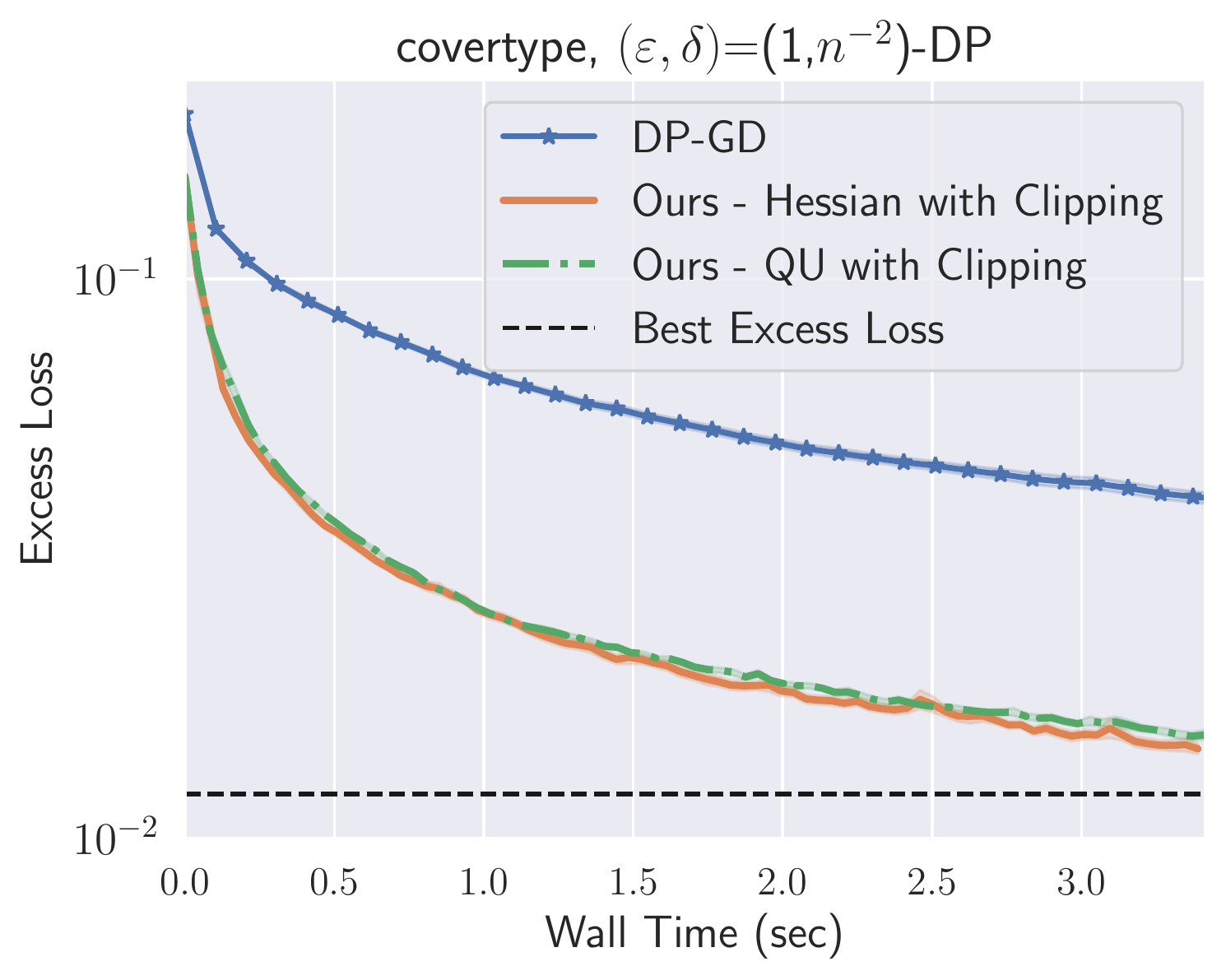}
  \caption{Excess loss versus runtime of DP-GD \& our algorithms.}
\label{fig:intro-comparison}
\end{wrapfigure}
In particular, we inject noise twice: $\xi_{t,1}$ privatizes the gradient and $\xi_{t,2}$ privatizes the direction. 
The function $\Psi$ modifies the Hessian to ensure that the eigenvalues are not too small; this is essential for bounding the sensitivity and, hence, the scale of $\xi_{t,2}$. We consider two types of modification based on \emph{eigenvalue clipping} and \emph{eigenvalue adding}. 
For eigenvalue clipping,  $\Psi(\nabla_{\ww_t}^2 \!\loss(\ww_t,\!\trainset))$  replaces the eigenvalues $\lambda_i$ of $\nabla_{\ww_t}^2 \!\loss(\ww_t,\!\trainset)$ with $\max\{\lambda_i,\lambda_0\}$, where $\lambda_0>0$ is a carefully chosen constant. For eigenvalue adding, $\Psi(\nabla_{\ww_t}^2 \!\loss(\ww_t,\!\trainset)) = \nabla_{\ww_t}^2 \!\loss(\ww_t,\!\trainset) + \lambda_0 I$. Using $\Psi$ we can control the sensitivity and still have fast convergence, since important curvature information is generally contained in the larger eigenvalues/vectors of the Hessian. We prove the local convergence of the update rule \eqref{eq:ours-hessian} in \cref{subsec:convergence-results-practical} and perform a thorough empirical evaluation \cref{sec:numerical-results}. We demonstrate that our algorithm outperforms existing baselines on a variety of benchmarks. 

\textbf{Ensuring Global Convergence.}
One limitation of the update rule in \cref{eq:ours-hessian} is it does not converge globally (even without noise added for DP). That is, if the initial point $\ww_0$ is too far from the optimal solution, then the iterates may diverge. 
To address this problem, we propose a variant of Newton's update rule where we replace the Hessian with a different form of second-order information which gives a \emph{Quadratic Upperbound} (QU) on the logistic loss. This is \emph{guaranteed to converge globally}, like the cubic Newton approach. And we show numerically that this algorithm converges almost as fast as the regular Newton's method in the private setting. \cref{fig:intro-comparison}  shows the convergence speed of our algorithms and DP-GD in terms of real wall time for the task of logistic regression on the Covertype dataset for $(\varepsilon,\delta)=(1,(\text{num. samples})^{-2})$-DP. Despite DP-GD having a lower per-iteration cost, our algorithm is $30\times$ faster than DP-GD and achieves better excess loss.

\textbf{Stochastic Minibatch Variant.} We also show that our algorithms naturally extend to the minibatch setting where gradient and second-order information are computed on a subset of samples. We numerically compare it with DP-SGD and show that it has faster convergence.   

\section{Related Work}
\label{subsec:related-work}
DP optimization is a well-studied topic \cite[e.g.,][]{song2013stochastic,mcmahan2017learning,abadi2016deep, smith2017interaction, WLKCJN17, iyengar2019towards, song2020characterizing, song2021evading,ganesh2022langevin,gopi22a,bassily2019private,bassily2014private}.
Most similar to our work, \citet{Medina21} consider second-order methods for DP convex optimization. We provide a detailed comparison between our results and theirs in \cref{rem:compare-damped-scvx} and \cref{sec:numerical-results} showing that our algorithms relax restrictive assumptions and provide better excess error for logistic regression.

\section{Preliminaries}
\label{sec:setup}
We use standard notation for linear algebra (see \cref{appx:standard-notations} for completeness).
Let $\dataspace$ be the data and let $\parspace \subseteq \Reals^d$ be the parameter space. Let $f: \parspace \times \dataspace \to \Reals  $ be a loss function. Throughout the paper, we assume $f$ is doubly continuous, a convex function in $w$, and  $\parspace$ is a closed and convex set. We say (1) $f$ is \emph{$\lipf$-Lipschitz} iff there exists $\lipf \in \Reals$ such that $\forall z \in \dataspace$, $\forall w,v \in \parspace: |f(w,z)-f(v,z)|\leq \lipf \norm{w-v}$,
    (2) $f$ is \emph{$\lipg$-smooth} iff there exists $\lipg \in \Reals$ such that $\forall z \in \dataspace$, $\forall w,v \in \parspace: \norm{\nabla f(w,z)-\nabla f(v,z)}\leq \lipg \norm{w-v}$,
    (3) $f$ has a \emph{$\liph$-Lipschitz Hessian} iff there exists $\liph\in \Reals $ such that  $\forall z \in \dataspace$, $\forall w,v \in \parspace: \norm{\nabla^2 f(w,z)-\nabla^2 f(v,z)}\leq \liph \norm{w-v}$,
    (4) $f$ is \emph{$\mu$-strongly convex} iff for all $w,v \in \parspace$ and $z\in \dataspace$ we have $f(v,z) \geq f(w,z) + \inner{\nabla f(w,z)}{v-w} + \frac{\mu}{2}\norm{v-w}^2$. For our privacy analysis, we use concentrated differential privacy \cite{dwork2016concentrated,bun2016concentrated}, as it provides a simpler composition theorem -- the privacy parameter $\rho$ adds up when we compose. %
\begin{definition}[{\citealp[][Def.~1.1]{bun2016concentrated}}]\nocite{dwork2016concentrated}
A randomized mechanism $\Alg:\dataspace^n \to \ProbMeasures{\mathcal{R}}$ is $\rho$-zCDP, iff, for every neighbouring dataset (i.e., addition or removal) $\trainset \in \dataspace^n$ and $\trainset' \in \dataspace^n$, and for every $\alpha \in (1,\infty)$, it holds $\displaystyle \mathrm{D}_\alpha(\Alg(\trainset) \| \Alg(\trainset') )\leq \rho \alpha,$ where $\mathrm{D}_\alpha(\Alg_n(\trainset) \| \Alg_n(\trainset') )$ is the $\alpha$-Renyi divergence between $\Alg_n(\trainset)$ and $\Alg_n(\trainset')$. %
\end{definition}

We should think of $\rho \approx \varepsilon^2$: to attain $(\varepsilon,\delta)$-DP, it suffices to set $\rho=\frac{\varepsilon^2}{4\log(1/\delta) + 4\varepsilon}$ \citealp[][Lem.~3.5]{bun2016concentrated}.
\section{Optimal Algorithm for the Class of Strongly Convex Functions}
\label{sec:cubic-newton}
In this section, we present a DP variant of the cubic-regularized Newton's method of \citet{nesterov2006cubic}. To motivate the idea behind our algorithm, we revisit DP gradient descent (DP-GD) for the class of $\lipf$-Lipschitz and $\lipg$-smooth convex loss functions. 

Let $\{\ww_t^{\text{GD}}\}_{t\in \range{T}}$ be the iterates of DP-GD. The smoothness of $\loss$ lets us construct a global quadratic upper bound on the function  \citep[Thm.~2.1.5]{nesterov1998introductory} as follows $\forall \ww \in \parspace$ and $\trainset \in \dataspace^n$ :
\begin{align}
 &   \loss(\ww,\trainset) \leq q_t(\ww) \triangleq  \loss(\ww_{t}^{\text{GD}},\trainset) + \inner{\nabla \loss(\ww_{t}^{\text{GD}},\trainset)}{\ww-\ww_t^{\text{GD}}}  +\frac{\lipg}{2}\norm{\ww-\ww_t^{\text{GD}}}^2 . \label{eq:quadratic-ub}
\end{align}
Then, DP-GD can be seen as a two-step process:
\[
\nonumber
 \text{(Step I)} \!&&\!v_{t+1} \!=\!  \argmin_{v }  q_t(v) \!=\! \ww_t^{\text{GD}} \!-\! \lipg^{-1} \nabla \loss(\ww_{t}^{\text{GD}},\trainset),\quad \text{(Step II)} \!&&\!\ww_{t+1}^{\text{GD}} \!=\! \proj(v_{t+1} + \lipg^{-1} \xi_t),
\]
where  $\xi_t = \Normal(0,\sigma^2 I_d)$ with $\sigma^2 = \frac{\lipf^2}{2\rho n^2}$ so that $\ww_{t+1}^{\text{GD}}$ satisfies $\rho$-zCDP \citealp[][Lem.~2.5]{bun2016concentrated}. That is, in each iteration of DP-GD, \emph{we find a minimum of the quadratic upper bound} $q_t(w)$ and then project back to $\parspace$. (In the unconstrained setting where $\parspace = \Reals^d$ we do not need the second projection step.)

Consider the class of $\liph$-Lipschitz Hessian convex loss functions. \citet[Lem.~1]{nesterov2006cubic} show that we can construct a  \emph{global cubic upper bound} exploiting the second-order information (i.e., Hessian) as follows: for all $\ww$ and $\ww_t$, $\loss(\ww,\trainset) \!\le\! \phi_t(\ww)$ where
\begin{align}
\small
    &\phi_t(\ww) \!\triangleq\! \loss(\ww_{t},\!\trainset) \!+\! \inner{\nabla \loss(\ww_{t},\trainset)}{\ww\!-\!\ww_t}  + \!\frac{1}{2}\!\inner{\nabla^2\! \loss(\ww_t,\!\trainset)(\ww\!-\!\ww_t)}{\!\ww\!-\!\ww_t} \!+\! \frac{\liph}{6}\!\norm{\ww\!-\!\ww_t}^3   \label{eq:cubic-upperbound-main}.
\end{align}
Their non-private algorithm is based on the \emph{exact} minimization of $\phi_t(\ww)$, i.e., the next iterate is $\ww_{t+1}=\argmin \phi_t(\ww)$. Note that $\argmin \phi_t(\ww)$ does not admit a closed form solution, as opposed to the quadratic upper bound \eqref{eq:quadratic-ub}. Similar to the intuition for DP-GD on smooth loss functions \eqref{eq:quadratic-ub}, our algorithms in this section are based on \emph{privately} minimizing $\phi_t(\ww)$ at each iteration. Our algorithm is shown in \cref{alg:cubic-private}. In each iteration the algorithm makes an oracle call to obtain $(\loss(\ww_t,\trainset),\nabla\loss(\ww_t,\trainset),\nabla^2 \loss(\ww_t,\trainset))$. Then,  the algorithm calls an efficient $\mathsf{DPSolver}$ for privately optimizing the cubic upper bound \eqref{eq:cubic-upperbound-main}. %
The privacy analysis of \cref{alg:cubic-private} is a direct application of the composition property of zCDP \citep[Lemma 2.3]{bun2016concentrated}; the output of $\mathsf{DPSolver}$ at each iteration satisfies $\rho/T$-zCDP where $\rho$ is the total privacy budget and $T$ is the total number of iterations.
\begin{remark}
$\mathsf{DPSolver}$ in \cref{alg:cubic-private} does not affect the \emph{oracle complexity} of \cref{alg:cubic-private}, 
as it is applied to the proxy loss $\phi_t(\ww)$, rather than the underlying loss $\loss(\ww,\trainset)$.  
\end{remark}

\vspace{-2em}
\begin{minipage}{0.46\textwidth}
  \begin{algorithm}[H]
\small
\caption{Meta Algorithm}\label{alg:cubic-private}
\begin{algorithmic}[1]
\State Input: training set $\trainset \in \mathcal{Z}^n$, privacy budget $\rho$-zCDP, initialization $\ww_0 \in \parspace$, number of iterations $T$.
\For{$t=0,\dots,T-1$}
    \State Query $ \loss(\ww_t,\trainset),\nabla\loss(\ww_t,\trainset),\nabla^2 \loss(\ww_t,\trainset)$
    \State Construct $\phi_t(\ww)$ from \cref{eq:cubic-upperbound-main}
    \State $ \displaystyle  \ww_{t+1} = \mathsf{DPSolver}(\phi_t(\ww),\rho / T,\ww_t)$
\EndFor
\State Output $\ww_{T}$.
\end{algorithmic}
\end{algorithm}
\end{minipage}
\hfill
\begin{minipage}{0.53\textwidth}
\begin{algorithm}[H]
\small
\caption{$\mathsf{DPSolver}$} \label{alg:subp-scvx}
\begin{algorithmic}[1]
\State Input: function $  \phi: \parspace \to \Reals:
 \phi(\theta) =\ell + \inner{g}{\theta -\theta_0} +\frac{1}{2}\inner{H(\theta-\theta_0)}{(\theta-\theta_0)}  
            +\frac{\liph}{6}\norm{\theta-\theta_0}^3$, privacy budget $\tilde{\rho}$-zCDP, initialization $\theta_0$. 
\State $  N = \frac{2\tilde{\rho} (\lipf + \lipg D + \liph D^2)^2 n^2}{(\lipf + \lipg D)^2 d} , \sigma^2 = \frac{N (\lipf + \lipg D)^2 }{2 \tilde{\rho}} $
\For{$ i=0,\dots,N-1$}
    \State $\ \eta_i = \frac{2}{\mu(i+2)}$
    \State $\text{grad}_i = g + H(\theta_i - \theta_0) + \frac{\liph}{2}\norm{\theta_i -\theta_0}(\theta_i - \theta_0)$.
    \State $ \theta_{i+1} = \proj(\theta_i - \eta_i (\text{grad}_i+ \Normal(0,\sigma^2 I_d))) $ 
\EndFor
\State Return $ \sum_{i=0}^{N-1} \frac{2i}{N(N+1)}\theta_i$
\end{algorithmic}
\end{algorithm}
\end{minipage}

\begin{theorem}\label{thm:scvx-guarantee}
Let $f$ be a $\lipf$-Lipschitz, $\lipg$-smooth, $\liph$-Lipschitz Hessian, and $\mu$-strongly convex function. Also, assume that $\parspace \subseteq \Reals^d$ has finite diameter $\diam$. Let $\ww^\star = \argmin_{w\in \parspace}\loss(\ww,\trainset)$. Then, for every $\rho>0$, $\beta \in (0,1)$, and $\trainset \in \dataspace^n$ for sufficiently large $n$, by setting the number of iterations in \cref{alg:cubic-private} to
\[
T = \Theta\Big(\frac{\sqrt{\liph}}{\mu^{3/4}}&(\loss(\ww_0,\trainset)-\loss(\ww^\star,\trainset))^{\frac{1}{4}} 
+ \log\log\big(\frac{n\sqrt{\rho}}{\sqrt{ \log(1/\beta)d}}\big)\Big)\nonumber,  
\]
and using \cref{alg:subp-scvx} as $\mathsf{DPSolver}$, we have the following:  The output of \cref{alg:cubic-private}, i.e., $\ww_T$, satisfies $\rho$-zCDP and with probability at least $1-\beta$
\[
\loss(\ww_T,\trainset)&-\loss(\ww^\star,\trainset) \leq \tilde{O} \Big(\frac{ d(\lipf + \lipg \diam)^2 \log(1/\beta)}{\mu \rho n^2} \cdot (\frac{\liph^2 \lipf \diam}{\mu^3})^{\frac{1}{4}} \Big)  \nonumber
\]
\end{theorem}

\begin{remark}
The lower bound on the excess error of any DP algorithm for the class of strongly convex functions \citealp[][Thm.~5.5]{bassily2014private} implies that the achievable excess error in \cref{thm:scvx-guarantee} is \emph{optimal} in terms of the dependence on $d$, $\rho$, and $n$. Also, the oracle complexity of our algorithm is an exponential improvement over the oracle complexity of first-order methods \citep{smith2017interaction}.
\end{remark}
\begin{remark}
The proof of \cref{thm:scvx-guarantee} suggests that \cref{alg:cubic-private} has two phases. First, while $\ww_t$ is far from $\ww^\star$, the convergence rate is $1/T^4$. Second, when $\ww_t$ is close to $\ww^\star$, the algorithm exhibits the convergence rate of $\exp(\exp(-T))$.  Notice that \cref{alg:cubic-private} is agnostic to this transition in the sense that we do not have an explicit switching step in \cref{alg:cubic-private} and \cref{alg:subp-scvx}.
\end{remark}
\begin{remark}[Comparison with \citep{Medina21}.]\label{rem:compare-damped-scvx}
In \citep[\S4]{Medina21}, the authors propose a DP variant of Newton's method. Their main idea is to add independent noise \emph{directly} to the Hessian matrix and the gradient vector using the Gaussian mechanism.  They also require that \emph{the Hessian be a rank-1 matrix}. The issue with adding noise directly to a full-rank Hessian matrix is that the noise scales with the dimension $d$, which can lead to a suboptimal excess loss. In contrast, our algorithm has a global convergence without placing restrictions on the rank of the Hessian matrix or the initialization.  
\end{remark}

\begin{remark}
\label{rem:gap}
The cubic Newton method has a non-private convergence rate of $T^{-2}$ for the class of convex (but not strongly convex) functions \citep[Thm.~4]{nesterov2006cubic}.
We leave it as an open question whether there exists a $\mathrm{DPSolver}$ such that \cref{alg:cubic-private} achieves an optimal excess error and oracle complexity for convex functions. However, this can be achieved by a DP variant of the first-order accelerated Nesterov's method \citep{nesterov1998introductory, nemirovskirobust, ghadimiacc}; see \cref{subsec:appx-nestrov}.
\end{remark}

\section{DP Logistic Regression using Second-Order Information}
\label{sec:practical-alg}
The main limitation of our cubic Newton's method (\cref{alg:cubic-private}) is that each iteration requires solving a nontrivial
subproblem. So, despite low oracle complexity, it is computationally expensive. 
Moreover, many loss functions, such as logistic loss, are not strongly convex in the unconstrained setting. 
In this section, we aim to develop a fast second-order algorithm for unconstrained logistic regression avoiding this issue. 
In many real-world classification tasks, the logistic loss is the loss of choice. The logistic loss is a convex surrogate of the 0-1 loss, and satisfies many regularity conditions that give rise to various practical optimization algorithms \citep{bach2010self,erdogdu2015newton,karimireddy2018global}.

First, we recall the logistic loss function. Let $d \in \Naturals$ and $\dataspace=\mathcal{B}^d(1) \times \{-1,1\}$ be the dimension and data space, where $\mathcal{B}^d(1)=\{x \in \Reals^d : \|x\|\le 1\}$ is the unit ball in $\Reals^d$. Let $\logf : \Reals^d \times \dataspace \to \Reals$ denote the logistic loss function defined as 
\[
\logf(w,(x,y)) = \log(1+\exp(-y\cdot\inner{w}{x})).
\]
 The gradient and Hessian of $\logf$ are given by
 \begin{equation}
     \nabla_w \logf(w,(x,y)) \!=\! \frac{-xy}{1\!+\!\exp(y\inner{w}{x})}, ~~  \nabla_w^2 \logf(w,\!(x,\!y)) \!=\! \frac{x x^\top}{(\exp(-\frac{\inner{w}{x}}{2})\!+\!\exp(\frac{\inner{w}{x}}{2}))^2}. \label{eq:hess-logloss}
 \end{equation}

Newton's method \citep[\S9.5]{boyd2004convex} is based on successively minimizing a \emph{local} second-order Taylor approximation on the function. Newton's method does not guarantee a global convergence \citep{jarre2016simple}; the reason is that the second-order Taylor approximation of the logistic loss can greatly underestimate the function. Next we show that it is possible to obtain a quadratic \emph{global upper bound} on the logistic loss function. We will use this to develop an algorithm that converges globally.%
\begin{lemma}\label{lem:quad-ub}
For every $v\in \Reals^d$, $x\in \Reals^d$, $\ww \in \Reals^d$, and $y \in \{-1,+1\}$, we have
 \begin{equation*}
 \begin{aligned}
    \logf(w,(x,y)) &\leq  \logf(v,(x,y)) + \inner{\nabla \logf(v,(x,y))}{w-v}   +  \frac{1}{2}\inner{H_{\text{qu}}(v,(x,y))(w-v) }{w-v},
    \end{aligned}
\end{equation*}
where $ \displaystyle
H_{\text{qu}}(v,(x,y))\triangleq
   \frac{\tanh(\nicefrac{\inner{x}{v}}{2})}{2\inner{x}{v}}  xx^\top \in \Reals^{d \times d}$. 
\end{lemma}
\begin{remark}
    Since $\logf$ is $\frac{1}{4}$-smooth, we can construct a simpler global quadratic upper-bound as follows \citep[Thm.~2.1.5]{nesterov1998introductory}:     \(
    \logf(w,(x,y)) \leq 
    \logf(v,(x,y)) + \inner{\nabla \logf(v,(x,y))}{w-v} + \frac{1}{8}\norm{w-v}^2.
    \)
   \cref{lem:quad-ub} is tighter than this, since $H_{\text{qu}}(v,(x,y)) \preccurlyeq \frac{1}{4}I_d$; see \cref{subsec:appx-comparison}. 
\end{remark}

\begin{remark}
\label{rem:rank-bound}

The second-order Taylor approximation and our upper bound in \cref{lem:quad-ub} both provide a quadratic approximation of the logistic loss. %
In the remainder of the paper, we write $H(v,(x,y))$ to refer to both  $\nabla^2 \logf(v,(x,y))$ and $H_{\text{qu}}(v,(x,y))$. We refer to $H(v,(x,y))$ as the second-order information (\soi) and to $H_{\text{qu}}$ as \emph{quadratic upperbound} SOI. Finally, notice both $\nabla^2 \logf(v,(x,y))$ and $H_{\text{qu}}(v,(x,y))$  are  PSD rank-1 matrices, with maximum eigenvalue $\le\frac14\|x\|^2\le\frac{1}{4}$.
\end{remark}
\subsection{Algorithm Description}
We are given a dataset $\trainset = ((x_1,y_1),\dots,(x_n,y_n)) \in (\mathcal{B}^d(1) \times \{-1,+1\})^n$ and we aim to minimize $
\losslog(\ww,\trainset) \triangleq  \frac1n \sum_{i \in \range{n}} \logf(\ww,(x_i,y_i))$. Our algorithm iteratively minimizes a quadratic approximation of $\losslog(\ww,\trainset)$. Consider 
\[
q_t(w) & \triangleq \losslog(\ww_t,\trainset) + \inner{\nabla \losslog(\ww_t,\trainset)}{w-\ww_t} + \frac{1}{2}\inner{H(\ww_t,\trainset)(w-\ww_t) }{(w-\ww_t)}\label{eq:quadratic-approx-fun},
\]
where $H(\ww_t,\trainset) \triangleq \frac1n \sum_{i \in \range{n}}  H(\ww_t,(x_i,y_i)) $.  In the non-private setting the next iterate is set to $\ww_{t+1}=\argmin_w q_t(w)= \ww_t - H(\ww_t,\trainset)^{-1} \nabla \losslog(\ww_t,\trainset)$. To develop a private variant of Newton's method, we need to characterize the sensitivity of this update rule. 
Our key observation is that \emph{the directions corresponding to small eigenvalues of $H(\ww_t,\trainset)$ are more \emph{sensitive} than the directions corresponding to large eigenvalues}. To overcome this issue, we modify the eigenvalues of $H(\ww_t,\trainset)$ to ensure a minimum eigenvalue $\ge \lambda_0$, where $\lambda_0>0$ is a carefully chosen constant. We show how to \emph{adaptively} tune $\lambda_0$ in \cref{subsec:adaptive-selection-mineig}. This procedure yields the desired stability with respect to neighbouring datasets. Formally, the modification operator is defined as follows:
\begin{definition}
\label{def:modif}
    Let $A \in \Reals^{d\times d}$ be a positive semi-definite (PSD) matrix and $\lambda_0 \ge 0$. Define 
    \[
    \nonumber
    \Psi_{\lambda_0}(A,\mathsf{clip}) = \sum_{i=1}^{d} \max\{\lambda_0,\lambda_i\} u_i u_i^\top, \quad \Psi_{\lambda_0}(A,\mathsf{add}) = \sum_{i=1}^{d} (\lambda_i + \lambda_0) u_i u_i^\top = A + \lambda_0 I_d.
    \]
    where $ A = \sum_{i=1}^{d} \lambda_i u_i u_i^\top$ is the eigendecomposition of $A$ -- i.e., $0 \le \lambda_1\le \dots \le \lambda_d $ are the eigenvalues and $u_1,\dots,u_d \in \Reals^d$ are the eigenvectors, which satisfy $\forall i \ne j ~ \|u_i\|=1 \wedge \langle u_i , u_j \rangle = 0$.
\end{definition}

\begin{wrapfigure}[17]{L}{0.51\textwidth}
\vspace{-2em}
\begin{minipage}{0.51\textwidth}
\begin{algorithm}[H] %
\small
\caption{Newton Method with Double noise}\label{alg:main-opt}
\begin{algorithmic}[1]
\State Inputs: training set $\trainset \in \mathcal{Z}^n$, $\lambda_0 >0$, $\theta \in (0,1)$,  privacy budget $\rho$-zCDP, initialization $\ww_0$, number of iterations $T$, SOI modification $\in \{\mathsf{clip},\mathsf{add}\}$.
\State Set  $\sigma_1 =\frac{\sqrt{T}}{n\sqrt{2 \rho (1-\theta)}}$
 \If {SOI modification = $\mathsf{Add}$}
 \State $ \sigma_2 = \frac{\sqrt{T}}{(4n \lambda_0^2 + \lambda_0)\sqrt{2 \rho \theta}}$
 \ElsIf{SOI modification = $\mathsf{Clip}$}
 \State $ \sigma_2 =  \frac{\sqrt{T}}{(4n \lambda_0^2 - \lambda_0)\sqrt{2 \rho \theta}}$
 \EndIf
\For{$t=0,\dots,T-1$}
    \State Query $\nabla f(\ww_t,\trainset)$ and $H(\ww_t,\trainset)$
    \State $\tilde{H}_t = \Psi_{\lambda_0}(H(\ww_t,\trainset),\text{SOI modification})$
    \State $\tilde{g}_t = \nabla \logf(\ww_t,\trainset) + \Normal(0,\sigma^2_1 I_d)$ \label{line:main-alg-grad-noise}
    \State $\ww_{t+1} = \ww_{t} - \tilde{H}_t^{-1}\tilde{g}_t + \Normal(0,\norm{\tilde{g}_t}^2 \sigma_2^2  I_d) $ \label{line:main-alg-direction-noise}
\EndFor
\State Output $\ww_{T}$.
\end{algorithmic}
\end{algorithm}
\end{minipage}
\end{wrapfigure}

 \cref{alg:main-opt} describes our algorithm where its output satisfies $\rho$-zCDP; the privacy analysis can be found in \cref{subsec:appx-proof-add,subsec:appx-proof-clip}. Our DP algorithm differs from the non-private Newton's method in three ways: (1) We first privatize the gradient by adding noise. (2) We modify $H(\ww_t,\trainset)$ to ensure its eigenvalues are not too small. And (3) we add a second noise to the update computed using the noised gradient and modified second-order information (SOI).   

Notice that \cref{alg:main-opt} has \emph{four variations} based on
the \soi and the modification of \soi, namely, \hessc, \hessa, \quc, and \qua which refer to using Hessian and clip, Hessian and add, quadratic upper bound (See \cref{lem:quad-ub}) and clip, and quadratic upper bound and add, respectively.

\begin{remark}[Generalization of \cref{alg:main-opt}]
In this section our main focus is on DP logistic regression, and the privacy guarantees hold for the logistic loss.
Nevertheless, in \cref{subsec:generalization-alg},  we present a generalization of \cref{alg:main-opt} whose privacy guarantee holds for \emph{every} convex, doubly differentiable, Lipschitz, and smooth loss function \emph{without any constraints on the rank of Hessian}. 
The main technical challenge for sensitivity analysis is proving the approximate Lipschitzness of $\Psi$ in the operator norm (See \cref{lem:lip-operator-norm}). This demonstrates that our algorithm is more general than objective perturbation \citep{chaudhuri11a,kifer12,iyengar2019towards} and the  private damped Newton's method \citep{Medina21} which both require a low-rank Hessian.
\end{remark}

\subsection{Private and Adaptive Selection of Minimum Eigenvalue}
\label{subsec:adaptive-selection-mineig}
One of the hyperparameters of \cref{alg:main-opt} is the minimum eigenvalue $\lambda_0$. 
There exists a tradeoff for choosing $\lambda_0$. We ideally want the modification to be as small as possible, so that the \soi is preserved. However, decreasing $\lambda_0$ increases $\sigma_2$ and we add more noise. 
To deal with this problem, we propose a heuristic rule for an adaptive, private, and time-varying selection of the minimum eigenvalue. We wish to find $\lambda_{0,t}$ that minimizes expected loss at the next iteration, for which we have the quadratic approximation \eqref{eq:quadratic-approx-fun}. More formally,  we compute $\lambda_{0,t}$ as $\argmin_{\lambda} \EE\left[ q_t\left(\ww_t -\Psi_{\lambda}(H(\ww_t,\trainset),\text{SOI modification})\tilde{g}_t + \norm{\tilde{g}_t}\sigma_2(\lambda) \cdot \xi \right) \right]$ where $q_t$ is given in \eqref{eq:quadratic-approx-fun} and $\xi \sim \Normal(0,I_d)$.
We show in \cref{subsec:Description-alg-adaptive} that an approximate minimizer is 
$ \lambda_{0,t} \propto \big( \frac{  \text{trace}(H_t(\ww_t,\trainset))}{n^2 \times \text{privacy budget for the direction}} \big)^{\frac{1}{3}}$. Note that $\lambda_{0,t}$ depends on the data through $\text{trace}(H(\ww_t,\trainset))$, which has sensitivity $1/4n$, so it can be estimated privately. In \cref{subsec:Description-alg-adaptive}, we provide the algorithmic description of  a variant of \cref{alg:main-opt} with an adaptive and private minimum eigenvalue. 
In particular, we divide the privacy budget at each iteration into three parts: (1) privatizing the gradient; (2) estimating the trace of \soi; and (3) privatizing the direction.  
We use this variant for our numerical experiments in \cref{sec:numerical-results}.
\subsection{Convergence Results for \cref{alg:main-opt}}
\label{subsec:convergence-results-practical}
In this section, we provide data-dependent convergence guarantees for \cref{alg:main-opt}. We express these guarantees in terms of the conditional expectation $\EE_t\left[\cdot\right]= \EE\left[\cdot | \{\ww_i\}_{i \in [t]}\right]$ and they can be easily extended to obtain high probability bounds. Before presenting the results, we introduce a notation. For a dataset $\trainset = ((x_1,y_1),\dots,(x_n,y_n)) \in (\Reals^d \times \{-1,+1\})^n$, let $V \in \Reals^{d\times d}$ denote the \emph{orthogonal projection matrix} on the linear subspace spanned by $\{x_1,\dots,x_n\}$. For every vector $u \in \Reals^d$, define $\norm{u}_V \triangleq \sqrt{u^\top V u}$. This norm naturally arises since for every $w\in \Reals^d$ we have $\losslog(\ww,\trainset)-\losslog(\ww^\star,\trainset)\leq \frac{1}{8}\norm{\ww-\ww^\star}_V^2$ where $\ww^\star = \argmin \losslog(\ww,\trainset)$ (See \cref{appx:suboptgap}). 

\subsubsection{Local Convergence Guarantee of \hessc and \hessa}
\begin{theorem}\label{thm:practical-local-convergence}
Let $\trainset$ denote the dataset and $\rank$  denote the dimension of the linar subspace spanned by  $\{ x_1,\dots,x_n\}$. Let $\lambda_{\text{min},t}$ be the smallest non-zero eigenvalue of $\nabla^2 \losslog(\ww_t,\trainset)$ and $\rho$ be the privacy budget (in zCDP) per iteration. Then,
\[
\nonumber
\EE_t\left[\norm{\ww_{t+1} - \ww^\star}_V^2\right] \leq  \nu_{1,t}^2\norm{\ww_{t} - \ww^\star}_V^2 +  2\nu_{1,t}\nu_{2,t} \norm{\ww_{t} - \ww^\star}_V^{3} + \nu_{2,t}^2 \norm{\ww_{t} - \ww^\star}_V^4 + \Delta,
\]
where the coefficients are given by
\small
\begin{align} \label{eq:coeff-hess}
    \nu_{1,t} = 1-\frac{\tilde{\lambda}_{\text{min},t}}{\lambda_0} + \frac{\sqrt{\rank}}{\left(4n\lambda_0^2 - \lambda_0\right)\sqrt{2\rho \theta}} , \quad   \nu_{2,t} = \frac{0.05}{\tilde{\lambda}_{\text{min},t}}, \quad \Delta = O\left(\frac{\rank}{\rho(1-\theta)n^2}\frac{1}{(\tilde{\lambda}_{\text{min},t})^2}\right).
\end{align}
Here,  $\tilde{\lambda}_{\text{min},t}=\left\{\def\arraystretch{1.2}\begin{tabular}{@{}l@{\quad}l@{}}
  $\min\{\lambda_{\text{min},t},\lambda_0\}$ & for $\hessc,$  \\
  $\lambda_{\text{min},t} + \lambda_0$ & for $\hessa$,
\end{tabular}\right.$ depends on the modification procedure.
\normalsize
\normalsize
\end{theorem}
This type of convergence is known as \emph{composite convergence}, as it is a combination of linear and quadratic rates, and has been observed in the convergence analysis of several quasi-Newton's methods \citep{erdogdu2015convergence,erdogdu2015newton,roosta2016sub,xu2016sub}.

\begin{remark}
$\lambda_{\text{min},t}$ is the smallest \emph{non-zero} eigenvalue of $\nabla^2 \losslog(\ww_t,\trainset)$. Therefore, for sufficiently large $n$ we have $0<\nu_{1,t}<1$. It shows \cref{alg:main-opt} with Hessian as SOI is, in-expectation, a descent algorithm locally given $\norm{\ww_t - \ww^\star}$ is sufficiently larger than $\Delta$. Roughly speaking, \cref{thm:practical-local-convergence} guarantees a linear convergence to a ball around the optimum whose radius is given by $\Delta$. We also observe the linear rate in  \cref{fig:convergence}. Moreover, the error due to the privacy, i.e., $\Delta$ in \cref{eq:coeff-hess}, is proportional to the rank of the feature vectors which is always smaller than $d$. These interesting properties is due to the convergence analysis with respect to $\norm{\cdot}_V$. 
\end{remark}

\begin{remark}
The coefficients of the convergence in \cref{eq:coeff-hess} depend on the iteration step which is an undesirable aspect of the results. In \cref{lem:logloss-prop}, we prove that $|\lambda_{\text{min},t}-\lambda^\star_{\text{min}}|\leq 0.1 \norm{\ww_t - \ww^\star}_V$ where $\lambda^\star_{\text{min}}$ is the smallest non-zero eigenvalue of $\nabla^2 \losslog(\ww^\star,\trainset)$. Therefore, the coefficients can be well-approximated by their analogous values evaluated at the optimum.
\end{remark}

\subsubsection{Global Convergence Guarantee of \quc and \qua}
We also establish a global convergence guarantee for \quc and \qua. Due to the space the formal statement and proof are deferred to \cref{sec:proof-qu-convergence}. Roughly speaking, under the assumption of \emph{local strong convexity at
the optimum} \citep{bach2014adaptivity},  \quc and \qua converge globally: this is intuitive since \quc and \qua are based on minimizing a global upper bound on the function. %

\color{black}

\section{Experimental Results}
\label{sec:numerical-results}

\begin{figure}
\centering
    \begin{subfigure}[t]{0.37\textwidth}
        \centering
       \includegraphics[width=\textwidth]{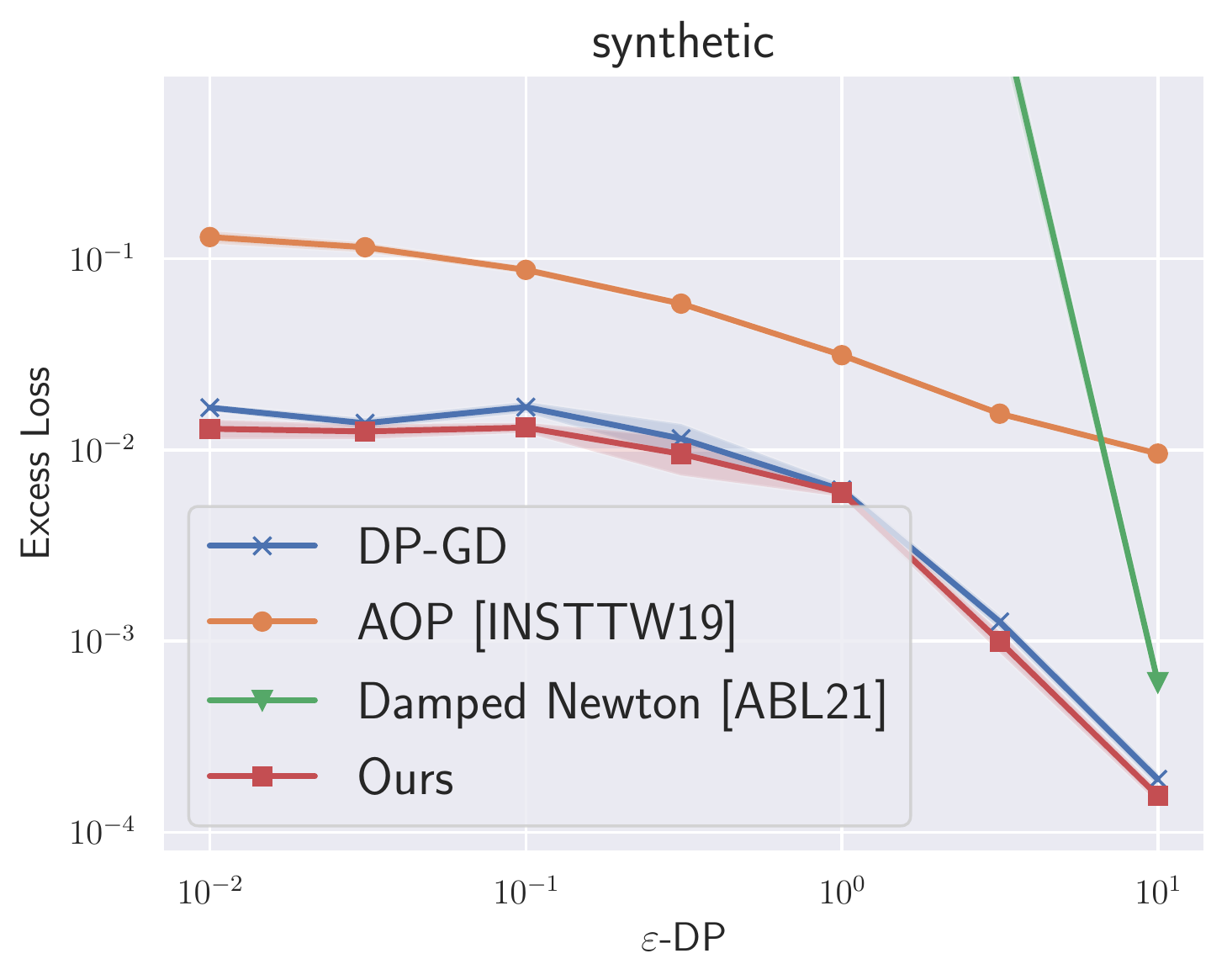}
    \end{subfigure}%
    ~ \hspace{-1em}
    \begin{subfigure}[t]{0.37\textwidth}
        \centering
      \includegraphics[width=\textwidth]{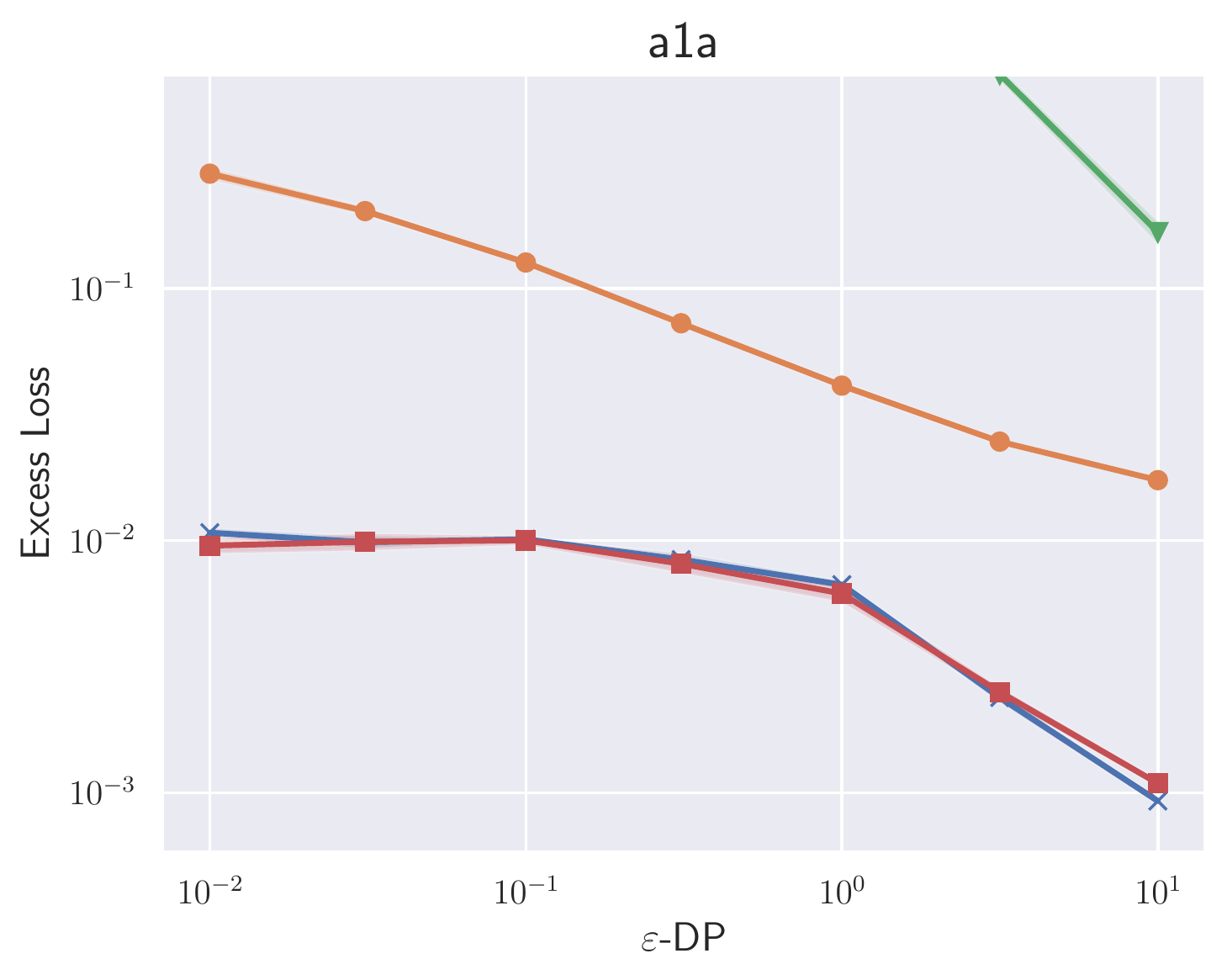}
    \end{subfigure}
     ~ \hspace{-1em}
         \begin{subfigure}[t]{0.37\textwidth}
        \centering
      \includegraphics[width=\textwidth]{figs-new/a1a-priv-util.pdf}
    \end{subfigure}%
    ~ \hspace{-1em}
    \begin{subfigure}[t]{0.37\textwidth}
        \centering
      \includegraphics[width=\textwidth]{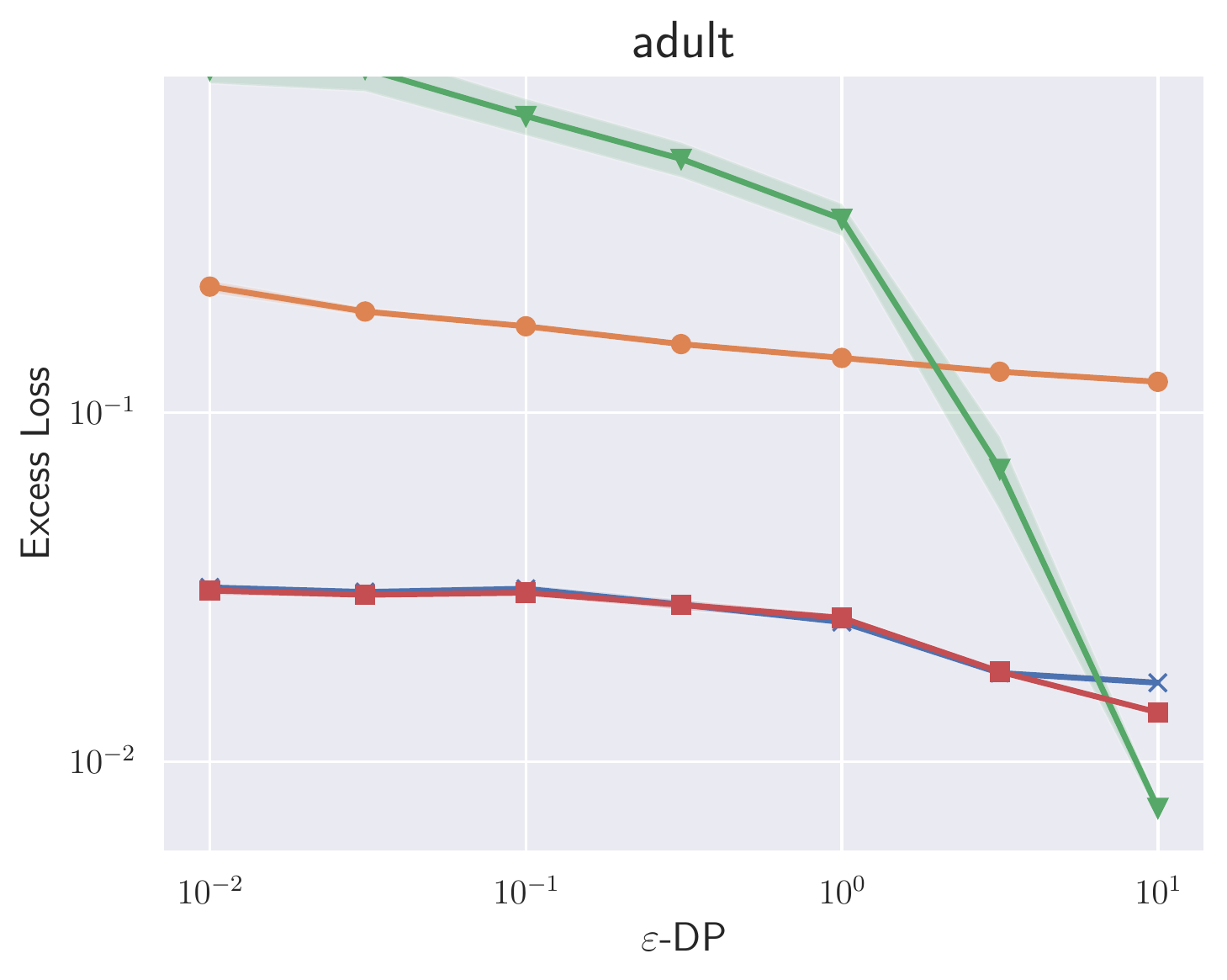}
    \end{subfigure}
    \caption{Privacy-Utility tradeoff on different datasets.}
        \label{fig:privacy-utility}
\end{figure}

In this section, we evaluate the performance of our algorithm (\cref{alg:main-opt} with the adaptive minimum eigenvalue selection from \cref{subsec:adaptive-selection-mineig}) for the problem of \emph{binary classification} using \emph{logistic regression}. 
For brevity, many of the details behind our implementation and more experimental results are deferred to \cref{appx:num-results}. The setup of the experiments is as follows: 

\textbf{Baseline1- DP-(S)GD}: The update rule is $\ww_{t+1} = \ww_t - \eta \nabla \loss(\ww_t,\trainset) + \xi$ where $\xi$ is a Gaussian noise \citep{song2013stochastic,bassily2014private,abadi2016deep}. Since the logistic loss is $1$-Lipschitz, we do not need gradient clipping. The Lipschitzness parameter controls the variance of the Gaussian random vector. 
To draw a fair comparison and show the advantage of using second-order information, we chose the stepsize to be equal to the inverse smoothness. This setting for DP-(S)GD  is \emph{minmax} optimal in terms of the privacy-utility tradeoff \citep{ghadimi2013stochastic}. \textbf{Baseline2- Approximate Objective Perturbation (AOP)}:  AOP is built on objective perturbation \citep{chaudhuri11a,kifer12}. Objective perturbation consists of a two-stage process:
(1) \emph{perturbing} the objective function by adding a random
linear term and (2) outputting the minimum of the perturbed
objective. Releasing such a minimum is sufficient for achieving DP guarantees \citep{chaudhuri11a,kifer12}, but only if we can find the exact minimum of the perturbed objective. AOP extends objective perturbation to permit using an \emph{approximate} minimum of the perturbed objective \citep{iyengar2019towards,code-aop}. Notice AOP is not an iterative optimization algorithm. 
\textbf{Baseline3- Damped Newton Method \citep{Medina21}}:   The algorithm in \citep{Medina21} is a variant of damped Newton's method with the assumption that the Hessian of loss function is rank-1, which holds for the logistic loss. Their algorithm is based on adding two \iid~noises to the Hessian and the gradient:  $\ww_{t+1} = \ww_{t} - \eta_t H_{\text{noisy},t}(\ww_t,\trainset)^{-1}\tilde{g}_t $, where $\eta_t$ is the stepsize, $H_{\text{noisy},t}(\ww_t,\trainset) = \nabla^2 \losslog(\ww_t,\trainset) + \Xi_t$ and $\tilde{g}_t=\nabla \losslog(\ww_t,\trainset) + \xi_t$. Here $ \Xi_t$ and $\xi_t$ are carefully chosen Gaussian noise. With $\eta_t=1$, our experiments show that their algorithm is not converging. We use the strategy suggested in  \citep[Page~22]{Medina21} and set $\eta_t = \log(1+\beta_t)/\beta_t  $ where $\beta_t = \norm{\nabla^2 \losslog(\ww_t,\trainset)^{-1}\nabla \losslog(\ww_t,\trainset)}$. This stepsize selection makes the algorithm \emph{non-private}, however, it serves as a good baseline. \textbf{Datasets:} We conducted experiments on six publicly available datasets: a1a, Adult, covertype, synthetic, fashion-MNIST, and protein  (\cref{appx:num-results} includes fashion-MNIST and protein results). The synthetic dataset is generated as follows: Fix $d\in \Naturals$ and $\ww^\star \in \Reals^d$. Then, (1) the feature vectors  $\{x_i\in \Reals^d ~:~ i \in \range{n}\}$  are independent and sampled uniformly at random from the unit sphere in $\Reals^d$, (2) for the $i$-th datapoint the label is $+1$ with probability  $ (1+\exp(-\inner{x_i}{w^\star}))^{-1}$ and $-1$ otherwise. \textbf{Privacy Notion:} The privacy notion for our experiments is $(\epsilon,\delta=(\text{num. of samples})^{-2})$-DP. Next, we present the results.

\begin{figure}
\centering
    \begin{subfigure}[t]{0.37\textwidth}
        \centering
       \includegraphics[width=\textwidth]{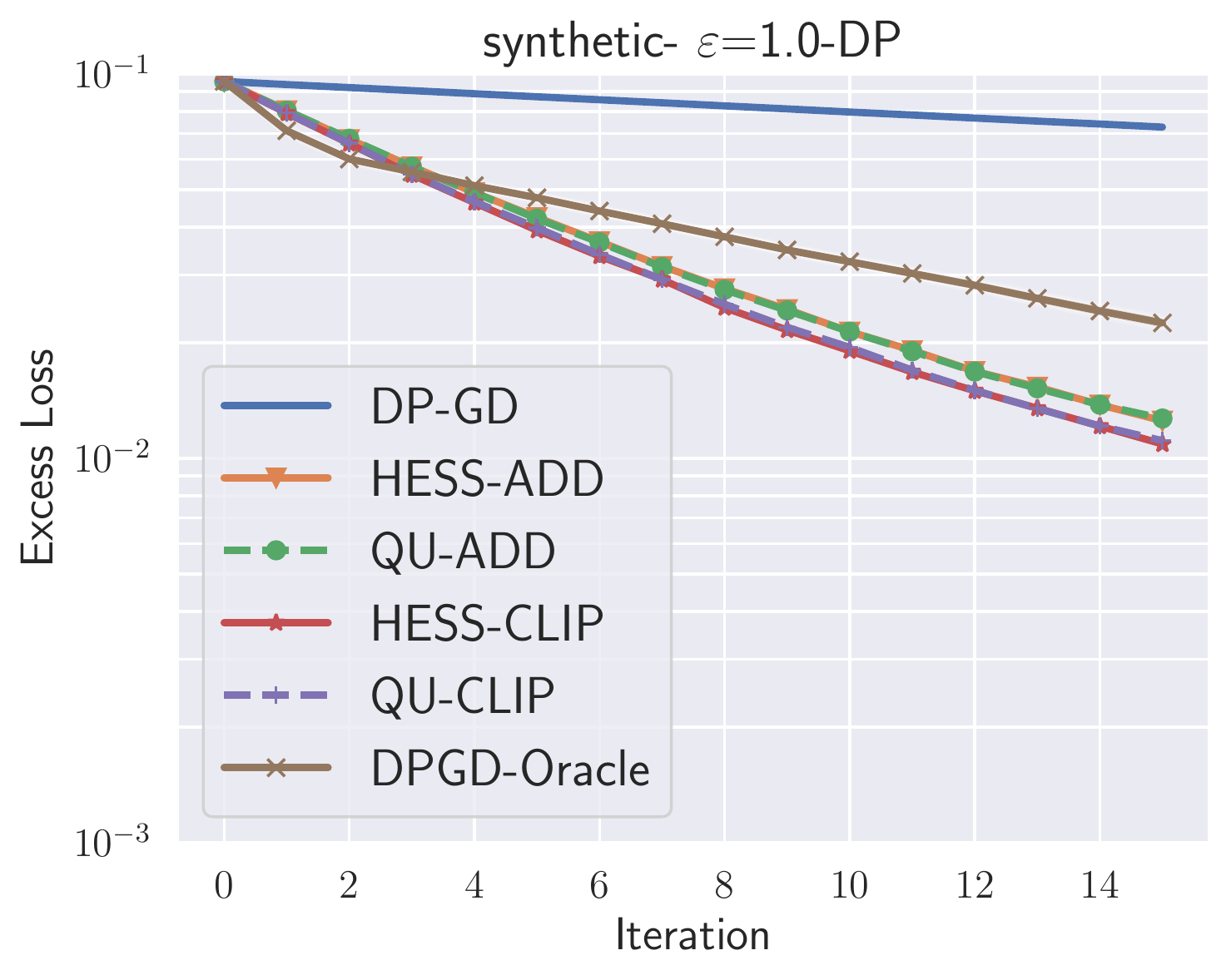}
    \end{subfigure}%
    ~ \hspace{-1.3em}
    \begin{subfigure}[t]{0.37\textwidth}
        \centering
      \includegraphics[width=\textwidth]{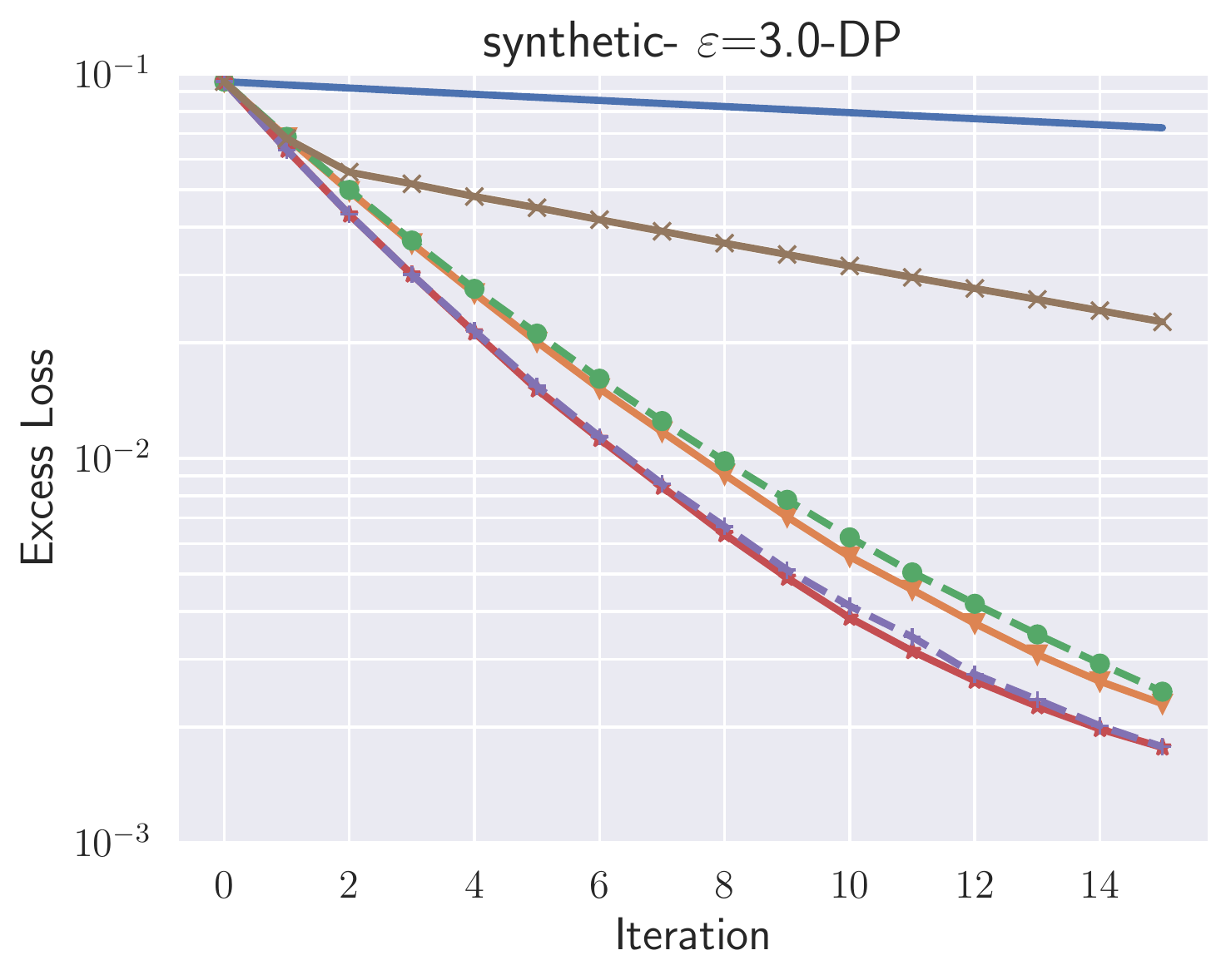}
    \end{subfigure}
     ~ \hspace{-1.3em}
         \begin{subfigure}[t]{0.37\textwidth}
        \centering
      \includegraphics[width=\textwidth]{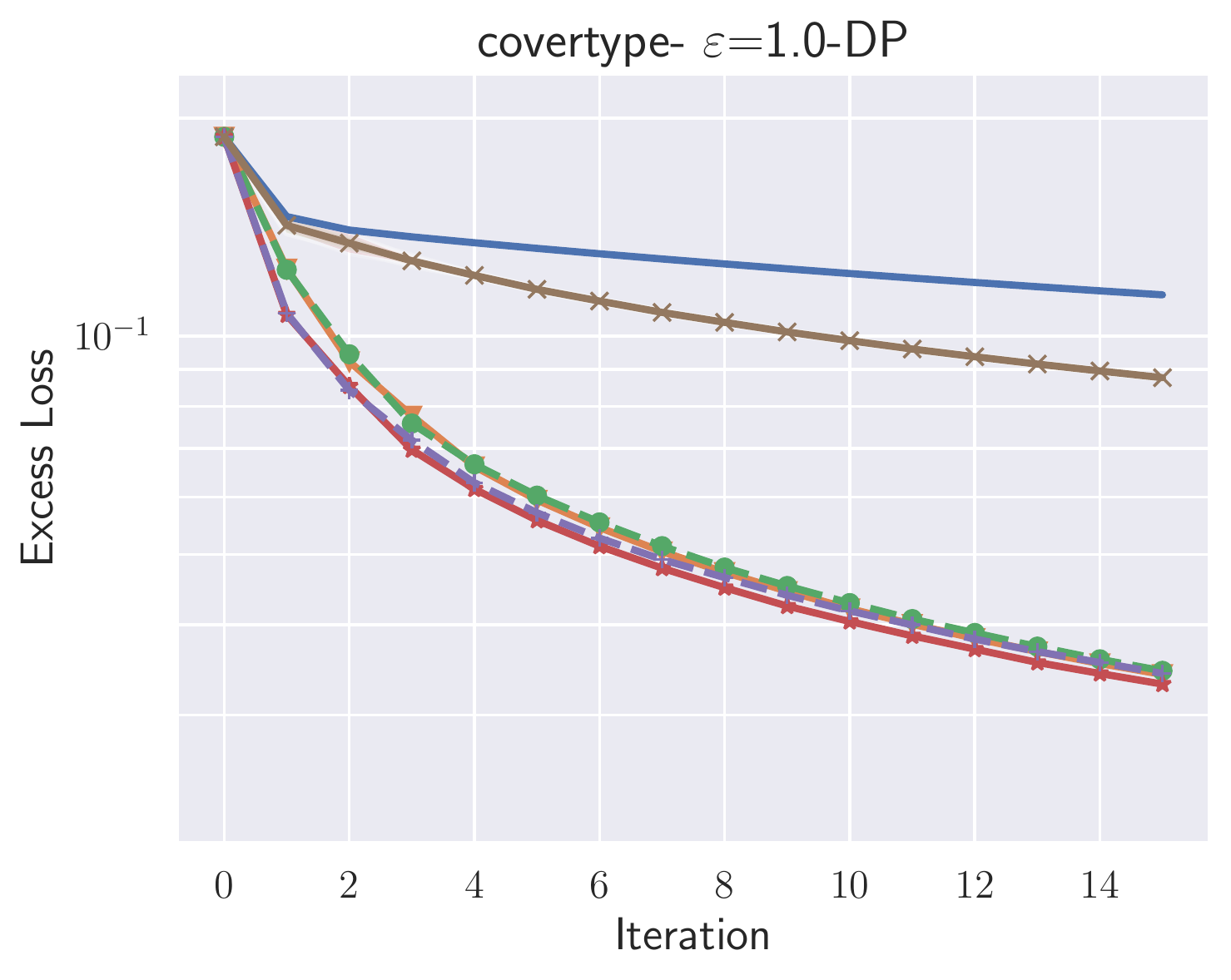}
    \end{subfigure}%
    ~ \hspace{-1.3em}
    \begin{subfigure}[t]{0.37\textwidth}
        \centering
      \includegraphics[width=\textwidth]{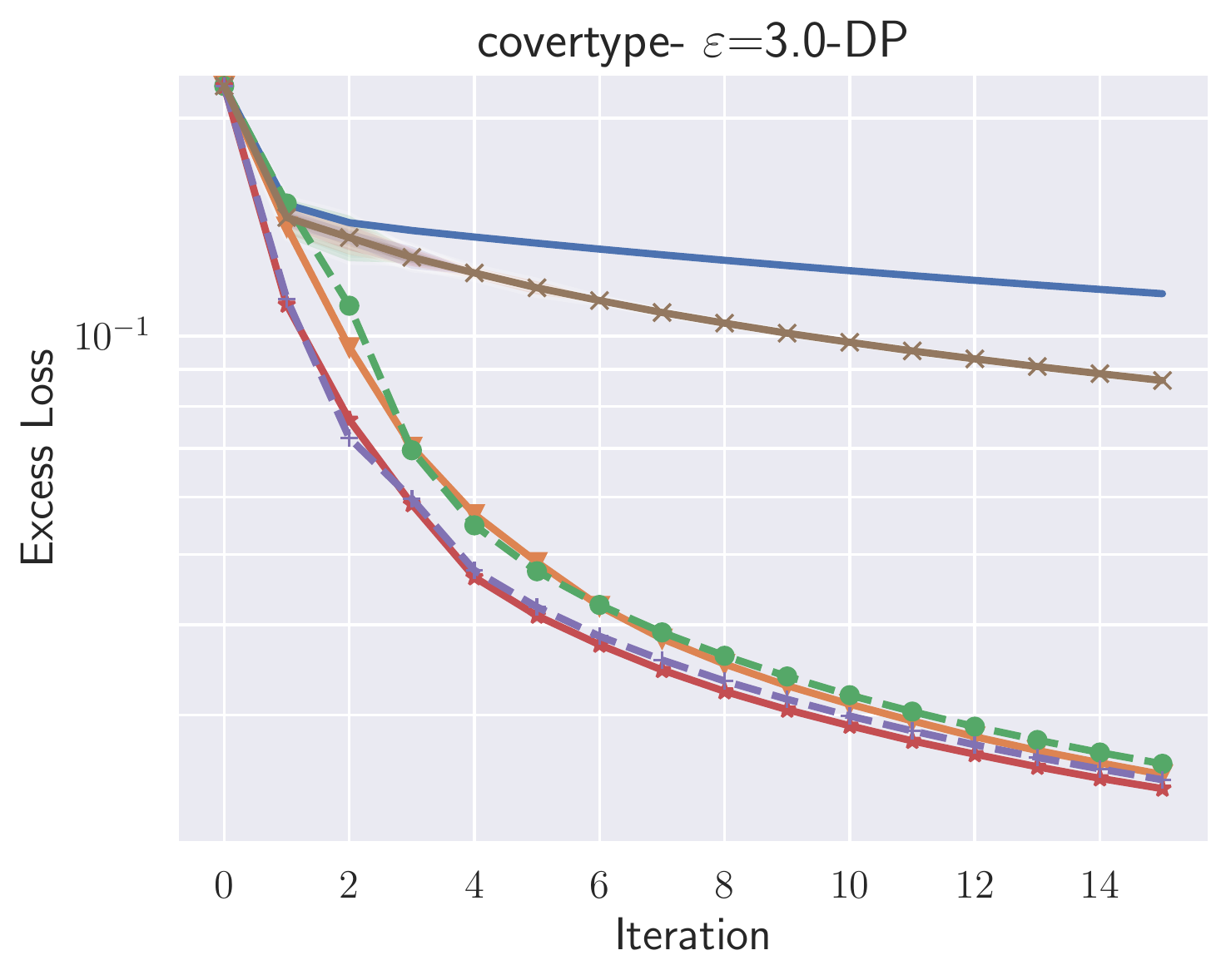}
    \end{subfigure}
    \caption{Comparison with DP-GD Oracle where at each iteration the stepsize tuned non-privately.}
        \label{fig:convergence}
\end{figure}

\textbf{\underline{Privacy-Utility-Run Time Tradeoff}}: We study the tradeoff for our algorithm and compare it with other baselines for a broad range of $\varepsilon \in \{0.01,\dots,10\}$. We  \emph{non-privately tune} the total number of iterations of the iterative algorithms and report the best achievable excess error in \cref{fig:privacy-utility}. As can be seen our algorithm almost always achieves the best excess loss for a broad range of $\epsilon$. Also, \cref{fig:privacy-utility} shows that damped private Newton method of \citep{Medina21} achieves a low excess loss only for large $\epsilon$. \cref{fig:privacy-utility} indicates that DP-GD and our algorithm are the best in terms of excess loss. In \cref{table:walltime}, we compare the run time of DP-GD and our algorithm, i.e., the computational time in seconds for achieving the excess loss in \cref{fig:privacy-utility}. As can be seen, for many challenging datasets, our algorithm is $10$-$40\times$ faster than DP-GD. Our experiments are run on CPU. We also remark that each step of \cref{alg:main-opt}, i.e., computing gradient and \soi, is  heavily parallelizable implying that the run time of \cref{alg:main-opt} can be made much smaller by an efficient implementation. Also, the reported numbers in \cref{fig:privacy-utility} and \cref{table:walltime} correspond to \hessc.

\newcommand\VRule[1][\arrayrulewidth]{\vrule width #1}
\begin{table}
\small
\centering
\begin{tabular}{lcccccc}\toprule
& \multicolumn{4}{c}{$\displaystyle\frac{T^\star_{\text{DP--GD}}}{T^\star_{\text{ours}}}$} & \multicolumn{2}{c}{$T^\star_{\text{ours}} (\text{sec})$}
\\\cmidrule(lr){2-5}\cmidrule(lr){6-7}
           & $\varepsilon=0.01$ & $\varepsilon=0.1$ & $\varepsilon=1$ & $\varepsilon=10$ & $\min(T^\star_{\text{ours}})$ (sec.)  & $\max(T^\star_{\text{ours}})$ (sec.)\\\midrule
a1a    & $4.87 \times$ &	$2.95\times$ &	$5.09\times$ 	& $30.59\times$   &      $ 2.45$    &      $ 4.2$   \\
synthetic & $2.90 \times$ &	$2.90 \times$ &	$5.19 \times $ &	$11.61 \times$   &      $  0.18$   &      $  0.21$  \\
adult &  $12.08 \times$	 & $11.84 \times$ &	$22.17 \times$ &	$38.16 \times$      &      $  6.81$   &      $  8.07$\\
covertype   & $24.19 \times$ &	$19.85 \times$  &	$35.70 \times$ &	$36.20 \times$   &      $ 2.93$   &      $  3.58$    \\\bottomrule
\end{tabular}
\caption{Comparison between the run time of our algorithm and DP-GD in terms of the ratio $\displaystyle {T^\star_{\text{DP-GD}}}/{T^\star_{\text{our}}}$. The last two columns show the minimum and maximum run time of our algorithm.}
\label{table:walltime}
\end{table}

\color{black}
\textbf{\underline{Second Order Information vs Optimal Stepsize}}:
In non-private convex optimization, the key to the success of second-order optimization algorithms is that the second-order information acts as a preconditioner, and the same performance \emph{cannot} be attained by optimally tuning the stepsize for GD algorithm. To investigate whether the same holds for our algorithms, we consider the following variant of DP-GD. Let $\tilde{g}_t$ denote the perturbed gradient obtained by adding a Gaussian random vector to $\nabla \losslog(\ww_t,\trainset)$. Instead of a constant stepsize, the stepsize at iteration $t$ is chosen based on $\eta_t = \arg\min_{\eta\geq 0} \losslog(\ww_t - \eta \tilde{g}_t)$. Notice this variant is obviously $\emph{not DP}$. We refer to this variant as \emph{DP-GD-Oracle}. The comparison with DP-GD-Oracle lets us answer the following question: \emph{Could we have just computed a single number, i.e., stepsize,  to achieve the same performance as our second-order optimization algorithms which require computing a $d\times d$ matrix?} In \cref{fig:convergence}, we compare the convergence speed of our algorithms with DP-GD-Oracle in low- and high-privacy regimes. \cref{fig:convergence} shows our algorithms converge faster than DP-GD-Oracle which is not even a DP algorithm. \cref{fig:convergence} confirms the expectation that as the privacy budget increases the difference between our algorithms and DP-GD-Oracle increases since we can use more curvature information.

\begin{figure}
\centering
    \begin{subfigure}[t]{0.37\textwidth}
        \centering
       \includegraphics[width=\textwidth]{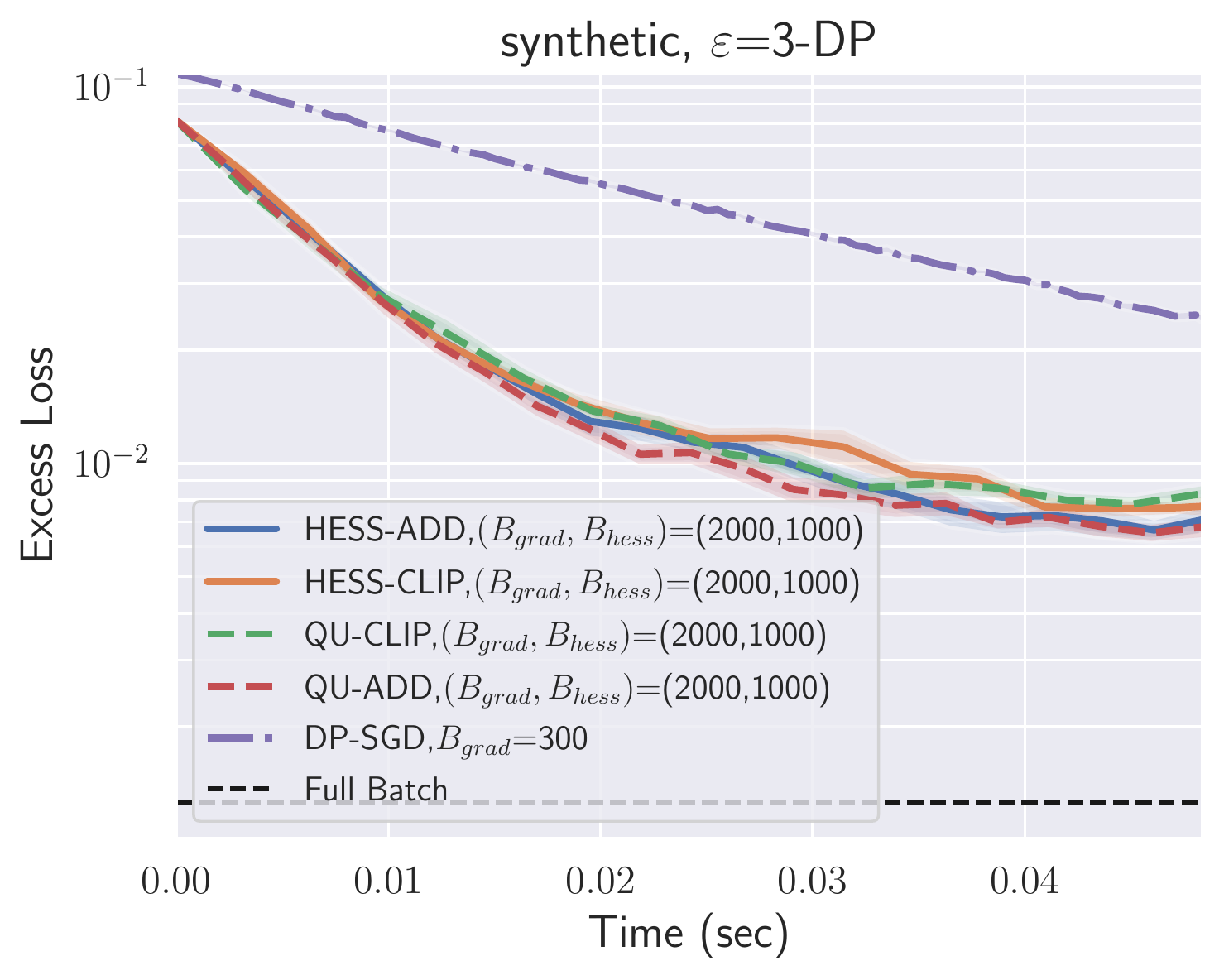}
    \end{subfigure}%
    ~ \hspace{-1em}
    \begin{subfigure}[t]{0.37\textwidth}
        \centering
      \includegraphics[width=\textwidth]{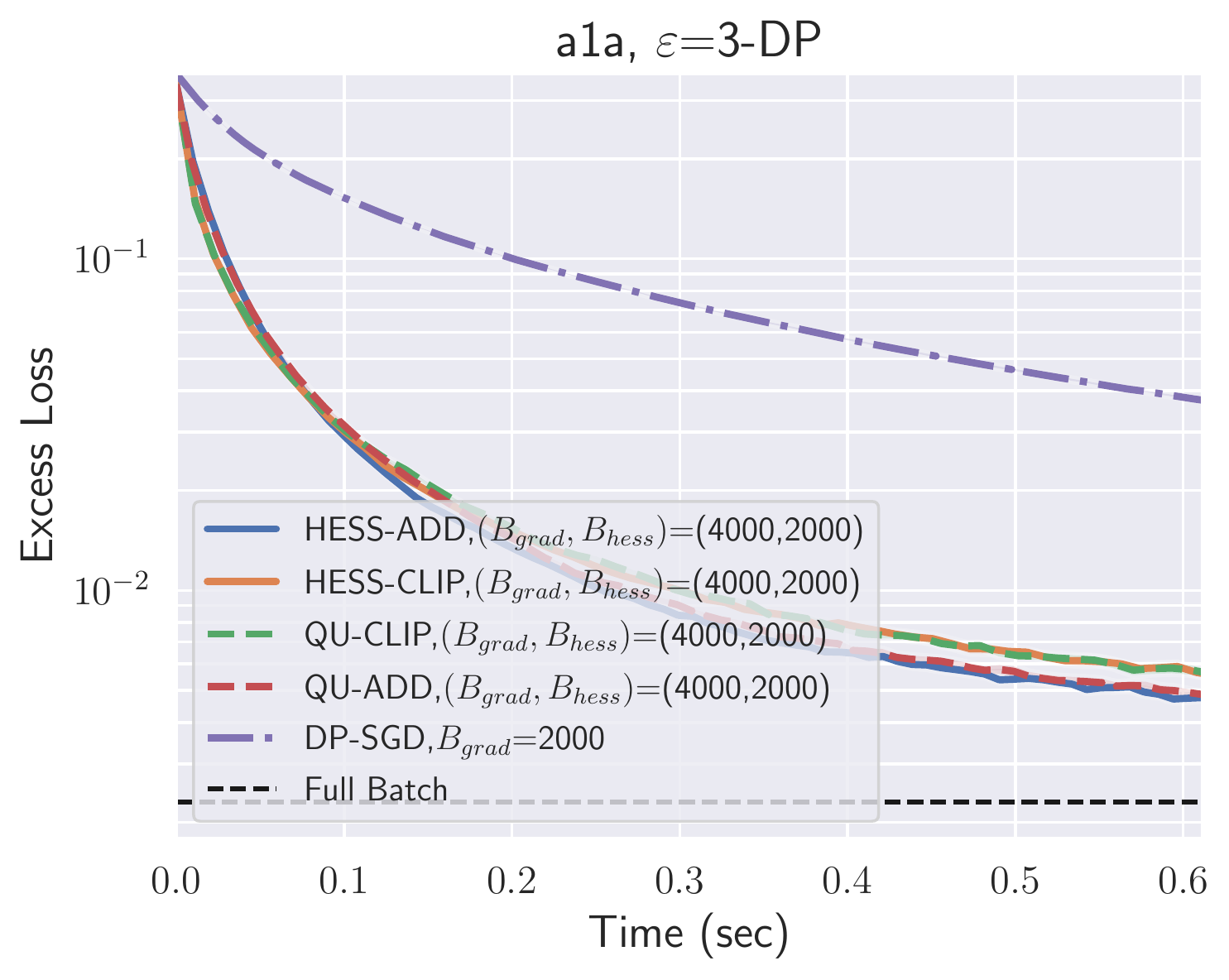}
    \end{subfigure}
     ~ \hspace{-1em}
         \begin{subfigure}[t]{0.37\textwidth}
        \centering
      \includegraphics[width=\textwidth]{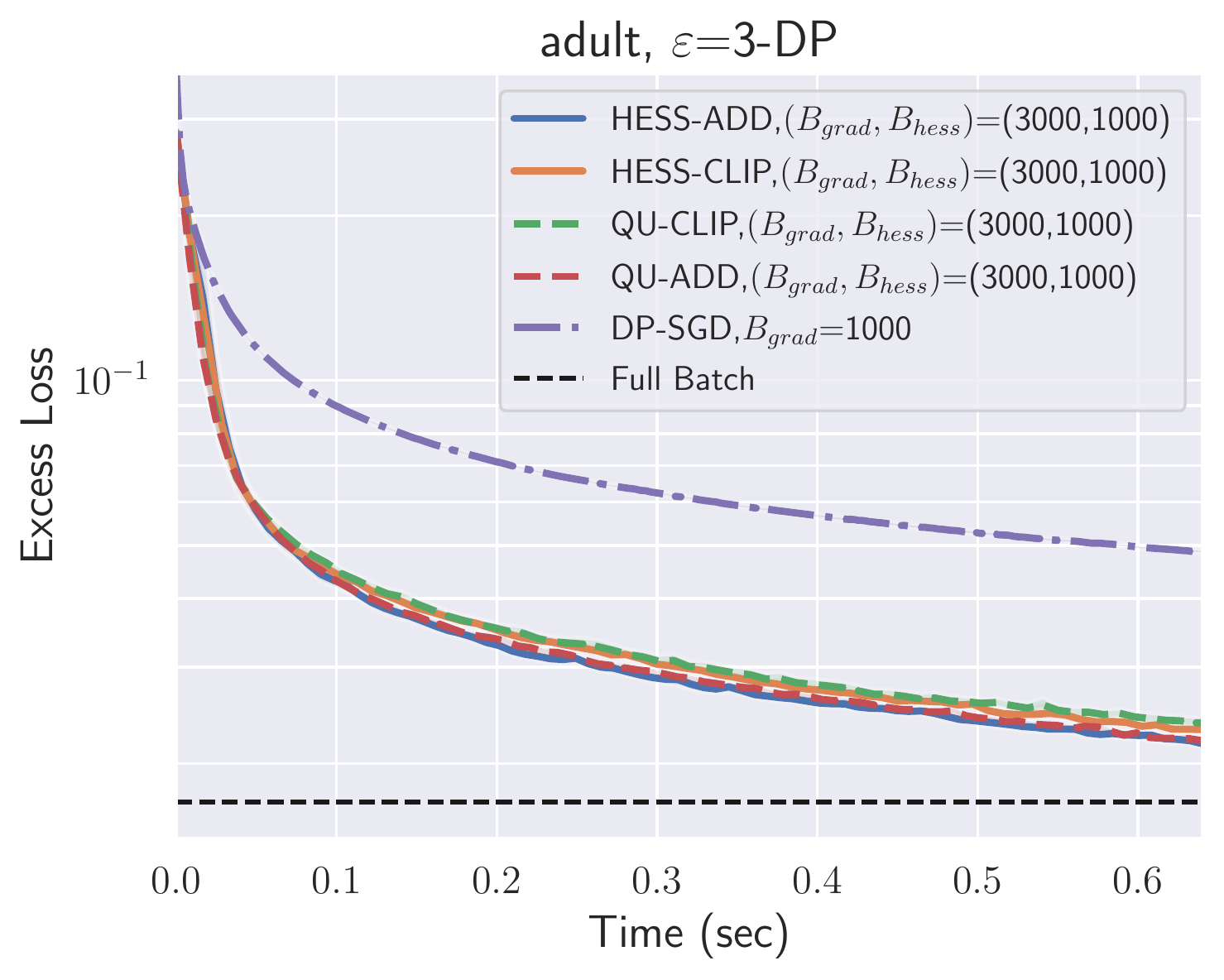}
    \end{subfigure}%
    ~ \hspace{-1em}
    \begin{subfigure}[t]{0.37\textwidth}
        \centering
      \includegraphics[width=\textwidth]{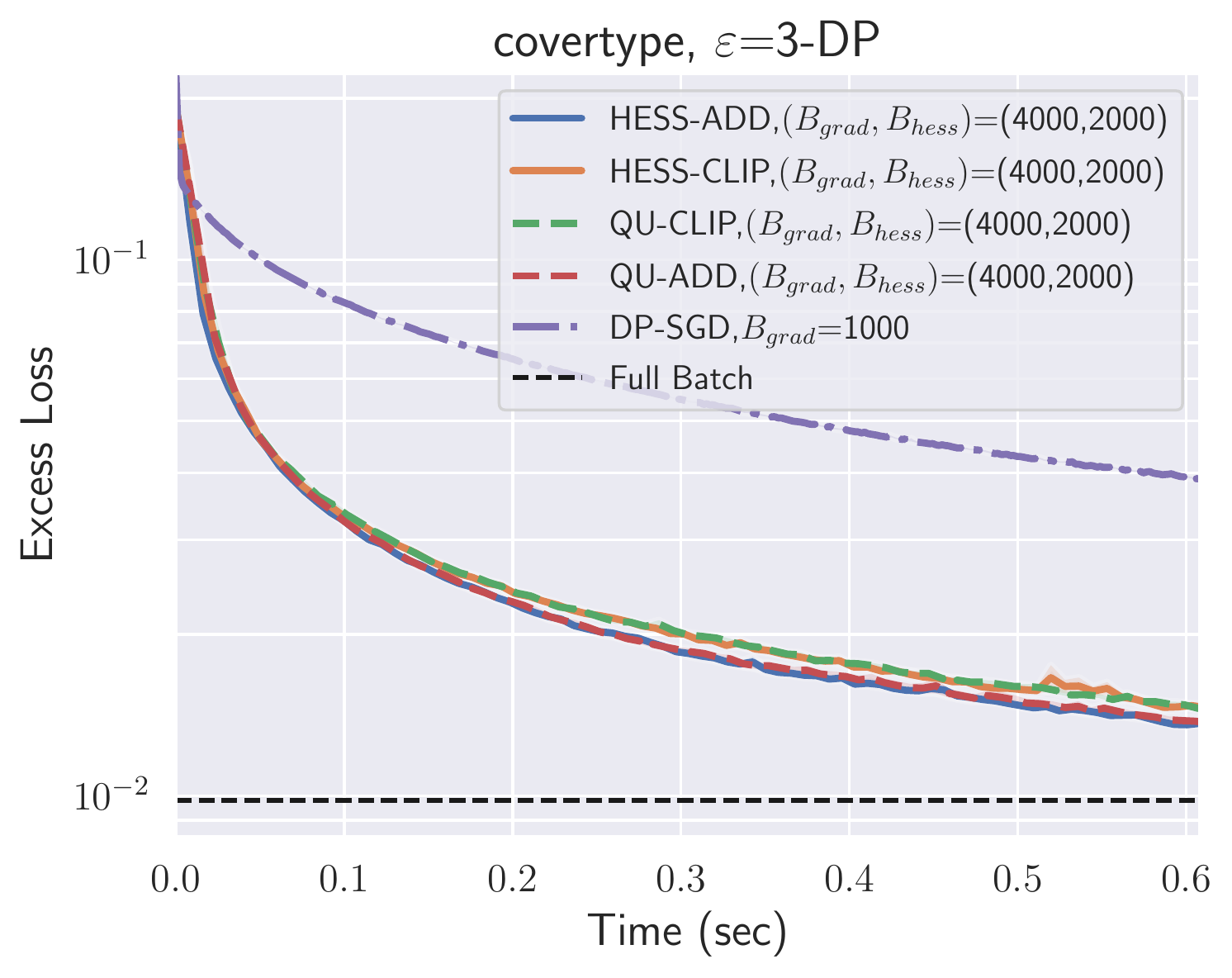}
    \end{subfigure}
    \caption{Minibatch Variant of Our Algorithm and Comparison with DP-SGD}
        \label{fig:stochasic}
\end{figure}

\subsection{Minibatch Variant of Our Algorithm and Comparison with DP-SGD}
So far we have considered full-batch algorithms that compute first- and second-order information on the entire dataset.
We extend  \cref{alg:main-opt} to the minibatch setting, where, at each iteration, the gradient and SOI matrix are computed using a subsample of the data points. 
In \cref{subsec:stochastic-variant-our} we provide a formal algorithmic description of the minibatch version of \cref{alg:main-opt} along with its privacy proof. 
Then, we compare the convergence speed and excess loss with DP-SGD.

DP-SGD is faster than DP-GD, but to achieve good privacy and utility, we need large batches \citealp[][Fig.~2]{ponomareva2023dp}. This is in stark contrast with non-private SGD, where larger batch sizes yield diminishing returns \citep{zhang2019algorithmic}. In particular, to achieve the best excess loss we need to select the batch size as large as possible. We select the batch size of DP-SGD so that the achievable excess loss will be close to the full batch versions. Specifically, we select $\frac{\text{batch size DP-SGD}}{\text{number of samples}}\approx 0.02$ and tune the number of iterations of DP-SGD to obtain the best result.  \cref{fig:stochasic} shows the progress of different algorithm versus run time. Obviously, for a fixed run time DP-SGD performs more iterations compared to our algorithms. Nevertheless, our algorithms achieve the same excess error as DP-GD with $8$-$10 \times $ faster run time over all the datasets while \emph{the batch sizes of our algorithms are larger than that of DP-SGD}. 
We observe that the variations of our algorithms based on the adding operator performs better in the minibatch setting. This can be attributed to the smaller $\sigma_2$ for the adding operator in \cref{alg:main-opt}.
In summary, the comparison between privacy-utility-wall time tradeoff of the subsampled variant of our algorithm and DP-SGD is similar to their full-batch counterparts.

\section{Future Directions}
We showed that second-order methods can be used in the DP setting both for improving worst-case convergence guarantees and designing faster practical algorithms. We believe our results open up many directions: One important direction is to improve the per-iteration computational cost of our algorithm for the case that $d$ is large. In the non-private setting, a line of research tries to address this limitation by constructing an approximation to \soi such that the update is efficient, yet still provides sufficient \soi \citep{erdogdu2015convergence,erdogdu2015newton,xu2016sub,agarwal2017second}. It would be interesting to investigate how the ideas developed in our paper could be incorporated into these methods.

\section*{Acknowledgments}
The authors would like to thank Murat Erdogdu, Jalaj Upadhyay, and Mohammad Yaghini for helpful discussions. Resources used in preparing this research were provided, in part, by the Province of Ontario, the Government of Canada through CIFAR, and companies sponsoring the Vector Institute \url{www.vectorinstitute.ai/partners}.

\printbibliography

\newpage
\appendix
\onecolumn

\section{Notations}
\label{appx:standard-notations}
Let $d \in \Naturals$. For a vector $x\in \Reals^d$, $\norm{x}$ denotes the $\ell_2$ norm of $x$. Let $n,m \in \Naturals$. For a matrix $A \in \Reals^{n\times m}$, $\norm{A}=\sup_{x \in  \Reals^m: \norm{x}\leq 1} \norm{Ax}$ denotes the operator norm, and $\norm{A}_F\triangleq \sqrt{\text{trace}(A^T\cdot A)}$ denotes the Frobenius norm of $A$ where $\text{trace}$ denotes the trace operator. $I_d \in \Reals^{d\times d}$ denotes the identity matrix. $\inner{\cdot}{\cdot}$ denotes the standard inner product in $\Reals^d$. For a convex and closed subset $\parspace \subseteq \Reals^d$, let $\proj : \Reals^d \to \parspace$ be the Euclidean projection operator, given by $\proj(x)=\argmin_{y\in \parspace}\norm{y-x}_2$. For a (measurable) space $\mathcal{R}$, $\ProbMeasures{\mathcal{R}}$ denotes the set of all probability measures on $\mathcal{R}$. Note that the statements in the paper about random variables hold almost surely. We will skip such declarations to aid readability.

\subsection{Properties of zCDP}

\begin{lemma}[{\citealp[][Prop.~1.3]{bun2016concentrated}}]
Assume we have a randomized mechanism $\Alg:\dataspace \to \ProbMeasures{\mathcal{R}}$ that satisfies $\rho$-zCDP, then for every $\delta >0$,  $\Alg$ is $(\rho + 2\sqrt{\rho \log(1/\delta)}, \delta)$-DP.
\end{lemma}
We also repeatedly use the \emph{Gaussian mechanism} which is formalized in the next lemma.
\begin{lemma}[{\citealp[][Lem.~2.5]{bun2016concentrated}}]
\label{lem:gauss-mech}
Let $q:\dataspace^n \to \Reals^d$ be a function. Let the $\ell_2$-sensitivity of $q$ be  $\Delta = \sup_{\trainset,\trainset'}\norm{q(\trainset) - q(\trainset')}$ where the supremum is over all the neighbouring datasets $\trainset,\trainset'$. For every $\rho>0$, define randomized mechanism $\Alg_n$ such that on input $\trainset \in \dataspace^n$, it outputs $\Normal(q(\trainset),\frac{\Delta^2}{2\rho}I_d)$. Then, $\Alg_n$ satisfies $\rho$-zCDP.
\end{lemma}

\section{Appendix of \cref{sec:cubic-newton}}
\label{appx::cubic-newton}
\subsection{Proof of \cref{thm:scvx-guarantee}}

Given a training set $\trainset=(z_1,\dots,z_n) \in \dataspace^n$, our goal is to minimize 
\[
\nonumber
\loss(\ww,\trainset)=\frac{1}{n}\sum_{i\in \range{n}} f(w,z_i).
\]
Since $f$ is a strongly convex function and $\parspace$ is a closed and convex set, there exists a unique $\ww^\star = \argmin_{w\in \parspace}\loss(\ww,\trainset)$.

Let $M \in \Reals$. In each step of the algorithm, we construct a cubic function $\phi: \parspace \times \parspace \to \Reals$ defined as 
\[
\label{eq:cubic-approx-scvx}
\phi_M(v;w) \triangleq \loss(w,\trainset) + \inner{\nabla \loss(w,\trainset)}{v-w} + \frac{1}{2}\inner{\nabla^2 \loss(w,\trainset) (v-w)}{v-w}+\frac{M}{6}\norm{v-w}^3.
\]
We provide a lemma on the properties of $\phi_M(v;w)$.
\begin{lemma}\label{lem:properties-cubic}
Let $f$ be a $\liph$-Lipschitz hessian function. Then, $\phi_M$ in \cref{eq:cubic-approx-scvx} satisfies the following properties:
\begin{enumerate}
    \item For every $M\geq 0$ and  $w,v \in \parspace$ such that  $v\neq w$,
\[
\nonumber
\nabla_v^2 \phi_M(v;w) = \nabla^2 \loss(w,\trainset)  + \frac{M}{2} \norm{v-w} I_d + \frac{M}{2\norm{v-w}}(v-w)(v-w)^T.
\]
Therefore, $\nabla_v^2 \phi_M(v;w) \succcurlyeq  \lambda_{\text{min}}(\nabla^2 \loss(w,\trainset)) I_d   + M\norm{w-v}I_d$ where $\lambda_{\text{min}}(\nabla^2 \loss(w,\trainset))$ denotes the minimum eigenvalue of $\nabla^2 \loss(w,\trainset)$. 
\item For every $M\geq \liph$ and $v,w \in \parspace$, 
\[
\nonumber
\loss(v,\trainset)  \leq \phi_M(v;w).
\]
\item For every $M\in \Reals_{+}$ and $v,w \in \parspace$,
\begin{align*}
    & \phi_M(v;w) \leq \loss(v,\trainset) +\frac{M+\liph}{6}\norm{v-w}^3.
\end{align*}
\end{enumerate}
\begin{proof}
To show the first claim, consider
\[
\nonumber
\nabla_v \phi_M(v;w) = \nabla \loss(w,\trainset)+ \nabla^2 \loss(w,\trainset) (v-w) + \frac{M}{2} \norm{v-w} (v-w).
\]
Then, the hessian of $\phi_M(v;w)$ is given by 
\[
\nonumber
\nabla_v^2 \big( \phi_M(v;w) \big) =  \nabla^2 \loss(w,\trainset) + M\norm{w-v} I_d + \frac{M}{\norm{w-v}}(w-v)(w-v)^T.
\]
Note that $(w-v)(w-v)^T$ is a PSD and rank-1 matrix whose non-zero eigenvalue is given by $\norm{v-w}^2$.

The second and third parts follow from~ \citep[Lemma 1]{nesterov2006cubic} where it is shown for $\liph$-Lipschitz hessian functions we have 
\[
\nonumber
\big|\loss(v,\trainset) -\big( \loss(w,\trainset) + \inner{\nabla \loss(w,\trainset)}{v-w} &+ \frac{1}{2}\inner{\nabla^2 \loss(w,\trainset) (v-w)}{(v-w)} \big)\big|\\
&\leq \frac{\liph}{6}\norm{v-w}^3. \nonumber
\]
The claims are straightforward applications of this inequality.
\end{proof}
\end{lemma}

We can rephrase \cref{eq:cubic-approx-scvx} as follows:
\[
&\phi_M(v;w)= \nonumber \\
& \frac{1}{n}\sum_{i \in \range{n}} f(w,z_i) + \frac{1}{n} \sum_{i \in \range{n}} \inner{\nabla f(w,z_i)}{v-w} + \frac{1}{2n} \sum_{ i\in \range{n}} \inner{\nabla^2 f(w,z_i)}{v-w} + \frac{M}{6} \norm{w-v}^3.
\]
Notice that $\phi_M(v;w)$ is a convex function as it is sum of a quadratic function and a cubic term, i.e., $\frac{M}{6}\norm{v-w}^3$. Also, Part 1 of \cref{lem:properties-cubic} shows that $\phi_M(v;w)$ is a $\mu$-strong convex function since $f$ is a $\mu$-strong convex function. Moreover, the $\ell_2$-sensitivity of $\nabla \phi_M(v;w)$ is $n^{-1}(\lipf + \lipg \diam)$, and $ \phi_M(v;w)$ is $(\lipf + \lipg \diam + \frac{M}{2}\diam^2)$-Lipschitz 
where $\diam$ denotes the diameter of $\parspace$.

First we provide the performance guarantee of the solver of the cubic subproblem in \cref{alg:subp-scvx}.
\begin{lemma} \label{lem:subproblem-scvx}
For every $\beta \in (0, 1)$, $\rho >0$, and $w \in \parspace$ the output of the subproblem solver, denoted by $\hat{v}$ satisfies $\rho T^{-1}$-zCDP and with probability at least $1-\beta$
\[\phi_M(\hat{v};w) - \min_{v\in \parspace}\phi_M(v;w) = O\left(\frac{ d(\lipf + \lipg \diam)^2  T}{\mu \rho n^2} \cdot \log(1/\beta)\right). \nonumber\]
\end{lemma}
\begin{proof}
The privacy analysis is as follows: let the total privacy budget and the number of iterations of Meta Algorithm, i.e., \cref{alg:cubic-private}, denoted by 
$\rho$ and $T$, respectively. We require that the output of the subproblem solver at each iteration satisfies $\rho / T$-zCDP which implies the output of  \cref{alg:cubic-private} satisfies $\rho$-zCDP since the zCDP constant increases linearly with the number of iterations.

Now we provide a detailed proof of the suboptimality gap. To ease the notations consider the following problem. Lets assume that we want to use $\mathsf{DPSolver}$ to minimize the function  $h$ which is $\mu$-strongly convex and  $L$-Lipschitz function whose $\ell_2$-gradient sensitivity is given by $\Delta$.  We are interested in analyzing the suboptimality gap of \cref{alg:subp-scvx} under the condition that the output satisfies $\tilde{\rho}$-zCDP. 

Lets assume we want to run the algorithm for $N$  iterations where $N$ will be determined later. Therefore, we need to make sure that each noisy gradient computation satisfies 
$\tilde{\rho}/N$-zCDP. To do so, the variance of the noise needs to be $ \sigma^2 = \frac{N \Delta^2}{2\tilde{\rho}}$ from \cref{lem:gauss-mech}.
     
Let $\xi_t$  be the noise added to the gradient at iteration $t$. Then, in each step we consider 
$\text{grad}_t + \xi_t$ for noisy gradient. From \citep[Lemma~1]{jin2019short}, we know that  $\norm{\xi_t}$ is a SubGaussian random variable with variance proxy of $c \sigma \sqrt{d}$ where $c$ is a universal constant. We are now ready to use \citep[Thm.~C.3]{harvey2019simple}. \citep[Thm.~C.3]{harvey2019simple} shows that for every $\beta \in (0,1]$  the suboptimality gap with probability at least $1-\beta$
 is given by
 \begin{equation*}
     O\left( \frac{ (L+\sigma \sqrt{d})^2 }{\mu} \frac{\log(1 / \beta)}{N} \right) = O\left( \frac{ (L^2+\sigma^2 d) }{\mu} \frac{\log(1 / \beta)}{N} \right),
 \end{equation*}
 where we simply use  $(a+b)^2\leq 2a^2 +2b^2$ for every $a,b$. Finally, we need to plug in the value of $\sigma$  to obtain that the suboptimality gap:
 \begin{equation*}
     O\left( \frac{ (L^2+\sigma^2 d) }{\mu} \frac{\log(1 / \beta)}{N} \right) = O\left( \left(\frac{L^2}{\mu N} + \frac{\sigma^2 d}{\mu N} \right) \log(1/\beta)\right) =  O\left( \left( \frac{L^2}{\mu N} + \frac{d \Delta^2}{\tilde{\rho}} \right) \log(1/\beta)\right).
 \end{equation*}
 
 Then, by setting the number of iterations to $N = \frac{2L^2 \tilde{\rho}}{\mu d \Delta^2}$, we obtain that  the suboptimality gap is given by $O \left( \frac{d \Delta^2}{\mu \tilde{\rho}} \log(1/\beta) \right)$. 
 
In the context of our paper, $\Delta =n^{-1}(L_0 + L_1 D)$, $L = (L_0 + L_1 D + \frac{M}{2}D^2)$, and $\tilde{\rho} = \rho/T$. Setting these constants proves the lemma.
\end{proof}
We drop the $\trainset$ argument from $\loss(\ww,\trainset)$ to reduce notational clutter. Using Part 2 of \cref{lem:properties-cubic} we can write
\[
\label{eq:main-decompo-scvx}
\loss(\ww_{t+1})-\loss(\ww^\star) \leq \phi_M(\ww_{t+1};\ww_t) - \min_{w \in \parspace}\phi_M(w;\ww_t) +  \min_{w \in \parspace}\phi_M(w;\ww_t) -\loss(\ww^\star).
\]
Since $\phi_M(w;\ww_t)$ as a function of $w$ is a strongly convex function and $\parspace$ is a closed and convex set there exists a unique $\ww_{t+1}^{\star}=\argmin_{w\in \parspace}\phi_M(w;\ww_t)$. 

Fix a $\beta \in (0,1]$ and define the following event
\[
\label{eq:good-event-scvx}
\mathcal{G} = \{\forall t \in \range{T} ~ : ~ \phi_M(\ww_{t};\ww_{t-1}) - \phi_M(\ww_{t}^{\star};\ww_{t-1})  \leq O\Big(\frac{ d(\lipf + \lipg \diam)^2  T}{\mu \rho n^2} \cdot \log(T/\beta)\Big) \triangleq \Delta_0 \}.
\]
We claim that $\Pr(\mathcal{G})\geq 1-\beta$. In each step of the algorithm, we find an approximate minimizer of $\phi_M(\ww;\ww_t)$ using the subproblem solver in \cref{alg:subp-scvx}. The performance guarantee of the subproblem solver is given in \cref{lem:subproblem-scvx} which shows that at each step of the algorithm the excess error in minimizing $\phi_M(\ww;\ww_t)$  is less than $\Delta_0$ with probability greater than $1-\beta/T$. Ergo, a union bound concludes the proof.

Next we provide an upperbound on $\phi_M(\ww^\star_{t+1};\ww_{t}) - \loss(\ww^\star)$ in \cref{eq:main-decompo-scvx}. By the third part of \cref{lem:properties-cubic} we have
\begin{align*}
\phi_M(\ww^\star_{t+1};\ww_t) - \loss(\ww^\star) &\leq \min_{w\in \parspace}\{\loss(w)+\frac{M+\liph}{6}\norm{w-\ww_t}^3 -\loss(w^\star)\}.
\end{align*}
Since $\parspace$ is a convex set and $\ww_t,\ww^\star \in \parspace$, for all $\alpha \in [0,1]$, $(1-\alpha) \ww_t + \alpha \ww^\star \in \parspace$. Therefore, 
\[
&\min_{w\in \parspace}\{\loss(w)+\frac{M+\liph}{6}\norm{w-\ww_t}^3 -\loss(w^\star)\} \nonumber \\
&\leq \min_{\alpha_t\in [0,1]}\{\loss((1-\alpha_t)\ww_t + \alpha_t \ww^\star) )+\frac{M+\liph}{6} \alpha_t^3 \norm{\ww_t-\ww^\star}^3 -\loss(w^\star)\} \nonumber.
\]
By the convexity of $\loss$ we have $\loss((1-\alpha_t)\ww_t + \alpha_t \ww^\star) ) - \loss(w^\star)\leq \loss(\ww_t)- \loss(\ww^\star) - \alpha_t (\loss(\ww_t)- \loss(\ww^\star)) $. Also, strong convexity implies that  $(\frac{2}{\mu}(\loss(\ww_t)-\loss(\ww^\star)))^\frac{3}{2}\geq \norm{\ww_t-\ww^\star}^3$\citep{nesterov1998introductory}. Thus,
\[
&\phi_M(\ww^\star_{t+1};\ww_t) - \loss(\ww^\star)  \nonumber\\
&\leq \min_{\alpha_t\in [0,1]}\{\loss(\ww_t)-\loss(\ww^\star) -\alpha_t(\loss(\ww_t)-\loss(\ww^\star)) +\alpha_t^3\frac{M+\liph}{6}(\frac{2}{\mu}(\loss(\ww_t)-\loss(\ww^\star)))^\frac{3}{2} \} \label{eq:opt-alpha-scvx}
\]
In the rest of the proof, under the event $\mathcal{G}$, we provide a convergence analysis. 

Let $\lambda=(\frac{3}{M+\liph})^2 (\frac{\mu}{2})^3 $ and $q_t = \lambda^{-1} \big(\loss(\ww_t)-\loss(\ww^\star)\big)$. Then, under the event $\mathcal{G}$ and by \cref{eq:opt-alpha-scvx}, we can rephrase \cref{eq:main-decompo-scvx} as
\[
\label{eq:decompos-cvx-simple}
q_{t+1} \leq \lambda^{-1} \Delta_0 +   \min_{\alpha_t\in [0,1]}\{ q_t - \alpha_t q_t + \frac{1}{2} \alpha_t^3 q_t^{\frac{3}{2}} \}. 
\]
Let $\alpha_t^\star =\argmin_{\alpha_t\in [0,1]}\{ q_t - \alpha_t q_t + \frac{1}{2} \alpha_t^3 q_t^{\frac{3}{2}} \} = \min\{\sqrt{\frac{2}{3\sqrt{q_t}}},1\}$. 

First, consider the case that $q_t \geq 4/9$ so that $\alpha^\star_t =  \sqrt{\frac{2}{3\sqrt{q_t}}}$. We can rephrase \cref{eq:decompos-cvx-simple} as follows
\[
\label{eq:recur-phase1}
q_{t+1} \leq \lambda^{-1} \Delta_0 + q_t - (\frac{2}{3})^\frac{3}{2} q_t^{\frac{3}{4}}. \quad \text{(Phase I)}
\]
In the second case, i.e., $q_t < 4/9$, we have $\alpha^\star_t =  1$  \cref{eq:decompos-cvx-simple} is given by
\[
\label{eq:recur-phase2}
q_{t+1} \leq \lambda^{-1} \Delta_0 + \frac{1}{2} q_t^\frac{3}{2}.  \quad \text{(Phase II)}
\]
Assume that $q_0 \geq 4/9$. We will show that, under the event $\mathcal{G}$, $\{q_t\}_{t\in \range{T}}$ is a decreasing sequence, and as a result there exists $T_1\in \Naturals$, \emph{independent of $n$}, such that $q_{t} <4/9$ for $t \geq T_1^\star$, and as a result $\alpha^\star_t =1$ for $t \geq T_1^\star$.

To prove the convergence in Phase I (see \cref{eq:recur-phase1}), we follow the techniques of \citet{nesterov2006cubic}. Let $ \displaystyle \tilde{q}_t=\frac{9q_t}{4}$, and assume $ \displaystyle \Delta_0 \leq  \frac{4\lambda}{27}$. Then, we can rephrase the recursion for Phase I as follows:
\[
\tilde{q}_{t+1} &\leq \frac{9 \Delta_0}{4 \lambda} + \tilde{q}_t - \frac{2}{3}\tilde{q}_t^{\frac{3}{4}} \nonumber\\
&\leq \tilde{q}_t - \frac{1}{3}\tilde{q}_t^{\frac{3}{4}}, \nonumber
\]
where the last step follows from $ \displaystyle \tilde{q}_t\geq 1$ and $\displaystyle \frac{9\Delta_0}{4\lambda} \leq  \frac{1}{3} \leq  \frac{\tilde{q}^{\frac{3}{4}}_t}{3} $. It also shows that provided that $\tilde{q}_t\geq1$, $\tilde{q}_{t+1}\leq \tilde{q}_t$.

Using induction, it is straightforward to show that in Phase 1
\[
\nonumber
\frac{9q_t}{4} \leq \big[\big(\frac{9q_0}{4}\big)^{\frac{1}{4}}-\frac{t}{12}\big]^4. \quad \text{(Phase I)}
\]
This result implies that after $T_1^\star$ iterations where
\[
\label{eq:numrecur-phase1}
T_1^\star \leq O \big( \frac{\sqrt{M+\liph}}{\mu^{\nicefrac{3}{4}}} (\loss(\ww_0)-\loss(\ww^\star))^{\frac{1}{4}}\big),
\]
we have $q_{T_1^\star} < \frac{4}{9}$, and we enter Phase II (see \cref{eq:recur-phase2}).  

Next, we analyze Phase II in which the recursion is given by 
\[
\nonumber
q_{t+1} &\leq \lambda^{-1} \Delta_0 + \frac{1}{2} q_t^\frac{3}{2}.
\]
Using \cref{lem:conv-seq-exp}, we obtain that after  $\Theta(\log(\log(\frac{\lambda}{\Delta_0})))$ iterations we have $O(\lambda^{-1}\Delta_0)$. Therefore, the number of iterations to achieve the minimum excess error in Phase II
\[
\label{eq:numrecur-phase2}
T_2^{\star} = \tilde{\Theta}(\log\log(\frac{n}{\sqrt{\rho \log(1/\beta)d}})) .
\]
Finally, the excess error is given by 
\[
\loss(\ww_T)-\loss(\ww^\star)=\tilde{O} \Big(\frac{ d(\lipf + \lipg \diam)^2 }{\mu \rho n^2} \cdot \log(1/\beta) \cdot (T_1^\star + T_2^\star)\Big),
\]
where $T = T_1^\star + T_2^\star$ and $T_1^\star$ and $T_2^\star$ are given by \cref{eq:numrecur-phase1} and \cref{eq:numrecur-phase2}, respectively.
\begin{lemma}
\label{lem:conv-seq-exp}
Let $\beta_0>0$ and define the sequence $a_{t+1} \leq  \beta_0 + \frac{1}{2} a_t^{3/2}$ where $a_{0}\leq \frac{16}{9}$. Then, after $T = \Theta(\log \log(\frac{1}{\beta_0})) $ we have $a_{T}=O(\beta_0)$.
\end{lemma}
\begin{proof}
Without loss of generality, assume $a_{t+1} =  \beta_0 + \frac{1}{2} a_t^{3/2}$. We define another sequence $\{b_t\}_{t\in \Naturals}$ as follows: $b_0=a_0$ and $b_{t+1}=\frac{3}{4}(b_{t})^{\frac{3}{2}}$. By induction one can easily prove that for every $t \in \Naturals$ such that $\beta_0 \leq \frac{1}{4}(b_{t})^{\frac{3}{2}}$, we have $b_{t+1}\geq a_{t+1}$. Then, we can write
\[
\nonumber
b_{t+1}=\frac{3}{4}(b_{t})^{\frac{3}{2}} \Leftrightarrow \frac{9}{16}b_{t+1}=\left(\frac{9}{16}b_{t}\right)^{\frac{3}{2}}.
\]
Therefore, we obtain that $\log(\frac{9}{16} b_t) = (\frac{3}{2})^t\log(\frac{9}{16}b_0)$. We want to find $T$ such that $\beta_0 \leq \frac{1}{4}(b_T)^{\frac{3}{2}}\leq 2\beta_0$ which is equivalent to $\log(\frac{8\beta_0}{3})\leq \log(b_T) \leq \log(\frac{16\beta_0}{3})$. Then, by some simple manipulations we can see that 
$
T = \Theta(\log(\log(\frac{1}{\beta_0}))).
$
Therefore, we have $b_{T+1}=O(\beta_0)$. Also, by the construction, $a_{T+1}\leq b_{T+1}=O(\beta_0)$.
\end{proof}

\subsection{Private Accelerated Nestrov's Method}
\label{subsec:appx-nestrov}
\newcommand{\xag}{\ww^{\text{ag}}}
\newcommand{\xte}{\ww}
\newcommand{\xmd}{\ww^{\text{md}}}
In this section, we present a DP variant of the accelerated Nestrov's Method. The proof ideas are based on \citep{nesterov1998introductory,nemirovskirobust,ghadimiacc}.
\begin{algorithm}[H]
\caption{Private Accelerated Nestrov's Method for $\lipf$-Lipschitz, $\lipg$-smooth convex function on a bounded feasible set $\parspace$ with diameter $\diam$.} \label{alg:nestrov}
\begin{algorithmic}[1]
\State Input: $\ww_0 \in \parspace$, Privacy Guarantee $\rho$-zCDP.
\State $T =\Theta\Big(\big(\frac{\diam \sqrt{\rho} n}{\lipf}\big)^{\nicefrac{1}{2}}\Big)$, $\sigma^2 = \frac{\lipf^2 T}{2\rho n^2}$.
\State $\xag_0 = \xte_0 \in \parspace$
\State $\alpha_t = \frac{2}{t+1},\gamma_t = \frac{4\gamma}{t(t+1)}$ where $\gamma = 2\lipg$.
\For{$\displaystyle t=1,\dots,T$}
    \State $\xmd_t = \xag_{t-1} + \alpha_t ( \xte_t - \xag_{t-1}) $
    \State $G_t = \nabla \loss(\xmd_t,\trainset ) + \Normal(0,\sigma^2 I_d) $
    \State $\xte_t = \proj(\xte_{t-1} - \frac{\alpha_t}{\gamma_t}G_t)=\argmin_{v \in \parspace} \{ \alpha_t \inner{G_t}{v} + \frac{\gamma_t}{2}\norm{v-\xte_{t-1}}^2\}$
    \State $ \xag_t  = \alpha_t \xte_t + (1-\alpha_t) \xag_{t-1} $
\EndFor
\State Return $\xag_T$
\end{algorithmic}
\end{algorithm}
\begin{theorem}
\label{thm:nestrov}
Let $f$ be a convex, $\lipf$-Lipschitz, and $\lipg$-smooth. Also, assume that $\parspace \subseteq \Reals^d$ is a convex set and has finite diameter $\diam$. Let $\ww^\star \in \argmin_{w\in \parspace}\loss(\ww,\trainset)$. Then, for every $n \in \Naturals$, $\trainset \in \dataspace^n$, and $\rho>0$, the output of \cref{alg:nestrov}, i.e., $\xag_T$, satisfies $\rho$-zCDP and 
\*[
\EE[\loss(\xag_T,\trainset)-\loss(\ww^\star,\trainset)]=O\bigg( \frac{\lipf \diam\sqrt{d}}{n\sqrt{\rho}} \bigg).
\]
Also, the oracle complexity of \cref{alg:nestrov} is
\*[
T  = \Theta\left(\sqrt{\frac{\diam  n\sqrt{\rho} }{\lipf}}\right).
\]
\end{theorem}
\begin{proof}
First, we start with the privacy proof. The $\ell_2$-sensitivity of $\nabla \loss (\ww,\trainset)$ for every $\ww \in \parspace$ is given by $\frac{\lipf}{n}$. Therefore, by the composition properties of zCDP in \citep[Lem.~2.3]{bun2016concentrated}
and the zCDP analysis of the Gaussian mechansim in \citep[Lem.~2.5]{bun2016concentrated}, it is straightforward to show that $\xag_T$ satisfies $\rho$-zCDP. 

Then, we analyze the excess error. For every $v \in \parspace$, by the smoothness of $\loss$ we can write
\[
\loss(\xag_t) &\leq \loss(v) + \inner{\nabla \loss(v)}{\xag_t - v} + \frac{\lipg}{2}\norm{\xag_t - v}^2 \nonumber\\
            &=  \loss(v) + \alpha_t\inner{\nabla \loss(v)}{\xte_{t} - v} + (1-\alpha_t) \inner{\nabla \loss(v)}{\xag_{t-1} - v}  + \frac{\lipg}{2}\norm{\xag_t - v}^2 \nonumber\\
            & \leq \alpha_t(\loss(v) + \inner{\nabla \loss(v)}{\xte_t - v}) + (1-\alpha_t) \loss(\xag_{t-1}) + \frac{\lipg}{2}\norm{\xag_t - v}^2\label{eq:nes-first-step}.
\]
Here, the second step is by definition of $\xag_t$, and the last step is by  convexity of $\ell$ which implies $\ell(v)+\inner{\nabla \loss(v)}{\xag_{t-1}-v}\leq \loss(\xag_{t-1})$.

Note that 
\*[
\xag_t - \xmd_t &= \alpha_t \xte_t + (1-\alpha_t) \xag_{t-1}- (\xag_{t-1} + \alpha_t ( \xte_{t-1} - \xag_{t-1}))\\
&= \alpha_t (\xte_t - \xte_{t-1}).
\]
In the next step, we substitute $v=\xmd_t$ in \cref{eq:nes-first-step} to obtain
\[
\loss(\xag_t)& \leq \alpha_t(\loss(\xmd_t) + \inner{\nabla \loss(\xmd_t)}{\xte_t - \xmd_t}) + (1-\alpha_t) \loss(\xag_{t-1}) + \frac{\lipg}{2}\norm{\xag_t - \xmd_t}^2 \nonumber\\
&= (1-\alpha_t) \loss(\xag_{t-1}) + \alpha_t(\loss(\xmd_t) + \inner{\nabla \loss(\xmd_t)}{\xte_t - \xmd_t}) + \frac{\lipg \alpha_t^2}{2}\norm{\xte_t - \xte_{t-1}}^2 \nonumber\\
&= (1-\alpha_t) \loss(\xag_{t-1}) + \alpha_t(\loss(\xmd_t) + \inner{\nabla \loss(\xmd_t)}{\xte_t - \xmd_t}) + \frac{\gamma_t}{2}\norm{\xte_{t}-\xte_{t-1}}^2 \nonumber\\
& \hspace{2em} - \frac{\gamma_t -\lipg \alpha_t^2}{2}\norm{\xte_{t}-\xte_{t-1}}^2. \label{eq:nes-second-step}
\]
Let $\xi_t = G_t - \nabla \loss(\xmd_t)$. Notice that the projection step can be written as $\xte_t = \argmin_{v \in \parspace} \{ \alpha_t \inner{G_t}{v-\xmd_t} + \frac{\gamma_t}{2}\norm{v-\xte_{t-1}}^2\}$. The function $g:\parspace \to \Reals$,  $g(v) =\alpha_t \inner{G_t}{v-\xmd_t} + \frac{\gamma_t}{2}\norm{v-\xte_{t-1}}^2$ is a $\gamma_t$ strongly convex function. By the optimally condition for strongly convex functions we can write, for every $v\in \parspace$
\[
&\alpha_t \inner{G_t}{\xte_t - \xmd_t} + \frac{\gamma_t}{2}\norm{\xte_t - \xte_{t-1}}^2 \nonumber \\
&\leq \alpha_t \inner{G_t}{v- \xmd_t} + \frac{\gamma_t}{2}\norm{v-\xte_{t-1}}^2 -   \frac{\gamma_t}{2}\norm{v-\xte_t}^2 \nonumber\\
&=\alpha_t \inner{\nabla \loss(\xmd_t)}{v- \xmd_t} + \alpha_t \inner{\xi_t}{v- \xmd_t} + \frac{\gamma_t}{2}\norm{v-\xte_{t-1}}^2 -   \frac{\gamma_t}{2}\norm{v-\xte_t}^2. \label{eq:nes-third-step}
\]
By replacing $v = \ww^\star$ where $\ww^\star \in \argmin \loss(\ww,\trainset)$ in \cref{eq:nes-third-step}, we obtain
\[
&\alpha_t(\loss(\xmd_t) + \inner{\nabla \loss(\xmd_t)}{\xte_t - \xmd_t}) + \frac{\gamma_t}{2}\norm{\xte_{t}-\xte_{t-1}}^2  \nonumber\\
& = \alpha_t \loss(\xmd_t) + \alpha_t \inner{G_t}{\xte_t - \xmd_t} -\alpha_t \inner{\xi_t}{\xte_t - \xmd_t}+ \frac{\gamma_t}{2}\norm{\xte_{t} - \xte_{t-1}}^2 \nonumber\\
&\leq \alpha_t (\loss(\xmd_t) +  \inner{\nabla \loss(\xmd_t)}{\ww^\star - \xmd_t}) + \alpha_t \inner{\xi_t }{\ww^\star - \xte_t} + \frac{\gamma_t}{2}\norm{\ww^\star - \xte_{t-1}}^2 - \frac{\gamma_t}{2}\norm{\ww^\star - \xte_t}^2 \nonumber\\
&=\alpha_t \loss(\ww^\star) + \alpha_t \inner{\xi_t }{\ww^\star - \xte_t} + \frac{\gamma_t}{2}\norm{\ww^\star - \xte_{t-1}}^2 - \frac{\gamma_t}{2}\norm{\ww^\star - \xte_t}^2. \label{eq:nes-fourth-step}
\]

From \cref{eq:nes-third-step} and \cref{eq:nes-fourth-step} we can write
\*[
&\loss(\xag_t) - \loss(\ww^\star) \leq  (1-\alpha_t) (\loss(\xag_{t-1}) - \loss(\ww^\star)) + \alpha_t \inner{\xi_t }{\ww^\star - \xte_t}  \\
&+ \frac{\gamma_t}{2}\norm{\ww^\star - \xte_{t-1}}^2 - \frac{\gamma_t}{2}\norm{\ww^\star - \xte_t}^2 - \frac{\gamma_t - \lipg \alpha_t^2}{2}\norm{\xte_{t}-\xte_{t-1}}^2.
\]
Note that $\xi_t$ is independent from  the history up to time $t-1$, i.e., $\{(\xmd_i,\xte_i,\xag_i)\}_{i=1}^{t-1}$. Therefore, $\EE[\inner{\xi_t}{\xte_{t-1}}]=0$ as $\xi_t \sim \Normal(0,\sigma^2 I_d)$. Using this observation, we can write
\[
&\EE[ \alpha_t \inner{\xi_t }{\ww^\star - \xte_t} - \frac{\gamma_t -L \alpha_t^2}{2}\norm{\xte_{t}-\xte_{t-1}}^2] \nonumber\\
&= \EE[ \alpha_t \inner{\xi_t }{\ww^\star - \xte_{t-1}} + \alpha_t  \inner{\xi_t }{\xte_{t} - \xte_{t-1}} - \frac{\gamma_t - \lipg \alpha_t^2}{2}\norm{\xte_{t}-\xte_{t-1}}^2] \nonumber\\
            & \leq \EE[ \alpha_t \norm{\xi_t} \norm{\xte_{t} - \xte_{t-1}}  - \frac{\gamma_t - \lipg \alpha_t^2}{2}\norm{\xte_{t}-\xte_{t-1}}^2 ] \nonumber \\
            & \leq \frac{\alpha_t ^2 }{\gamma_t - \lipg \alpha_t^2} \EE[ \norm{\xi_t}^2] \nonumber\\
            & = \frac{\alpha_t ^2 }{\gamma_t - \lipg \alpha_t^2} \cdot \sigma^2 I_d,
\]
where the second step follows from Cauchy–Schwarz inequality.
Therefore, we obtain that
\[
&\EE[\loss(\xag_t) - \loss(\ww^\star)] \nonumber \\
&\leq (1-\alpha_t) \EE[\loss(\xag_{t-1}) - \loss(\ww^\star)] + \frac{\alpha_t ^2 (\sigma^2 d) }{\gamma_t - \lipg \alpha_t^2} +  \frac{\gamma_t}{2}\norm{\ww^\star - \xte_{t-1}}^2 - \frac{\gamma_t}{2}\norm{\ww^\star - \xte_t}^2. \label{eq:nes-recur-pre}
\] 
Let  $ \displaystyle \Gamma_t = \begin{cases}
 1  &  t=1 \\
  (1-\alpha_t)\Gamma_{t-1} & t\geq 2\\
\end{cases} = \frac{2}{t(t+1)}$. Note that since $\gamma = 2\lipg$, we have $\displaystyle \gamma_t - \lipg \alpha_t^2 = \frac{4\gamma}{t(t+1)} - \frac{4\lipg}{(t+1)^2} = \frac{2 \gamma}{(t+1)^2}$, and $\displaystyle \frac{\gamma_t}{\Gamma_t}=2\gamma$. Consider dividing both side of \cref{eq:nes-recur-pre} by $\Gamma_t$ and summing up from $1$ to $T$ to obtain
\*[
&\frac{\EE[\loss(\xag_T) - \loss(\ww^\star)]}{\Gamma_T} \leq   \sum_{\tau=1}^{T}\frac{\gamma_\tau}{2\Gamma_\tau} \big( \norm{\ww^\star - \xte_{\tau-1}}^2 - \norm{\ww^\star - \xte_{\tau}}^2\big) +  \sigma^2 d \cdot \sum_{\tau=1}^{T} \frac{1}{\Gamma_\tau} \cdot \frac{\alpha_\tau ^2 }{\gamma_\tau - \lipg \alpha_\tau^2}.
\]
Note that $\frac{\gamma_t}{\Gamma_t}=2\gamma$. Therefore, 
\*[
\sum_{\tau=1}^{T}\frac{\gamma_\tau}{2\Gamma_\tau} \big( \norm{\ww^\star - \xte_{\tau-1}}^2 - \norm{\ww^\star - \xte_{\tau}}^2\big) = \gamma \norm{\xte_0 - \ww^\star}^2\leq \gamma \diam^2.
\]
Then, for the last term consider 
\*[
\sum_{\tau=1}^T\frac{1}{\Gamma_\tau} \cdot \frac{\alpha_\tau ^2}{\gamma_\tau - \lipg \alpha_\tau^2} = \frac{1}{3\gamma} T (T+1) (T+2).
\]
By by combining all the previous steps, we get the following bound on the expected excess error
\[
\EE[\loss(\xag_T) - \loss(\ww^\star)] &\leq \frac{2}{T(T+1)} \gamma \diam^2 + 2\sigma^2 d \frac{T+2}{3\gamma} \nonumber\\ 
& =  \frac{2}{T(T+1)} \gamma \diam^2 +  d\frac{2\lipf^2 T(T+2)}{6\gamma\rho n^2}.  \label{eq:nes-excess-error-before-opt}
\]
Finally, optimizing \cref{eq:nes-excess-error-before-opt} over $T$, we conclude that with at most $T$ oracle calls where 
\*[
T  = \Theta\bigg(\big(\frac{\diam^2\rho n^2}{d \lipf^2}\big)^{\nicefrac{1}{4}}\bigg),
\]
the achievable excess error is given by
\*[
\EE[\loss(\xag_T) - \loss(\ww^\star)] = O\big( \frac{\lipf \diam \sqrt{d}}{n\sqrt{\rho}} \big).
\]
\end{proof}

\section{Appendix of \cref{sec:practical-alg}}
\label{appx:practical-alg}
\subsection{Proof of \cref{lem:quad-ub}}
\label{lem:subg-ub}
We begin the proof by a lemma from \citep{arbel2020strict}.
\begin{lemma}[{\citealp[][Prop.~4.1]{arbel2020strict}}]
\label{lem:subg-ub}
For all $0<\mu<1$ and $\lambda\in \Reals$, we have
\[
\nonumber
\frac{2}{\lambda^2}  \big( \log(\mu \exp(\lambda) + 1 - \mu) -\mu \lambda \big) \leq \frac{\nicefrac{1}{2}-\mu}{\log(\nicefrac{1}{\mu}-1)}.
\]
\end{lemma}
We will use the following reformulation of \cref{lem:subg-ub}. Let $\alpha \in \Reals$ and $\beta \in \Reals$  be two constants. Then, substitute $\mu=\frac{\exp(\alpha)}{1+\exp(\alpha)}$ and $\lambda=\beta - \alpha$ in \cref{lem:subg-ub}. By some simple manipulations we obtain that $\forall \beta ,\alpha \in \Reals$
\[
\label{eq:ll-ub-onedim}
\log(1+\exp(\beta)) \leq \log(1+\exp(\alpha)) + \frac{ \beta-\alpha}{1+\exp(-\alpha)} +  \begin{cases}
     \frac{\exp(\alpha)-1}{4\alpha(\exp(\alpha)+1)} ( \beta-\alpha)^2 & \alpha \neq 0\\
     \frac{ (\beta-\alpha)^2}{4} & \alpha = 0
\end{cases}.
\]
Finally, let $w,v,x \in \Reals^d$ and $y \in \{-1,+1\}$, by substituting $\alpha = \inner{-yx}{v}$ and $\beta = \inner{-yx}{w}$, we obtain the stated result in \cref{lem:quad-ub}.

\subsection{Comparison of the Approximations}
\label{subsec:appx-comparison}
Consider $f: \Reals \to \Reals,~ f(x) = \log(1+\exp(x))$ which can be seen as a logistic loss in one dimension. \cref{fig:compare-ub} compares the three approaches for the quadratic approximation: second-order Taylor approximation, our upper bound in \cref{lem:quad-ub}, and the upper bound based on  smoothness. As can be seen the upper bound in \cref{lem:quad-ub} provides tighter approximation compared to the upper bound based on smoothness. Also, the second-order Taylor approximation is not an upper bound on the function.

\begin{figure}
    \centering
    \includegraphics[scale=0.4]{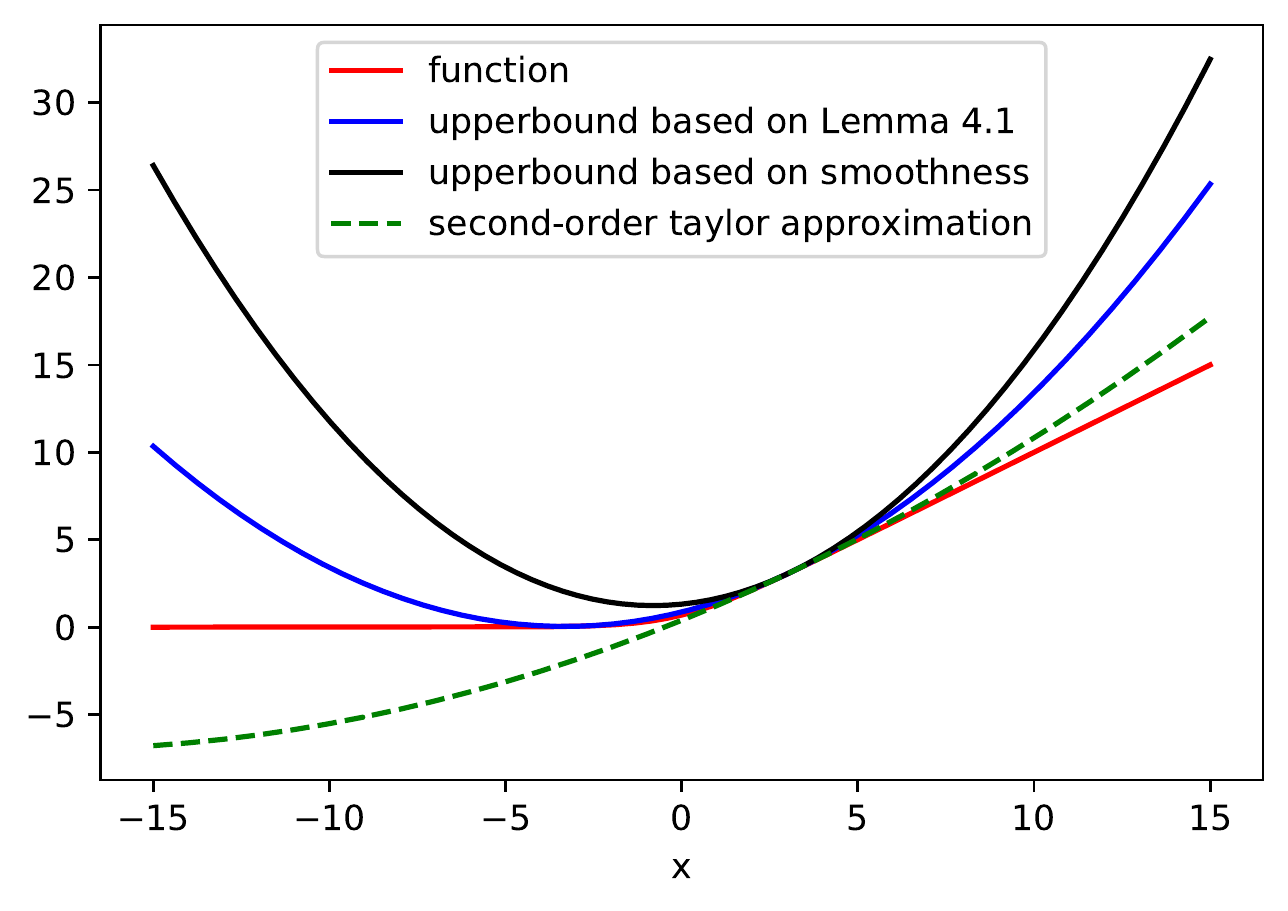}
    \caption{Comparison between the approximation of the logistic loss function}
    \label{fig:compare-ub}
\end{figure}
\subsection{Privacy proof of \cref{alg:main-opt} for $\mathsf{add}$}
\label{subsec:appx-proof-add}
\begin{theorem}\label{thm:sens-add}
Assume in \cref{alg:main-opt} we choose $\mathsf{add}$ for the SOI modification. Then, for every training set $\trainset \in (\Reals^d \times \{-1,+1\})^n$, $\ww_0 \in \parspace$, $\lambda_0>0$, $T\in \Naturals$, $\rho \in \Reals_{+}$, and $\theta \in (0,1)$, by setting 
\[
\nonumber
\sigma_1 = \frac{\sqrt{T}}{n\sqrt{2 \rho (1-\theta)}}, \quad \sigma_2 = \frac{\sqrt{T}}{(4n \lambda_0^2 + \lambda_0)\sqrt{2 \rho \theta}},
\]
$\ww_T$ satisfies $\rho$-zCDP. 
\end{theorem}

\begin{proof}
For the privacy analysis, we assume a two stage procedure. The loss function is $1$-Lipschitz. Therefore, by setting 
$$
\sigma_1 = \frac{\sqrt{T}}{n\sqrt{2 \rho (1-\theta)}},
$$
the mechanism in \cref{line:main-alg-grad-noise} of \cref{alg:main-opt}  satisfies $(1-\theta)\frac{\rho}{T}$-zCDP by \cref{lem:gauss-mech}. 

Notice that $\Psi_{\lambda_0}(H(\ww_t,\trainset),\mathsf{add}) =H(\ww_t,\trainset) + \lambda_0 I_d$. We need to bound the $\ell_2$ sensitivity of the search direction which is given by
\[
&\sup_{\trainset \in (\Reals^d \times \{-1,1\})^{n} }\sup_{ \substack{z_{n+1}=(x_{n+1},y_{n+1})\in \Reals^d \times \{-1,1\}:\\ \norm{x_{n+1}}\leq 1}} \nonumber \\ 
&\norm{\big[H(\ww_t,\trainset) + \lambda_0 I_d\big]^{-1} \tilde{g}_t - \big[H(\ww_t,\trainset) + \frac{1}{n} H(w_t,z_{n+1}) + \lambda_0 I_d\big]^{-1} \tilde{g}_t}.
\]
By the definition of the operator norm, we have
\begin{align*}
&\norm{\big[H(\ww_t,\trainset) + \lambda_0 I_d\big]^{-1} \tilde{g}_t - \big[H(\ww_t,\trainset) + \frac{1}{n} H(w_t,z_{n+1}) + \lambda_0 I_d\big]^{-1} \tilde{g}_t}\\ 
&\leq \norm{\big[H(\ww_t,\trainset) + \lambda_0 I_d\big]^{-1}  - \big[H(\ww_t,\trainset) + \frac{1}{n} H(w_t,z_{n+1}) + \lambda_0 I_d\big]^{-1} }\norm{\tilde{g}_t}.
\end{align*}
Let $A \triangleq H(\ww_t,\trainset) + \lambda_0 I_d$. 
For both type of the \soi, we have  $H(\ww_t,\trainset) + \frac{1}{n} H(w_t,z_{n+1}) + \lambda_0 I_d = A + \beta x_{n+1}x_{n+1}^\top$, where $\beta$ only depends on $x_{n+1},\ww_t$  and $\beta \leq \frac{1}{4n}$ (see \cref{rem:rank-bound}).

We can drop the subscript $n+1$ and  rephrase the problem as follows
\[
\sup_{x\in \mathbb{R}^d:\norm{x}\leq 1}\norm{(A+ \beta xx^\top)^{-1}-A^{-1}}.
\]
We begin by applying the Sherman–Morrison formula \citep{golub2013matrix} to $(A + \beta xx^\top)^{-1}$ to obtain 
\[
\nonumber
(A+ \beta xx^\top)^{-1} - A^{-1} = - \frac{ \beta A^{-1}xx^\top A^{-1}}{1+ \beta x^\top A^{-1}x}.
\]
$A$ is a PSD matrix. Let the eigenvalue decomposition of $A$ be $A=\sum_{i \in \range{d}} \lambda _i u_i u_i^\top = U \Lambda U^\top$ where $\Lambda=\text{diag}(\lambda_1,\dots,\lambda_d)$. Using this representation, we can write 
\begin{align*}
\sup_{x\in \mathbb{R}^d:\norm{x}\leq 1}\norm{(A+\beta xx^\top)^{-1}-A^{-1}} 
&= \sup_{x\in \mathbb{R}^d:\norm{x}\leq 1}   \frac{\norm{\beta A^{-1}xx^\top A^{-1}}}{1+\beta x^\top A^{-1}x}\\
&= \sup_{x\in \mathbb{R}^d:\norm{x}\leq 1}  \frac{\norm{ \beta U \Lambda^{-1} U^\top xx^\top U \Lambda^{-1} U^\top}}{1+ \beta x^\top U \Lambda^{-1} U^\top x}.
\end{align*}
Then, consider the change of variable to $v=U^\top x$:
\begin{align*}
\sup_{x\in \mathbb{R}^d:\norm{x}\leq 1}\norm{(A+\beta xx^\top)^{-1}-A^{-1}}&=\sup_{v\in \mathbb{R}^d:\norm{v}\leq 1}    \frac{ \norm{\beta U \Lambda^{-1} v v^\top \Lambda^{-1} U^\top}}{1+ \beta  v^\top\Lambda^{-1} v}\\
&= \sup_{v\in \mathbb{R}^d:\norm{v}\leq 1}   \frac{ \beta\norm{  U \Lambda^{-1} v v^\top \Lambda^{-1} U^\top}}{1+ \beta v^\top\Lambda^{-1} v}.
\end{align*}
Notice that $ U \Lambda^{-1} v v^\top \Lambda^{-1} U^\top$ is a rank-one matrix. For a rank-1 matrix, the operator norm is given by its non-zero eigenvalue. Thus, 
\begin{align*}
  \norm{U \Lambda^{-1} v v^\top \Lambda^{-1} U^\top} &= \norm{U \Lambda^{-1} v}_2^2\\
    &=\norm{\Lambda^{-1} v}_2^2.
\end{align*}
The last step follows from the fact that $U$ is an orthonormal matrix.
Therefore, by combining the previous representations we obtain
\[
\nonumber
\sup_{x\in \mathbb{R}^d:\norm{x}\leq 1}\norm{(A+ \beta xx^\top)^{-1}-A^{-1}}=  \sup_{v\in \mathbb{R}^d:\norm{v}\leq 1} \frac{\beta\norm{\Lambda^{-1} v}_2^2}{1+\beta v^\top\Lambda^{-1} v}.
\]
For every $a>0$,  define $h:\Reals \to \Reals,h(x) \triangleq \frac{x}{1+ax}$ and notice that $h$ is increasing for $x>0$. Using this fact and by considering  $v^\top\Lambda^{-1} v > 0$ and $0<\beta\leq \frac{1}{4n}$, we obtain 
\[
\sup_{v\in \mathbb{R}^d:\norm{v}\leq 1} \frac{\beta\norm{\Lambda^{-1} v}_2^2}{1+\beta v^\top\Lambda^{-1} v} &\leq \sup_{v\in \mathbb{R}^d:\norm{v}\leq 1} \sup_{\beta \in [0,\frac{1}{4}]}\frac{\beta\norm{\Lambda^{-1} v}_2^2}{1+\beta v^\top\Lambda^{-1} v} 
 \nonumber \\
&\leq \sup_{v\in \mathbb{R}^d:\norm{v}\leq 1}  \frac{\frac{1}{4n}\norm{\Lambda^{-1} v}_2^2}{1+\frac{1}{4n} v^\top\Lambda^{-1} v} \nonumber\\
&=\sup_{v\in \mathbb{R}^d:\norm{v}\leq 1} \frac{ \sum_{i=1}^{d} \frac{1}{\lambda_i^2} v_i^2}{4n+\sum_{i=1}^{d} \frac{1}{\lambda_i} v_i^2 } \nonumber.
\]
Notice that by the definition of $A$, for $i \in \range{d}$,  $\lambda_0 \leq \lambda_i$. Therefore, 
\[
\nonumber
\frac{ \sum_{i=1}^{d} \frac{1}{\lambda_i^2} v_i^2}{4n+\sum_{i=1}^{d} \frac{1}{\lambda_i} v_i^2 } \leq \frac{1}{\lambda_0} \frac{ \sum_{i=1}^{d} \frac{1}{\lambda_i} v_i^2}{4n+\sum_{i=1}^{d} \frac{1}{\lambda_i} v_i^2 }.
\]
 Then, note that for every $v\in \Reals^d$ such that $\norm{v}\leq 1$, we have $
\sum_{i=1}^{d} \frac{1}{\lambda_i} v_i^2 \leq \frac{1}{\lambda_0}$. Also,  $h:\Reals \to \Reals, h(x)\triangleq \frac{x}{4n+x}$ is increasing for $x>0$. Thus, using these two facts we obtain
\[
\nonumber
\frac{1}{\lambda_0} \frac{ \sum_{i=1}^{d} \frac{1}{\lambda_i} v_i^2}{4n+\sum_{i=1}^{d} \frac{1}{\lambda_i} v_i^2 }\leq \frac{1}{4n \lambda_0^2 + \lambda_0}.
\]
Therefore, we showed that 
\[
\sup_{ \substack{z_{n+1} = (x_{n+1},y_{n+1}):\\ \norm{x_{n+1}}\leq 1}} &\norm{\big[H(\ww_t,\trainset) + \lambda_0 I_d\big]^{-1} \tilde{g}_t - \big[H(\ww_t,\trainset) + \frac{1}{n} H(w_t,z_{n+1}) + \lambda_0 I_d\big]^{-1} \tilde{g}_t} \nonumber\\
&\leq \frac{\norm{\tilde{g}_t}}{4n \lambda_0^2 + \lambda_0} . \nonumber
\]
This shows that by setting 
\[
\sigma_2 =  \frac{\sqrt{T}}{(4n \lambda_0^2 + \lambda_0)\sqrt{2 \rho \theta}}, \nonumber
\]
the mechanism in \cref{line:main-alg-direction-noise} of \cref{alg:main-opt} is $\frac{\theta \rho}{T}$-zCDP using \cref{lem:gauss-mech}.

In each step of the algorithm we have two privatizing step that satisfy $\frac{(1-\theta) \rho}{T}$ and $\frac{\theta \rho}{T}$. By the composition property of zCDP \citep[Lemma 2.3]{bun2016concentrated}, we conclude that $\ww_T$ satisfies $\rho$-zCDP.
\end{proof}
\subsection{Privacy Proof of \cref{alg:main-opt} for $\mathsf{clip}$ }
\label{subsec:appx-proof-clip}
\begin{theorem} \label{thm:sens-clip}
Assume in \cref{alg:main-opt}, we choose $\mathsf{clip}$ for the SOI modification. 
Then, for every training set $\trainset \in (\Reals^d \times \{-1,+1\})^n$, $\ww_0 \in \parspace$, $\lambda_0>0$, $T\in \Naturals$, $\rho \in \Reals_{+}$, and $\theta \in (0,1)$ such that $n> \frac{1}{4\lambda_0}$, by setting 
\[
\nonumber
\sigma_1 = \frac{\sqrt{T}}{n\sqrt{2 \rho (1-\theta)}}, \quad \sigma_2 = \frac{\sqrt{T}}{(4n \lambda_0^2 - \lambda_0)\sqrt{2 \rho \theta}},
\]
$\ww_T$ satisfies $\rho$-zCDP. 
\end{theorem}
\begin{proof}

Similar to the proof of \cref{thm:sens-add}, we use a two-stage approach. Since the logistic loss is a $1-$Lipschitz function, by setting 
$$
\sigma_1 = \frac{\sqrt{T}}{n\sqrt{2 \rho (1-\theta)}},
$$
the mechanism in \cref{line:main-alg-grad-noise} of \cref{alg:main-opt} satisfies $(1-\theta)\frac{\rho}{T}$-zCDP by \cref{lem:gauss-mech}. 

For the second step, following the same line as in the proof of \cref{thm:sens-add}, we need to upper bound 
\[
\nonumber
&\sup_{\trainset \in (\Reals^d \times \{-1,1\})^{n} }\sup_{ \substack{z_{n+1}=(x_{n+1},y_{n+1})\in \Reals^d \times \{-1,1\}:\\ \norm{x_{n+1}}\leq 1}} \\
&\norm{\big[\Psi_{\lambda_0}(H(w_t,\dataset),\mathsf{clip})\big]^{-1}  - \big[\Psi_{\lambda_0}(H(w_t,\dataset)+\frac{1}{n} H(\ww_t,z_{n+1}),\mathsf{clip})\big]^{-1} }. \nonumber 
\]
Let $A = \Psi_{\lambda_0}(H(w_t,\dataset),\mathsf{clip})$ and $B = \Psi_{\lambda_0}(H(w_t,\dataset)+\frac{1}{n} H(\ww_t;z_{n+1}),\mathsf{clip})$. We need a lemma for the next step of the proof.
\begin{lemma}\label{lem:inverse-continuous}
     Let $A,B \in \mathbb{R}^{d \times d}$ be positive definite matrices. If $\|A-B\| \cdot \|A^{-1}\| < 1$, then \[\nonumber\left\| A^{-1} - B^{-1} \right\| \le \frac{\|A-B\| \cdot \|A^{-1}\|^2}{1-\|A-B\| \cdot \|A^{-1}\|}.\]
\end{lemma}
\begin{proof}
    Let $B=A-C$. We have the identity $(A-C)^{-1} = A^{-1} \sum_{k=0}^\infty (CA^{-1})^k$, which holds as long as $\|CA^{-1}\|<1$ \citep[Thm.~4.8]{stewart1998matrix}. Thus \[\|A^{-1}-B^{-1}\| = \left\| A^{-1} \sum_{k=1}^\infty (CA^{-1})^k \right\| \le \|A^{-1}\| \sum_{k=1}^\infty \|CA^{-1}\|^k = \frac{\|A^{-1}\| \cdot \|CA^{-1}\|}{1-\|CA^{-1}\|}.\]
    Now $\|CA^{-1}\| \le \|C\| \cdot \|A^{-1}\| = \|A-B\| \cdot \|A^{-1}\|$, which gives the result.
\end{proof}
Using \cref{lem:inverse-continuous}, we can write
\[
\label{eq:sens-clip-inversemat}
\norm{A^{-1}  -B^{-1}}  \le \frac{\|A-B\| \cdot \|A^{-1}\|^2}{1-\|A-B\| \cdot \|A^{-1}\|},
\]
provided that $\|A-B\| \cdot \|A^{-1}\| < 1$. Note that  Frobenius norm of a matrix is not smaller than the operator norm, i.e.,  $\norm{A-B}\leq \norm{A-B}_F$. For the next step of the proof, we need a lemma.

\begin{lemma}\label{lem:proj-forb}
For every $\lambda_0 \ge 0$ and $A \in \mathbb{R}^{d \times d}$ with $A^\top=A$, we have
\[
\label{eq:proj-opt}
\Psi_{\lambda_0}(A,\mathsf{clip}) = \argmin_{\hat{A}\in \mathbb{R}^{d \times d} : \hat{A}^\top=\hat{A}, ~\forall x \in \mathbb{R}^d ~ x^\top \hat{A} x \ge \lambda_0 \|x\|_2^2}\|\hat{A}-A\|_F.
\]
Moreover, for every $\lambda_0 \ge 0$ and every PSD matrices $A \in \Reals^{d \times d}$ and $B \in \Reals^{d \times d}$, we have
\[
\nonumber
\norm{\Psi_{\lambda_0}(A,\mathsf{clip}) -\Psi_{\lambda_0}(B,\mathsf{clip})}_F \leq \norm{A-B}_F.
\]
\end{lemma}
\begin{proof}
Consider the eigenvalue decomposition of $A$ as $A=\sum_{i=1}^{d} \lambda_i u_i u_i^\top =U\Lambda U^\top$ where $\Lambda = \text{diag}(\lambda_1,\dots,\lambda_d)$ and $U \in \Reals^{d \times d}$ is matrix with $u_i$ on its $i$-th column.  We can represent every matrix in the feasible set of the optimization in \cref{eq:proj-opt} as $\hat{A} = \sum_{i=1}^{d} v_i v_i^\top \tilde{\lambda}_i=V\tilde{\Lambda} V^\top$ where $\{v_i\}_{i \in \range{d}}$  are a orthonormal basis for $\Reals^d$ and $\min_{i\in \range{d}}\tilde{\lambda}_i\geq \lambda_0$. By using simple facts about Frobenius norm and the eigenvalue decomposition, we can write
\begin{align}
    \norm{A-\hat{A}}_F^2 &= \norm{U \Lambda U^\top-\hat{A}}_F^2 \nonumber\\
&=\text{trace}\big((\Lambda-U^\top \hat{A} U)^2\big). \nonumber
\end{align}
We have $(\Lambda-U^\top \hat{A} U)^2 = \Lambda^2 - U^\top \hat{A} U\Lambda - \Lambda U^\top  \hat{A} U + U^\top \hat{A}^2U$. Thus,
\begin{align}
\norm{A-\hat{A}}_F^2 &= \text{trace}(\Lambda^2) -2\text{trace}(U^\top\hat{A}U\Lambda) + \text{trace}( U^\top \hat{A}^2U) \nonumber\\
                        &= \text{trace}(\Lambda^2)-2\text{trace}(U^\top\hat{A}U\Lambda) + \text{trace}(\hat{\Lambda}^2), \label{eq:dif-forb-step1}
\end{align}
where the last step follows from $\text{trace}( U^\top \hat{A}^2U) = \text{trace}(\hat{A}^2)=\text{trace}(V\tilde{\Lambda}^2 V^\top)=\text{trace}(\tilde{\Lambda}^2) $.

As  the trace operator is invariant under the cyclic permutation, we have $\text{trace}(U^\top\hat{A}U\Lambda) = \text{trace}(U\Lambda U^\top\hat{A})=\text{trace}(A\hat{A})$. Then, we invoke Von Neumann's trace inequality \citep{mirsky1975trace} which states that
\[
\label{eq:eigenvec-ineq}
\text{trace}(A\hat{A}) \leq \sum_{i=1}^{d}\lambda_i \hat{\lambda}_i,
\]
where the equality holds if $A$ and $\hat{A}$ share the same eigenvectors. Therefore, by \cref{eq:dif-forb-step1} and \cref{eq:eigenvec-ineq}, we have
\begin{align*}
\norm{A-\hat{A}}_F^2 &\geq \text{trace}(\Lambda^2)-2\sum_{i=1}^{n}\lambda_i \hat{\lambda}_i + \text{trace}(\hat{\Lambda}^2)\\
&= \sum_{i=1}^{d} (\lambda_i - \hat{\lambda}_i)^2.
\end{align*}
It is straightforward to see that 
\[
\label{eq:eigenval-ineq}
\sum_{i=1}^{d} (\lambda_i - \hat{\lambda}_i)^2\geq \sum_{i=1}^{d} (\lambda_i - \max\{\lambda_0,\lambda_i\})^2.
\]
Thus, we obtain that for every $\hat{A}$ in the feasible set of \cref{eq:proj-opt} the following holds
\[
\nonumber
\norm{A-\hat{A}}_F^2 \geq \sum_{i=1}^{d} (\lambda_i - \max\{\lambda_0,\lambda_i\})^2.
\]
For deriving this lower bound we used two inequalities in \cref{eq:eigenvec-ineq} and \cref{eq:eigenval-ineq}. The equality condition for \cref{eq:eigenvec-ineq} is that $A$ and $\hat{A}$ share the same eigenvectors, and, for \cref{eq:eigenval-ineq}, the equality condition is $\hat{\lambda}_i= \max\{\lambda_i, \lambda_0\}$ for every $i \in \range{d}$. Therefore, we conclude that $\Psi_{\lambda_0}(A,\mathsf{clip})$ is a minimizer of \cref{eq:proj-opt}.

For the second part, notice that the feasible set in the optimization problem \cref{eq:proj-opt} is a convex and closed subset of $\Reals^{d \times d}$. Also, Frobenius norm is a metric induced by an inner product over the vector space of the real symmetric matrices. Therefore, we conclude that for every $\lambda_0>0$, $\Psi_{\lambda_0}(\cdot,\mathsf{clip})$ is a projection onto a convex and closed set, and the second claim follows. 
\end{proof}

In \cref{lem:proj-forb}, we show that $\Psi_{\lambda_0}(\cdot, \mathsf{clip})$ is a Frobenius-norm projection onto a convex and closed set. Therefore, by the contraction property of the projection, we have
\[
\nonumber
\norm{A-B}_F &= \norm{\Psi_{\lambda_0}(H(w_t,\dataset)+\frac{1}{n} H(\ww_t,z_{n+1}),\mathsf{clip}) -  \Psi_{\lambda_0}(H(w_t,\dataset),\mathsf{clip})}_F\\
&\leq \frac{1}{n} \norm{ H(\ww_t,z_{n+1})}_F. \nonumber
\]
Since $H(\ww_t,z_{n+1})$ is a rank-1 matrix, we have $ \frac{1}{n} \norm{ H(\ww_t,z_{n+1})}_F \leq \frac{1}{4n}$  (see \cref{rem:rank-bound}.).

For every $a>0$, $h:\Reals \to \Reals, h(x)=\frac{x}{1-ax}$ is increasing for $x<\frac{1}{a}$. Therefore, from \cref{eq:sens-clip-inversemat} 
\[
\nonumber
\norm{A^{-1}  -B^{-1}}  &\leq  \|A^{-1}\|^2 \frac{\|A-B\| }{1-\|A-B\| \cdot \|A^{-1}\|} \\
&\leq \frac{\norm{A}^{-2}}{4n-\norm{A}^{-1}}. \nonumber
\]
Consider $h:\Reals \to \Reals, h(x)=\frac{x^2}{4n-x}$. This function  is increasing in the interval $0 \leq x<4n$. Using this observation, $\norm{A}^{-1}\leq \frac{1}{\lambda_0}$, and $4n \lambda_0 \geq 1$ we obtain
\[
\nonumber
\norm{A^{-1}  -B^{-1}} &\leq \frac{\norm{A}^{-2}}{4n-\norm{A}^{-1}}\\
                        &\leq \frac{1}{4n \lambda_0^2-\lambda_0}. \nonumber
\]
Therefore, we have shown that 
\[
&\sup_{\trainset }\sup_{ \substack{z_{n+1}=(x_{n+1},y_{n+1}):\\ \norm{x_{n+1}}\leq 1}} \nonumber\\ &\norm{\big[\Psi_{\lambda_0}(H(w_t,\dataset),\mathsf{clip})\big]^{-1}  \tilde{g}_t  - \big[\Psi_{\lambda_0}(H(w_t,\dataset)+\frac{1}{n} H(\ww_t,z_{n+1}),\mathsf{clip})\big]^{-1} \tilde{g}_t} \nonumber\\
&\leq  \frac{\norm{\tilde{g}_t}}{4n \lambda_0^2-\lambda_0}  . \nonumber
\]
This shows that by setting 
\[
\sigma_2 =  \frac{\sqrt{T}}{(4n \lambda_0^2 - \lambda_0)\sqrt{2 \rho \theta}}, \nonumber
\]
the mechanism in \cref{line:main-alg-direction-noise} of \cref{alg:main-opt} is $\frac{\theta \rho}{T}$-zCDP.

In each step of the algorithm we have two privatizing step that satisfy $\frac{(1-\theta) \rho}{T}$ and $\frac{\theta \rho}{T}$. By the composition property of zCDP \citep[Lemma 2.3]{bun2016concentrated}, we conclude that $\ww_T$ satisfies $\rho$-zCDP.
\end{proof}

\subsection{Description of Double noise with adaptive minimum eigenvalue selection}
\label{subsec:Description-alg-adaptive}
In \cref{alg:opt-adaptive}, we provide the detailed algorithmic description of the variant of our algorithm with an adaptive minimum eigenvalue selection. We use this variant in our numerical results.
\begin{algorithm}
\caption{Newton Method with Double noise and adaptive min.~eigenvalue}\label{alg:opt-adaptive}
\begin{algorithmic}[1]
\State Input: training set $\trainset \in \mathcal{Z}^n$, $\theta \in (0,1)$ for dividing privacy budget for gradient vs SOI, $\gamma \in (0,1)$ for dividing privacy budget for trace estimation, $\beta >0$ as the coefficient for min. eig. value ,privacy budget $\rho$-zCDP, initialization $\ww_0$, number of iterations $T$, Hessian modification $\in \{\mathsf{clip},\mathsf{add}\}$.
\State Set $\sigma_1 =\frac{\sqrt{T}}{n\sqrt{2 \rho (1-\theta)}}$
\State Set $\sigma_{\text{tr}} = \frac{\sqrt{T} }{4n \sqrt{2\theta \rho \gamma} }$ 
\For{$t=0,\dots,T-1$}
    \State Query $\gradl(\ww_t,\trainset)$ and $H(\ww_t,\trainset)$
    \State $\widetilde{\text{trace}}_t= \max\{ \text{trace}(H(\ww_t)) + \Normal(0,\sigma_{\text{tr}}^2I_d),0\}$  \label{line:private}
    \State 
    \State $\lambda_{0,t}= \max \Big \lbrace \beta \cdot  (\widetilde{\text{trace}}_t)^{\nicefrac{1}{3}}\bigg( \frac{T}{n^2 (1-\gamma)\rho \theta} \bigg)^{\nicefrac{1}{3}}, \frac{1}{n}\Big\rbrace$ \Comment{To prevent $\lambda_{0,t}$ makes $\sigma_2$ negative.}
     
     \If {Hessian modification = $\mathsf{Add}$}
    \State $ \sigma_2 = \frac{\sqrt{T}}{(4n \lambda_{0,t}^2 + \lambda_{0,t})\sqrt{2 (1-\gamma)\rho \theta}}$
    \ElsIf{Hessian modification = $\mathsf{Clip}$}
    \State $ \sigma_2 =  \frac{\sqrt{T}}{(4n \lambda_{0,t}^2 -\lambda_{0,t})\sqrt{2 (1-\gamma)\rho \theta}}$
    \EndIf
    \State $\tilde{H}_t = \Psi_{\lambda_0}(H(\ww_t,\trainset),\text{Hessian modification})$
    \State $\tilde{g}_t = H(\ww_t) + \Normal(0,\sigma^2_1 I_d)$
    \State $\ww_{t+1} = \ww_{t} - \tilde{H}_t^{-1}\tilde{g}_t + \Normal(0,\norm{\tilde{g}_t}^2\sigma_2^2 I_d) $
\EndFor
\State Output $\ww_{T}$.
\end{algorithmic}
\end{algorithm}

\subsubsection{Derivation}
\label{subec:appx-eigen-selection}
Let $\phi: \Reals^d \to \Reals$ be the second-order approximation of $\loss(v,\trainset)$ at $\ww_t$ given by
\[
\phi(v) = \loss(\ww_t,\trainset) + \inner{\gradl(\ww_t,\trainset)}{v-\ww_t} + \frac{1}{2} \inner{H(\ww_t,\trainset)(v-\ww_t)}{(v-\ww_t)},
\]
where $H(\ww_t,\trainset)$ can be either $H_\text{qu}(\ww_t,\trainset)$ from \cref{lem:quad-ub} or $\nabla^2\loss(\ww_t,\trainset)$. 

Let $\lambda > 0$, $\tilde{H}_t = \Psi_{\lambda}(H(\ww_t,\trainset),\text{hessian modification})$, and $v_{\lambda}=\ww_t - \tilde{H}_t^{-1}\tilde{g}_t + \sigma_2 \norm{\tilde{g}_t} \xi_t$ where $\sigma_2 >0$ is a constant and $\xi_t \sim \Normal(0,I_d)$. Our goal here is to find $\lambda >0$ as an approximate minimizer of $ \EE_{\xi_t \sim \Normal(0,I_d)}[\phi(v_{\lambda})]$. Note that we condition on the random variables $\ww_t$ and $\tilde{g}_t$. 

We begin by expanding $\phi(v_{\lambda})$ as follows:
\begin{align}
    \EE_{\xi_t \sim \Normal(0,I_d)}[\phi(v_{\lambda})] &= \EE_{\xi_t \sim \Normal(0,I_d)}\Big[\loss(\ww_t,\trainset) + \inner{\gradl(\ww_t,\trainset)}{- \tilde{H}_t^{-1}\tilde{g}_t + \sigma_2 \norm{\tilde{g}_t} \xi_t}  \nonumber\\
    &+\frac{1}{2} (- \tilde{H}_t^{-1}\tilde{g}_t + \sigma_2 \norm{\tilde{g}_t} \xi_t)^\top H(\ww_t,\trainset) (- \tilde{H}_t^{-1}\tilde{g}_t + \sigma_2 \norm{\tilde{g}_t} \xi_t)\Big] \nonumber\\
    & = \loss(\ww_t,\trainset) +  \inner{\gradl(\ww_t,\trainset)}{- \tilde{H}_t^{-1}\tilde{g}_t} + \frac{1}{2}\tilde{g}_t^\top  (\tilde{H}_t^{-1})^\top H(\ww_t,\trainset) \tilde{H}_t^{-1}\tilde{g}_t\nonumber\\
    & + \EE_{\xi_t \sim \Normal(0,I_d)}\left[ \frac{1}{2}\sigma_2^2 \norm{\tilde{g}_t}^2 \xi_t^\top H(\ww_t,\trainset)\xi_t \right] \nonumber\\
    & = \loss(\ww_t,\trainset) + \inner{\gradl(\ww_t,\trainset)}{- \tilde{H}_t^{-1}\tilde{g}_t} + \frac{1}{2}\tilde{g}_t^\top (\tilde{H}_t^{-1})^\top H(\ww_t,\trainset) \tilde{H}_t^{-1}\tilde{g}_t\nonumber\\
    & + \frac{\sigma_2^2 \norm{\tilde{g}_t}^2}{2} \text{trace}(H(\ww_t,\trainset)), \label{eq:phi-exact}
\end{align}
where we have used $\EE[\xi_t]=0$ and $\EE[\xi_t^\top H(\ww_t,\trainset) \xi_t]= \EE[\text{trace}(\xi_t \xi_t^\top H(\ww_t,\trainset) )]=\text{trace}(H(\ww_t,\trainset))$. 

Consider the eigenvalue decomposition of $\tilde{H}_t \triangleq U\tilde{\Lambda} U^\top$ and $H(\ww_t,\trainset) \triangleq U \Lambda U^\top$. Notice that by the definition of the adding and clipping operators in \cref{def:modif}, $H(\ww_t,\trainset)$ and $\tilde{H}_t$ share the same eigenvectors. 

To approximate \cref{eq:phi-exact}, we assume $\tilde{g}_t\approx \gradl(\ww_t,\trainset)$. Then, by the change of variable $b=U^\top \tilde{g}$, we can rephrase \cref{eq:phi-exact} as follows
\[
\label{eq:phi-approx}
\argmin_{\lambda>0}\EE_{\xi_t \sim \Normal(0,I_d)}[\phi(v_{\lambda})] \approx \argmin_{\lambda>0} \big\{-b^\top \tilde{\Lambda}^{-1} b + \frac{1}{2}b^\top \tilde{\Lambda}^{-1} \Lambda  \tilde{\Lambda}^{-1} b + \frac{\sigma_2^2 \norm{b}^2}{2}\text{trace}(H(\ww_t,\trainset))\big\}.
\]
Consider the eigenvalue modification using $\mathsf{add}$ operator. In this case $\tilde{H}_t = H(\ww_t,\trainset) + \lambda I_d$. Let $\Lambda = \text{diag}(\lambda_1,\dots,\lambda_d)$ and $\tilde{\Lambda}=\Lambda + \lambda I_d$. Also from \cref{thm:sens-add},  $\sigma_2 = \frac{1}{(4n \lambda^2 + \lambda)\sqrt{2 \rho_2 }}$ where $\rho_2 > 0$ is the privacy budget. Setting these
    parameters in \cref{eq:phi-approx}, we get
    \[
    \nonumber
    h(\lambda) = \sum_{i=1}^{d} b_i^2 \big(\frac{-1}{\lambda_i + \lambda} + \frac{0.5\lambda_i}{(\lambda_i + \lambda)^2}\big) +  \frac{\norm{b}^2}{2} \big(\frac{1}{(4n \lambda^2 + \lambda)\sqrt{2 \rho_2}}\big)^2 \text{trace}(H(\ww_t,\trainset)).
    \]
    By taking the derivative of $h(\lambda)$ and setting it to zero, we obtain
    \[
    \label{eq:add-appx-optimal}
    \frac{\text{d}h(\lambda^\star)}{\text{d}\lambda^\star}=0 \imp \sum_{i=1}^{d} b_i^2 \frac{\lambda^\star}{(\lambda_i + \lambda^\star)^3} = \frac{\norm{\tilde{g}_t}^2 \text{trace}(H(\ww_t,\trainset))}{2\rho_2 } \frac{1+8n\lambda^\star}{(4n(\lambda^\star)^2 + \lambda^\star)^3}.
    \]
    In many practical scenarios, the SOI matrix has zero eigenvalues. This observation motivates us to use the approximation $\frac{\lambda^\star}{(\lambda_i + \lambda^\star)^3} \approx \frac{1}{(\lambda^\star)^2}$ for all $i \in \range{d}$. Let $\beta \in (0,1)$ such that 
    \[
    \nonumber
     \sum_{i=1}^{d} b_i^2 \frac{\lambda^\star}{(\lambda_i + \lambda^\star)^3} &=  \frac{\beta}{(\lambda^\star)^2} \sum_{i=1}^{d} b_i^2\\
     \nonumber
     &= \frac{ \beta}{(\lambda^\star)^2} \norm{\tilde{g}_t}^2.
    \]
    where we have used $\norm{b}=\norm{U^\top \tilde{g}_t}=\norm{\tilde{g}_t}$.
    We can approximate \cref{eq:add-appx-optimal} by
    \[
    \nonumber
         \beta \lambda^\star = \frac{\text{trace}(H(\ww_t,\trainset))}{2\rho_2 } \frac{1+8n\lambda^\star}{(4n\lambda^\star + 1)^3}.
    \]
    Assume $n\lambda_0^\star \gg 1$, we obtain 
    \[
    \label{eq:optapprox-mineig-add}
    \lambda^\star \approx \bigg( \frac{\text{trace}(H(\ww_t,\trainset))}{n^2 \rho_2} \bigg)^{\frac{1}{3}}.
    \]

The derivation for $\mathsf{clip}$ follows similarly using the smooth approximation of the max function: Llet $m> 0$ be a constant. For all $i \in \range{d}$, we approximate $\max\{\lambda_i,\lambda\}\approx m^{-1}\log(\exp(m \lambda_i)+\exp(m \lambda))$.

\subsection{Generalization of \cref{alg:main-opt} for convex, Lipschitz, and smooth loss functions}\label{subsec:generalization-alg}
\begin{algorithm}[h!]
\caption{Generalization of \cref{alg:main-opt} for  convex, $\lipf$-Lipschitz, and $\lipg$-smooth losses }
\label{alg:general-version}
\begin{algorithmic}[1]
\State Inputs: training set $\trainset \in \mathcal{Z}^n$, $\lambda_0 >0$, $\theta \in (0,1)$,  privacy budget $\rho$-zCDP, initialization $\ww_0$, number of iterations $T$, hessian modification $\in \{\mathsf{clip},\mathsf{add}\}$.
\State Set  $\sigma_1 =\frac{\lipf\sqrt{T }}{n\sqrt{2 \rho (1-\theta)}}$
\If {hessian modification = $\mathsf{Add}$}
 \State Condition: $n  \lambda_0 > \lipg $
 \State $ \displaystyle  \sigma_2 =  \frac{\lipg}{n  \lambda_0^2 - \lambda_0 \lipg}\cdot \frac{\sqrt{T}}{\sqrt{2 \rho \theta}}$
 \ElsIf{hessian modification = $\mathsf{Clip}$}
 \State Condition: $n\lambda_0 >  \lipg \big(\frac{2}{\pi} + \frac{1}{2} +  \frac{1}{\pi} \log\big(\frac{n(\lipg - \lambda_0) + \lipg }{ \lipg }\big) \big)$ and $2\lambda_0 \leq \lipg$
 \State $ \displaystyle \sigma_2 = \frac{\lipg \big( \frac{2}{\pi} + \frac{1}{2} +  \frac{1}{\pi} \log\big(\frac{n(\lipg - \lambda_0) + \lipg }{ \lipg }\big) \big)}{n\lambda_0^2-   \lambda_0 \lipg \big(\frac{2}{\pi} + \frac{1}{2} +  \frac{1}{\pi} \log\big(\frac{n(\lipg - \lambda_0) + \lipg }{ \lipg }\big) \big) }\cdot \frac{\sqrt{T}}{\sqrt{2\rho \theta}}$.
 \EndIf
\For{$t=0,\dots,T-1$}
    \State Query $\nabla \loss(\ww_t,\trainset)$ and $\nabla^2 \loss(\ww_t,\trainset)$
    \State $\tilde{H}_t = \Psi_{\lambda_0}(\nabla^2 \loss(\ww_t,\trainset);\text{hessian modification})$
    \State $\tilde{g}_t = \nabla \loss(\ww_t,\trainset) + \Normal(0,\sigma^2_1 I_d)$. \label{line:gen-alg-grad-noise}
    \State $\ww_{t+1} = \ww_{t} - \tilde{H}_t^{-1}\tilde{g}_t + \Normal(0,\norm{\tilde{g}_t}^2 \sigma_2^2  I_d) $ \label{line:gen-alg-direction-noise} .
\EndFor
\State Output $\ww_{T}$.
\end{algorithmic}
\end{algorithm}

\begin{theorem}
For every convex, $\lipf$-Lipschitz, $\lipg$-smooth loss function $f(\cdot,\cdot)$, training set $\trainset \in \dataspace^n $, initialization $\ww_0 \in \parspace$, $T\in \Naturals$, $\rho \in \Reals_{+}$, and $\theta \in (0,1)$,
$\ww_T$ in \cref{alg:general-version} satisfies $\rho$-zCDP. 
\end{theorem}
\begin{proof}
    We follow the two-stage procedure of \cref{thm:sens-add} and \cref{thm:sens-clip}. Since the loss function is $\lipf$-Lipschitz, by setting,
    \[
    \nonumber
        \sigma_1 = \frac{\lipf\sqrt{T }}{n\sqrt{2 \rho (1-\theta)}},
    \]
    the mechanism in \cref{line:gen-alg-grad-noise} of \cref{alg:general-version} satisfies $\frac{\rho (1-\theta)}{T}$-zCDP.

First consider the case with using $\mathsf{add}$ for the SOI modification. For the second step, following the same line as in the proof of \cref{thm:sens-add}, we need to upper bound 
\[
\nonumber
\sup_{\trainset \in \dataspace^{n} }\sup_{ z_{n+1}\in \dataspace} \norm{\big[\Psi_{\lambda_0}(\nabla^2 \loss(\ww_t,\trainset)+\frac{1}{n} \nabla^2 f(\ww_t,z_{n+1}),\mathsf{add})\big]^{-1}  - \big[\Psi_{\lambda_0}(\nabla^2 \loss (\ww_t,\dataset),\mathsf{add})\big]^{-1} } .
\]
Let $A = \Psi_{\lambda_0}(\nabla^2 \loss (w_t,\dataset),\mathsf{add}) = \nabla^2 \loss (w_t,\dataset) + \lambda_0 I_d$ and $B = \frac{1}{n} \nabla^2 f(\ww_t,z_{n+1})$. We need a lemma for the next step of the proof.

\begin{lemma}\label{lem:lip-operator-norm}
For every PSD matrix $A \in \Reals^{d\times d}$, $B \in \Reals^{d\times d}$, and $\lambda_0  \geq 0$ such that $\norm{A},\norm{B} < \infty$ we have
\[
\nonumber
&\norm{\Psi_{\lambda_0}(A+B,\mathsf{clip})-\Psi_{\lambda_0}(A,\mathsf{clip})}\leq \norm{B} \bigg(\frac{2}{\pi} + \frac{1}{2} + \frac{1}{\pi} \log\Big(\frac{\norm{A -\lambda_0 I_d} + \norm{B}}{ \norm{B} }\Big) \bigg),\\
&\norm{\Psi_{\lambda_0}(A+B,\mathsf{add})-\Psi_{\lambda_0}(A,\mathsf{add})}\leq \norm{B} . \nonumber
\]
\end{lemma}
\begin{proof}
The Lipschitzness of $\Psi_{\lambda_0}(\cdot,\mathsf{add})$ is obvious from  \cref{def:modif}. We prove the result for $\Psi_{\lambda_0}(\cdot,\mathsf{clip})$. 

For a symmetric matrix $A \in \Reals^{d\times d}$, let $A= U \Lambda U^\top$ be the eigenvalue decomposition of $A$ where $\Lambda = \text{diag}(\lambda_1,\dots,\lambda_d)$. Then,   define \emph{the absolute value of $A$} as $|A| \triangleq U |\Lambda| U^\top \in \Reals^{d \times d}$ where $|\Lambda| = \text{diag}(|\lambda_1|,\dots,|\lambda_d|)$.

It is straightforward to see that 
\[
\nonumber
\Psi_{\lambda_0}(A,\mathsf{clip}) = \frac{1}{2} \big( |A - \lambda_0 I_d| + A + \lambda_0 I_d \big).
\]

Therefore, 
\[
&\norm{\Psi_{\lambda_0}(A+B,\mathsf{clip})-\Psi_{\lambda_0}(A,\mathsf{clip})} \nonumber\\
&= \frac{1}{2}\norm{\big( |A + B -\lambda_0 I_d| + (A+B) + \lambda_0 I_d \big) - \big( |A - \lambda_0 I_d| + A + \lambda_0 I_d \big)} \nonumber\\
&=\frac{1}{2}\norm{ |A + B -\lambda_0 I_d| -  |A - \lambda_0 I_d| +  B} \nonumber\\
&\leq \frac{1}{2} \norm{|A + B -\lambda_0 I_d| -  |A - \lambda_0 I_d|} + \frac{1}{2}\norm{B}. \label{eq:clip-absval}
\]
Then, we invoke the result of \citep{kato1973continuity} which states that
\[
\frac{1}{2} \norm{|A + B -\lambda_0 I_d| -  |A - \lambda_0 I_d|} \leq \frac{\norm{B}}{\pi}  \Big(2 + \log\Big(\frac{\norm{A -\lambda_0 I_d} + \norm{B}}{ \norm{B} }\Big) \Big) \label{eq:cont-abs-result}
\]
Combining \cref{eq:clip-absval,eq:cont-abs-result} concludes the proof.
\end{proof}
Using \cref{lem:lip-operator-norm} and \cref{lem:inverse-continuous} we can write
\begin{align}
&\norm{\big[\Psi_{\lambda_0}(\nabla^2 \loss(\ww_t,\trainset)+\frac{1}{n} \nabla^2 f(\ww_t,z_{n+1}),\mathsf{add})\big]^{-1}  - \big[\Psi_{\lambda_0}(\nabla^2 \loss (w_t,\dataset),\mathsf{add})\big]^{-1} } \nonumber \\
&= \norm{(A+B)^{-1}-A^{-1}} \nonumber\\
&\leq \frac{\norm{A^{-1}}^2 \norm{B}}{1- \norm{A^{-1}}\norm{B}}.\label{eq:gen-version-sens}
\end{align}
 Let $n \geq \lipg \lambda_0^{-1}$. 
Notice that $\norm{A^{-1}} \leq \lambda_0^{-1}$ and $\norm{B}\leq \lipg n^{-1}$ because of the modification operator and the smoothness of the loss function. Using this observation, we can write
 \[
 \nonumber
 \frac{\norm{A^{-1}}^2 \norm{B}}{1- \norm{A^{-1}}\norm{B}} &\leq \sup_{0\le x \le \lipg n^{-1}}  \frac{\norm{A^{-1}}^2 x}{1- \norm{A^{-1}}x} \\
 & = \norm{A^{-1}}^2 \frac{\lipg}{n - \norm{A^{-1}}\lipg} \nonumber.
 \]
 Here the last step follows from the following fact: For every $a>0$,  $h:\Reals \to \Reals, h(x)=\frac{x}{1-ax}$ is increasing for $x<\frac{1}{a}$. Then,
 \[
\norm{A^{-1}}^2 \frac{\lipg}{n - \norm{A^{-1}}\lipg} &\leq \sup_{0\le x \le \lambda_0^{-1}}   \frac{x^2\lipg}{n - x\lipg} \nonumber\\
                                                        &= \frac{\lipg }{n \lambda_0^2 - \lambda_0 \lipg} \nonumber,
 \]
where the last step follows from the following fact: $h:\Reals \to \Reals, h(x)=\frac{x^2}{n-x\lipg}$ is increasing in the interval $ 0 \leq x<n \lipg^{-1}$. 
 Therefore, we conclude that 
 \[
 &\sup_{\trainset \in \dataspace^{n} }\sup_{ z_{n+1}\in \dataspace} \norm{\big[\Psi_{\lambda_0}(\nabla^2 \loss(\ww_t,\trainset)+\frac{1}{n} \nabla^2 f(\ww_t,z_{n+1}),\mathsf{add})\big]^{-1}  - \big[\Psi_{\lambda_0}(\nabla^2 \loss (w_t,\dataset),\mathsf{add})\big]^{-1} } \nonumber \\
 &\leq  \frac{\lipg }{n \lambda_0^2 - \lambda_0 \lipg}.
 \]
 This shows that by setting 
\[
\sigma_2 =  \frac{\lipg\sqrt{T}}{(n  \lambda_0^2 - \lambda_0 \lipg)\sqrt{2 \rho \theta}}, \nonumber
\]
the mechanism in \cref{line:gen-alg-direction-noise} of \cref{alg:general-version} is $\frac{\theta \rho}{T}$-zCDP.

In each step of the algorithm we have two privitization step that satisfy $\frac{(1-\theta) \rho}{T}$ and $\frac{\theta \rho}{T}$. By the composition property of zCDP \citep[Lemma 2.3]{bun2016concentrated}, we conclude that $\ww_T$ satisfies $\rho$-zCDP.

Next, we provide a privacy analysis for the clipping operator. We are interested in upper-bounding the following term
\[
\nonumber
\sup_{\trainset \in \dataspace^{n} }\sup_{ z_{n+1}\in \dataspace} \norm{\big[\Psi_{\lambda_0}(\nabla^2 \loss(\ww_t,\trainset)+\frac{1}{n} \nabla^2 f(\ww_t,z_{n+1}),\mathsf{clip})\big]^{-1}  - \big[\Psi_{\lambda_0}(\nabla^2 \loss (\ww_t,\dataset),\mathsf{clip})\big]^{-1} } .
\]

Let 
\[
A = \Psi_{\lambda_0}(\nabla^2 \loss (w_t,\dataset),\mathsf{clip}), \quad B = \Psi_{\lambda_0}(\nabla^2 \loss(\ww_t,\trainset)+\frac{1}{n} \nabla^2 f(\ww_t,z_{n+1}),\mathsf{clip}). \nonumber 
\]
Then, using \cref{lem:inverse-continuous} we can write
\begin{align}
&\norm{\big[\Psi_{\lambda_0}(\nabla^2 \loss(\ww_t,\trainset)+\frac{1}{n} \nabla^2 f(\ww_t,z_{n+1}),\mathsf{clip})\big]^{-1}  - \big[\Psi_{\lambda_0}(\nabla^2 \loss (w_t,\dataset),\mathsf{clip})\big]^{-1} } \nonumber \\
&\leq \frac{\norm{B-A}\norm{A^{-1}}^2}{1-\norm{B-A} \norm{A^{-1}}}.\label{eq:gen-version-sens-clip}
\end{align}

Then, we invoke \cref{lem:lip-operator-norm} to write
\*[
&\norm{B-A} \\
&\leq  \frac{1}{n} \norm{ \nabla^2 f(\ww_t,z_{n+1})} \bigg(\frac{2}{\pi} + \frac{1}{2} +  \frac{1}{\pi} \log\Big(\frac{n\norm{\nabla^2 \loss(\ww_t,\trainset) -\lambda_0 I_d} + \norm{\nabla^2 f(\ww_t,z_{n+1})} }{ \norm{ \nabla^2 f(\ww_t,z_{n+1})}  }\Big) \bigg)\\
            & \leq  \frac{1}{n} \norm{ \nabla^2 f(\ww_t,z_{n+1})} \bigg(\frac{2}{\pi} + \frac{1}{2} +  \frac{1}{\pi} \log\Big(\frac{n(\lipg - \lambda_0) + \norm{ f(\ww_t,z_{n+1})} }{ \norm{ \nabla^2 f(\ww_t,z_{n+1})}  }\Big) \bigg),
\]
where the last step follows from the smoothness of $f$ and the assumption that $2\lambda_0 \leq \lipg$. 
By the smoothness we have $\norm{ \nabla^2 f(\ww_t,z_{n+1})} \leq \lipg$, therefore, to upper bound $\norm{B-A}$ we can write
\*[
\norm{B-A} &\leq \frac{1}{n} \norm{ \nabla^2  f(\ww_t,z_{n+1})} \bigg(\frac{2}{\pi} + \frac{1}{2} +  \frac{1}{\pi} \log\Big(\frac{n(\lipg - \lambda_0) + \norm{ \nabla^2  f(\ww_t,z_{n+1})} }{ \norm{ \nabla^2  f(\ww_t,z_{n+1})}  }\Big) \bigg)\\
&\leq \sup_{0\leq y\leq \lipg} \frac{y}{n}  \bigg(\frac{2}{\pi} + \frac{1}{2} +  \frac{1}{\pi} \log\Big(\frac{n(\lipg - \lambda_0) + y }{ y }\Big) \bigg)\\
&= \frac{\lipg}{n}  \bigg(\frac{2}{\pi} + \frac{1}{2} +  \frac{1}{\pi} \log\Big(\frac{n(\lipg - \lambda_0) + \lipg }{ \lipg }\Big) \bigg) \triangleq \Delta.
\]
where the last step follows from the following technical lemma.

\begin{lemma}
\label{lem:aux-increasing}
For every $a>0$, function $ \displaystyle f:\Reals \to \Reals,~f(x)=x \left( \log\frac{x+a}{x} \right)$ is increasing for $x>0$.
\end{lemma}
\begin{proof}
The derivative of $f$ is given by
$ 
\frac{\text{d}f(x)}{\text{d}x} = \log(1+\frac{a}{x}) - \frac{a}{x+a}.$ By using the inequality $\log(1+y) \geq \frac{y}{1+y} $ for $y>-1$, we can show that $\frac{\text{d}f(x)}{\text{d}x} \geq 0$, as was to be shown. 
\end{proof}

Then, we can further upper bound \cref{eq:gen-version-sens-clip} as follows:
\*[
\frac{\norm{B-A}\norm{A^{-1}}^2}{1-\norm{B-A} \norm{A^{-1}}} &\leq \frac{\Delta \norm{A^{-1}}^2}{1- \Delta \norm{A^{-1}}}\\
                                                            & \leq \frac{\Delta}{\lambda_0^2- \Delta \lambda_0},
\]
where the last step follows from $\norm{A^{-1}}\leq \lambda_0^{-1}$.

Therefore, we conclude that 
\*[
&\sup_{\trainset \in \dataspace^{n} }\sup_{ z_{n+1}\in \dataspace} \norm{\big[\Psi_{\lambda_0}(\nabla^2 \loss(\ww_t,\trainset)+\frac{1}{n} \nabla^2 f(\ww_t,z_{n+1}),\mathsf{clip})\big]^{-1}  - \big[\Psi_{\lambda_0}(\nabla^2 \loss (\ww_t,\dataset),\mathsf{clip})\big]^{-1} } \\
&\leq  \frac{\lipg \bigg( \frac{2}{\pi} + \frac{1}{2} +  \frac{1}{\pi} \log\Big(\frac{n(\lipg - \lambda_0) + \lipg }{ \lipg }\Big) \bigg)}{n\lambda_0^2-  \lipg \lambda_0 \bigg(\frac{2}{\pi} + \frac{1}{2} +  \frac{1}{\pi} \log\Big(\frac{n(\lipg - \lambda_0) + \lipg }{ \lipg }\Big) \bigg) }.
\]

The rest of the proof is similar to the proof of the Hessian modification using the adding operator.
\end{proof} 

\subsection{Suboptimality Gap for Logistic Loss and $\norm{\cdot}_V$}
\label{appx:suboptgap}
From \cref{lem:quad-ub}, since $\nabla \losslog(\ww^\star,\trainset)=0$ we have
\*[
\losslog(\ww,\trainset) \leq \losslog(\ww^\star,\trainset) +  \left(\ww - \ww^\star \right)^\top \left(\frac{1}{n} \sum_{i=1}^{n} x_i x_i^\top  \frac{\tanh(\nicefrac{\inner{x_i}{\ww^\star}}{2})}{4\inner{x_i}{\ww^\star}}\right)    \left(\ww - \ww^\star \right).
\]
By definition of $V$, $V x_i = x_i$. Therefore, 
\*[
\losslog(\ww,\trainset) \leq \losslog(\ww^\star,\trainset) +  \left(\ww - \ww^\star \right)^\top V \left(\frac{1}{n} \sum_{i=1}^{n} x_i x_i^\top  \frac{\tanh(\nicefrac{\inner{x_i}{\ww^\star}}{2})}{4\inner{x_i}{\ww^\star}}\right)  V  \left(\ww - \ww^\star \right).
\]
Since $\left(\frac{1}{n} \sum_{i=1}^{n} x_i x_i^\top  \frac{\tanh(\nicefrac{\inner{x_i}{\ww^\star}}{2})}{2\inner{x_i}{\ww^\star}}\right)  \preccurlyeq \frac{1}{4} I_d $ and $V^2 = V$, we have
\*[
\losslog(\ww,\trainset) &\leq \losslog(\ww^\star,\trainset) + \frac{1}{8} \left(\ww - \ww^\star\right)^\top V \left(\ww - \ww^\star\right)\\
&= \frac{1}{8}\norm{\ww - \ww^\star}_V^2,
\]
which was to be shown.
\subsection{Proof of \cref{thm:practical-local-convergence}} \label{sec:proof-hess-convergence}
We start this section by recalling some of the well-known properties of \emph{Mahalanobis semi-norm}. 
\begin{lemma}
\label{lem:seminorm1}
Let $A \in \Reals^{d\times d}$ be a positive semi-definite matrix. For every $x, y\in \Reals^d$ define $ \inner{x}{y}_A  \triangleq x^\top A y$ and $\norm{x}_A \triangleq \sqrt{\inner{x}{x}_A}$. Then, the following holds:
\begin{itemize}
    \item For every $x\in \Reals^d$, we have $\norm{x}_A \geq 0$. 
    \item for every $\alpha \in \Reals$, we have $\norm{\alpha x}_A = |\alpha| \norm{x}_A$.
    \item For every $x,y\in \Reals^d$, $\norm{x+y}_A \leq \norm{x}_A + \norm{y}_A $.
    \item For every $x,y\in \Reals^d$, we have $ | \inner{x}{y}_A |  \leq \norm{x}_A \norm{y}_A$.
\end{itemize}
\end{lemma}

\begin{lemma}
\label{lem:seminorm2}
Let $A \in \Reals^{d\times d}$ be a positive semi-definite matrix. Then, for every $M \in \Reals^{d\times d}$ define 
$$
\norm{M}_A \triangleq \sup_{x\in \Reals^d} \frac{\norm{Mx}_A}{\norm{x}_A}.
$$
Then, $\norm{M}_A \norm{x}_A \geq \norm{Mx}_A$. Also, for $M, M' \in \Reals^{d\times d}$, we have $\norm{M+M'}_A \leq \norm{M}_A + \norm{M'}_A$.   
\end{lemma}

The following lemma summarizes some of the properties of the logistic loss that will be used in the proof.
\begin{lemma}
\label{lem:logloss-prop}
Fix $n\in \Naturals$ and data set $\trainset = ((x_1,y_1),\dots,(x_n,y_n)) \in (\Reals^d \times \{-1,+1\})^n$. Let $V \in \Reals^{d\times d}$ denote the orthogonal projection matrix on the linear subspace spanned by $\{x_1,\dots,x_n\}$. Then, the following holds:
\begin{enumerate}
    \item For every $w \in \Reals^d$ and $w' \in \Reals^d$  
\[
\nonumber
\norm{\nabla^2 \losslog(w',\trainset) - \nabla^2 \losslog(w,\trainset) } _V \leq 0.1 \cdot \norm{w'-w}_V.
\]
\item For every $w \in \Reals^d$, $u^\top \nabla^2 \losslog(w,\trainset)u = 0 \Leftrightarrow Vu = 0.$ In words, the directions of zero eigenvalue of $\nabla^2 \losslog(w,\trainset)$ are orthogonal to the linear subspace spanned by $\{x_1,\dots,x_n\}$. 
\item For every $w\in \Reals^d$, the eigenvectors of $\nabla^2 \losslog(w,\trainset)$ corresponding to non-zero eigenvalue lie in the linear subspace spanned by $\{x_1,\dots,x_n\}$.
\item Fix $w\in \Reals^d$ and consider the eigenvalue decomposition of $\nabla^2 \losslog(\ww,\trainset)$ as $\sum_{i=1}^{d} \lambda_i u_i u_i^\top$ where $\{\lambda_i \in \Reals: i\in\range{d}\}$ and $\{u_i\in \Reals^d: i \in \range{d}\}$ denote the eigenvalues and eigenvectors. Let $\lambda_{\text{min},\ww} = \min\{\lambda_i: \lambda_i >0\}$. Then, 
\begin{equation*}
    \lambda_{\text{min},w}  = \min_{u \in \text{span}\{x_1,\dots,x_n\},\norm{y}=1} u^\top \nabla^2 \losslog(\ww,\trainset) u.
\end{equation*}
Also, 
\begin{equation*}
     \lvert\lambda_{\text{min},\ww} -  \lambda_{\text{min},\ww'} \rvert \leq \norm{ \nabla^2 \losslog(\ww,\trainset) -  \nabla^2 \losslog(\ww',\trainset)}_V.
\end{equation*}
\item Let $\lambda_0>0$. For every $\ww\in \Reals^d$,
 \small 
\begin{equation*}
    \norm{\Psi_{\lambda_0}\left( \nabla^2 \losslog(\ww,\trainset),\mathsf{clip}\right) }_V = \frac{1}{\max\{\lambda_0,\lambda_{\text{min},w}\}}, \quad   \norm{\Psi_{\lambda_0}\left( \nabla^2 \losslog(\ww,\trainset),\mathsf{add}\right) }_V = \frac{1}{\lambda_0 + \lambda_{\text{min},w}},
\end{equation*}
 \normalsize 
where $\Psi_{\lambda_0}(\cdot,\cdot)$ is defined in \cref{def:modif}.
\end{enumerate}
\end{lemma}
\begin{proof}
For Part 1, from \cref{eq:hess-logloss}, we know that for every $w$
\*[
\nabla^2 \losslog(w,\trainset) = \frac{1}{n}\sum_{i=1}^{n} \frac{x_ix_i^\top}{\left(\exp\left(-\inner{w}{x_i}/2\right) +\exp\left(\inner{w}{x_i}/2\right)\right)^2 }
\]
For every $i \in \range{n}$, let $g: \Reals \to \Reals$ be $g(t)= \frac{1}{\left(\exp\left(-t/2\right) +\exp\left(t/2\right)\right)^2}$. Then 
\begin{align*}
\norm{\nabla^2 \losslog(w,\trainset) - \nabla^2 \losslog(w',\trainset)}_V &= \frac{1}{n} \norm{\sum_{i=1}^{n} x_i x_i^\top (g(\inner{w}{x_i})-g(\inner{w'}{x_i})}_V \\
&\leq  \frac{1}{n} \sum_{i=1}^{n} \norm{x_i x_i^\top }_V \max_{i\in \range{n}} |(g(\inner{w}{x_i})-g(\inner{w'}{x_i})|\\
&\leq \max_{i\in \range{n}} \|g(\inner{w}{x_i})-g(\inner{w'}{x_i}\|,
\end{align*}
where the second and third steps follow form \cref{lem:seminorm2} and $\norm{x_ix_i^\top}_V=\norm{x_i x_i^\top}_2\leq 1$. It is easy to show  there exists $\liph< 0.1$ such that $g$ is $\liph$-Lipschitz. Therefore, 
\*[
\|g(\inner{w}{x_i})-g(\inner{w'}{x_i}\| &\leq \liph \lvert\inner{w-w'}{x_i} \rvert\\
& = \liph \lvert\inner{w-w'}{x_i}_V \rvert\\
& \leq \norm{w-w'}_V,
\]
where the second step follows from $\inner{w-w'}{x_i} = (w-w')^\top x_i = (w-w')^\top V x_i $ since $x_i = V x_i$ by definition. Also, the last step follows from \cref{lem:seminorm1}.

For Part 2, by \cref{eq:hess-logloss}, we have
\*[
u^\top \nabla^2 \losslog(w,\trainset) u = \frac{1}{n}\sum_{i=1}^{n} \frac{(x_i^\top u)^2}{\left(\exp\left(-\inner{w}{x_i}/2\right) +\exp\left(\inner{w}{x_i}/2\right)\right)^2 }
\]
Notice that  every summand is positive, therefore, given $u\in \Reals^d$ such that $u^\top \nabla^2 \losslog(w,\trainset) u =0$ implies that for every $i\in \range{n}$, $x_i^\top u = 0$. The other direction is obvious.

The proof of Part 3 follows from the definition of eigenvalues. Let $u \in \Reals^d$ be an eigenvector corresponding to eigenvalue of $\lambda>0$, then
\*[
\nabla^2 \losslog(w,\trainset) u = \lambda u \imp \sum_{i=1}^{n} \frac{x_i^\top u}{n \lambda\left(\exp\left(-\inner{w}{x_i}/2\right) +\exp\left(\inner{w}{x_i}/2\right)\right)^2 } x_i =  u,
\]
which shows that $u$ is a linear combination of $x_i$s.

For Part 4, the first statement is a corollary of Part 3. For the second statement, let $u \in \Reals^d$ be the eigenvector corresponding to $\lambda_{\text{min},w}$. Then, 
\begin{align*}
    \lambda_{\text{min},\ww'} - \lambda_{\text{min},\ww} &= \min_{u' \in \text{span}\{x_1,\dots,x_n\},\norm{u'}=1} (u')^\top \nabla^2 \losslog(\ww',\trainset) (u')  - u^\top \nabla^2 \losslog(\ww',\trainset) u\\
    &\leq u^\top  \nabla^2 \losslog(\ww,\trainset) u  - u^\top \nabla^2 \losslog(\ww',\trainset) u\\
    & = u^\top \left(\nabla^2\losslog(\ww,\trainset) - \nabla^2\losslog(\ww',\trainset) \right) u\\
    &= u^\top V \left(\nabla^2\losslog(\ww,\trainset) - \nabla^2\losslog(\ww',\trainset) \right) u \\
    &\leq \norm{u}_V \norm{\nabla^2\left(\losslog(\ww,\trainset) - \nabla^2\losslog(\ww',\trainset) \right) u}_V\\
    &\leq \norm{\nabla^2\losslog(\ww,\trainset) - \nabla^2\losslog(\ww',\trainset) }_V.
\end{align*}
Part 4 is based on the definition of the matrix norm in \cref{lem:seminorm2} and Parts 2, 3.
\end{proof}

Let $\trainset = ((x_1,y_1),\dots,(x_n,y_n)) \in (\Reals^d \times \{-1,+1\})^n$. Let $V \in \Reals^{d\times d}$ denote the orthogonal projection matrix on the linear subspace spanned by $\{x_1,\dots,x_n\}$.  Let $\ww_t^\star$ denote the minimizer of the empirical loss. We assume it exists and $\grad \losslog(\ww_t^\star,\trainset)=0$. To reduce notation clutter we drop $\trainset$. Let $\xi_1 \sim \Normal(0,I_d)$ and $\xi_2 \sim \Normal(0,I_d)$. We can rephrase the update rule of \cref{alg:main-opt} as  
\begin{align}
    \norm{\ww_{t+1}-\ww^\star}_V^2  &= \norm{\ww_{t}-\ww^\star - \tilde{H}_t^{-1}\left(\nabla \losslog(\ww_t) + \sigma_1 \xi_1\right) + \norm{\nabla \losslog(\ww_t) + \sigma_1 \xi_1}_2 \sigma_2 \xi_2}_V^2 \nonumber\\
                                    & = \norm{\ww_{t}-\ww^\star - \tilde{H}_t^{-1} \nabla \losslog(\ww_t)}^2_V +\norm{\sigma_1 \tilde{H}_t^{-1} \xi_1 - \norm{\nabla \losslog(\ww_t) + \sigma_1 \xi_1}_2 \sigma_2 \xi_2}_V^2  \nonumber \\
                                    & -2 \inner{\ww_{t}-\ww^\star - \tilde{H}_t^{-1} \nabla \losslog(\ww_t)}{\sigma_1 \tilde{H}_t^{-1} \xi_1 - \norm{\nabla \losslog(\ww_t) + \sigma_1 \xi_1}_2 \sigma_2 \xi_2}_V. 
                     \label{eq:main-recur-local}
\end{align}

In the next step we analyze the first term in \cref{eq:main-recur-local}: 
\begin{align*}
   &\norm{\ww_{t} - \tilde{H_t}^{-1} \grad \losslog(\ww_{t}) - \ww^\star}^2_V \\
    &= \norm{\ww_{t}- \ww^\star - \tilde{H_t}^{-1} \left(\grad \losslog(\ww_{t}) - \grad \losslog(\ww^\star)\right)}^2_V\\
    & = \norm{\ww_{t}- \ww^\star - \tilde{H_t}^{-1} \left(\int_{0}^{1} \nabla^2 \losslog(\ww^\star + \tau(\ww_t - \ww^\star))(\ww_t - \ww^\star) d\tau \right)}^2_V\\
    & =  \norm{\ww_{t}- \ww^\star - \tilde{H_t}^{-1} \left(\int_{0}^{1} \left[ \nabla^2 \losslog(\ww^\star + \tau(\ww_t - \ww^\star)) -  \nabla^2 \losslog(\ww_t) + \nabla^2 \losslog(\ww_t)\right](\ww_t - \ww^\star) d\tau \right)}^2_V. \\
\end{align*}
For every $w\in \Reals^d$ and $\tau 
\in [0,1]$, let $\Delta_\tau(w) =  \nabla^2 \losslog(\ww^\star + \tau(\ww - \ww^\star)) -  \nabla^2 \losslog(\ww) $. We write
\begin{align*}
   &\norm{\ww_{t} - \tilde{H_t}^{-1} \grad \losslog(\ww_{t}) - \ww^\star}^2_V \\
   &\leq  \norm{\ww_{t}- \ww^\star - \tilde{H}_t^{-1} \nabla^2 \losslog(\ww_t)(\ww_t - \ww^\star)}_V^2 + \norm{\tilde{H}_t^{-1}\left( \int_{0}^1\Delta_\tau(\ww_t) (\ww_t - \ww^\star) d\tau \right) }^2_V \\ 
    &+  2 \norm{\ww_{t}- \ww^\star - \tilde{H}_t^{-1} \nabla^2 \losslog(\ww_t)(\ww_t - \ww^\star)}_V \norm{\tilde{H}_t^{-1}\left( \int_{0}^1\Delta_\tau(\ww_t) (\ww_t - \ww^\star) d\tau \right)}_V\\
    &\leq \norm{I - \tilde{H}_t^{-1} \nabla^2 \losslog(\ww_t)}_V^2 \norm{\ww_t - \ww^\star}_V^2 + \norm{\tilde{H}_t^{-1}}_V^2 \left(\int_{0}^{1}  \norm{\Delta_\tau(\ww_t)}_V d \tau \right)^2 \norm{\ww_t - \ww^\star}^2_V \\
    & + 2 \norm{I - \tilde{H}_t^{-1} \nabla^2 \losslog(\ww_t)}_V \norm{\tilde{H}_t^{-1}}_V \left(\int_{0}^{1}  \norm{\Delta_\tau(\ww_t)}_V d \tau \right) \norm{\ww_t - \ww^\star}_V^2.
\end{align*}
Here, we repeatedly use the properties of $\norm{\cdot}_V$ from \cref{lem:seminorm1,lem:seminorm2}. 

Consider the eigenvalue decomposition of $\nabla^2 \losslog(\ww_t) = \sum_{i=1}^{d}\lambda_i u_i u_i^\top$ where some $\lambda_i$ may be zero since we do not assume that $\nabla^2 \losslog(\ww_t)$ is a full-rank matrix. Let $\lambda_{\text{min},t}$ be the smallest \emph{non-zero} eigenvalue of $\nabla^2 \losslog(\ww_t)$. Then, by \cref{def:modif}
\[
\label{eq:hess-inverse-closedform}
\tilde{H}_t^{-1} =
\begin{cases}
\displaystyle \sum_{i=1}^{d}  \frac{1}{\max\{\lambda_i,\lambda_0\}} u_i u_i^\top  & \text{if Hessian modification is $\mathsf{clip}$}, \\
\displaystyle \sum_{i=1}^{d} \frac{1}{\lambda_i + \lambda_0} u_i u_i^\top & \text{if Hessian modification is $\mathsf{add}$},
\end{cases}
\]
and
\[
\label{eq:identity-minus-closedform}
 I- \tilde{H}_t^{-1} \nabla^2 \losslog(\ww_t) = \begin{cases}
\displaystyle \sum_{i: \lambda_i < \lambda_0} \left(1-\frac{\lambda_i}{\lambda_0}\right)u_i u_i^\top  & \text{if Hessian modification is $\mathsf{clip}$}, \\
\displaystyle \sum_{i=1}^{d} \frac{\lambda_i}{\lambda_i + \lambda_0} u_i u_i^\top & \text{if Hessian modification is $\mathsf{add}$},
 \end{cases}
\]
Therefore from \cref{eq:hess-inverse-closedform}, \cref{eq:identity-minus-closedform}, and \cref{lem:logloss-prop},
\[
\nonumber 
\norm{\tilde{H}_t^{-1}}_V = 
\begin{cases}
\displaystyle \frac{1}{\max\{\lambda_0,\lambda_{\text{min},t}\}}  & \text{if Hessian modification is $\mathsf{clip}$}, \\
\displaystyle \frac{1}{\lambda_0 + \lambda_{\text{min},t}} & \text{if Hessian modification is $\mathsf{add}$},
\end{cases}
\]
and
\[
\nonumber 
\norm{I- \tilde{H}_t^{-1} \nabla^2 \losslog(\ww_t)}_V = 
\begin{cases}
\displaystyle 1 - \frac{\min\{\lambda_0,\lambda_{\text{min},t}\}}{\lambda_0}  & \text{if Hessian modification is $\mathsf{clip}$}, \\
\displaystyle 1- \frac{\lambda_{\text{min},t}}{\lambda_0 + \lambda_{\text{min},t}} & \text{if Hessian modification is $\mathsf{add}$}.
\end{cases}
\]
Also, from  \cref{lem:logloss-prop},
\[
\int_{0}^{1} \norm{\Delta_\tau(\ww_t)}_V d \tau &= \int_{0}^{1} \norm{\nabla^2 \losslog(\ww^\star + \tau (\ww_t - \ww^\star)) - \nabla^2 \losslog(\ww_t) }_V d\tau  \nonumber\\
& \leq \frac{0.1}{2}  \norm{\ww_t - \ww^\star}_V,
\]

Therefore, 

\[
&\norm{\ww_{t} - \tilde{H_t}^{-1} \grad \losslog(\ww_{t}) - \ww^\star}^2_V  \nonumber \leq  \norm{I - \tilde{H}_t^{-1} \nabla^2 \losslog(\ww_t)}_V^2 \norm{\ww_t - \ww^\star}_V^2\\
& + 0.1 \cdot \norm{I - \tilde{H}_t^{-1} \nabla^2 \losslog(\ww_t)}_V \norm{\tilde{H}_t^{-1}}_V \norm{\ww_t - \ww^\star}_V^3 + \frac{(0.1)^2}{4} \norm{\tilde{H}_t^{-1}}_V^2 \norm{\ww_t - \ww^\star}_V^4. \label{eq:first-term-hess-expansion}
\]
Consider the second term in \cref{eq:main-recur-local}. Using the facts that $\EE[\xi_1]=\EE[\xi_2]=0$, $\xi_1 \indep \xi_2$, $(\xi_1,\xi_2)\indep \ww_t$, and $\EE[\norm{\xi_1}^2_V]=\EE[\norm{\xi_2}^2_V]=\rank$, we obtain
\begin{align}
   &\EE_t\left[\norm{\sigma_1 \tilde{H}_t^{-1} \xi_1 - \norm{\nabla \losslog(\ww_t) + \sigma_1 \xi_1}_2 \sigma_2 \xi_2}_V^2\right] \nonumber \\
   &=  \sigma_1^2 \EE_t\left[\norm{\tilde{H}_t^{-1} \xi_1}_V^2\right] +   \sigma_2^2\EE_t\left[\norm{\nabla \losslog(\ww_t) + \sigma_1 \xi_1}_2^2  \norm{\xi_2}_V^2\right] \nonumber\\
    &= \sigma_1^2 \norm{\tilde{H}^{-1}_t}_V^2  \EE_t\left[\norm{\xi_1}_V^2\right]  +   \sigma_2^2  \EE_t\left[\norm{\nabla \losslog(\ww_t) + \sigma_1 \xi_1}_2^2\right]  \EE_t\left[\norm{\xi_2}_V^2\right] \nonumber\\
    & \leq \sigma_1^2 \norm{\tilde{H}^{-1}_t}_V^2 \EE_t\left[\norm{\xi_1}_V^2\right] + \sigma_2^2 \norm{\nabla \losslog(\ww_t)}_2^2 \EE_t\left[\norm{\xi_2}^2_V\right] \nonumber\\
    & + \sigma_2^2 \sigma_1^2 \EE_t\left[\norm{\xi_1}^2\right] \EE_t\left[\norm{\xi_2}^2_V\right] \nonumber\\
    &=  \sigma_1^2 \norm{\tilde{H}^{-1}_t}_V^2 \rank + \sigma_2^2 \norm{\ww_t - \ww^\star}_V^2 \rank + \sigma_2^2 \sigma_1^2 d \ \rank. \label{eq:second-term-hess-expansion}
\end{align}
Notice that the expectation of the third term in \cref{eq:main-recur-local} is zero. By combining \cref{eq:main-recur-local}, \cref{eq:first-term-hess-expansion}, and \cref{eq:second-term-hess-expansion} we obtain
\[
&\EE_t\left[\norm{\ww_{t+1}-\ww^\star}_V^2 \right] \leq \left(\norm{I - \tilde{H}_t^{-1} \nabla^2 \losslog(\ww_t)}_V^2  + \sigma_2^2 \cdot \rank \right) \norm{\ww_t - \ww^\star}_V^2 \nonumber \\
&+ 0.1 \cdot \norm{I - \tilde{H}_t^{-1} \nabla^2 \losslog(\ww_t)}_V \norm{\tilde{H}_t^{-1}}_V \norm{\ww_t - \ww^\star}_V^3 + \frac{(0.1)^2}{4} \norm{\tilde{H}_t^{-1}}_V^2 \norm{\ww_t - \ww^\star}_V^4 \nonumber \\
&+ \sigma_1^2 \norm{\tilde{H}^{-1}_t}_V^2 \rank +  \sigma_2^2 \sigma_1^2 d \ \rank.
\]
Finally, setting the values of $\sigma_1$ and $\sigma_2$ from \cref{alg:main-opt} completes the proof.
\subsection{Global Convergence of of \quc and \qua} \label{sec:proof-qu-convergence}
\begin{theorem}[Global Convergence Guarantee of \quc and \qua] \label{thm:practical-global-convergence}
Let $\lambda_{\text{min}}^\star$ be the minimum non-zero eigenvalue of $\nabla^2\losslog(\ww^\star,\trainset)$, $\rho$ be the privacy budget (in zCDP) per iteration,  $\delta_t = \losslog(\ww_{t},\trainset) - \losslog(\ww^\star,\trainset) $ be the suboptimality gap at iteration $t$, and $\lambda_{\text{max},t}$ be the maximum eigenvalue of $H_\text{qu}(\ww_t,\trainset)$ from \cref{lem:quad-ub}. Let 
$$
\tilde{\lambda}_{\text{max},t}= \begin{cases}
\displaystyle \max\{\lambda_0, \lambda_{\text{max},t}\}  & \text{if SOI modification is $\mathsf{clip}$}, \\
\displaystyle \left(\lambda_0 + \lambda_{\text{max},t}\right) & \text{if SOI modification is $\mathsf{add}$}.
\end{cases}
$$
Then, if $\norm{\nabla \losslog(\ww_t,\trainset)} \geq \frac{3 \lambda_{\text{min}}^\star}{4}$
\[
\EE_t \left[ \delta_{t+1} \right] \leq \delta_t -\frac{9}{8} \lambda^\star_{\text{min}} \cdot \nu + \Delta,
\]
Also, if $\norm{\nabla \losslog(\ww_t,\trainset)} < \frac{3 \lambda_{\text{min}}^\star}{4} $, we have
\[
\EE_t \left[ \delta_{t+1} \right] \leq  \left( 1 -  \nu \right) \delta_t + \Delta \label{eq:quad-ub-exponential},
\]
where 
\*[
&\nu = \frac{\lambda^\star_{\text{min}}}{4 \tilde{\lambda}_{\text{max},t}} - \frac{\lambda_{\text{max}} \lambda^\star_{\text{min}} \rank }{8 \rho \theta \left(4n \lambda_0^2 -\lambda_0\right)^2},\quad \Delta = O\left(\frac{\rank}{4\lambda_0 \rho \theta(1-\theta)n^2}\right), \quad \text{if SOI modification is $\mathsf{clip}$}. \\
& \nu = \frac{\lambda^\star_{\text{min}}}{4 \tilde{\lambda}_{\text{max},t}} - \frac{\lambda_{\text{max}} \lambda^\star_{\text{min}} \rank }{8 \rho \theta \left(4n \lambda_0^2 +\lambda_0\right)^2},\quad \Delta = O\left(\frac{\rank}{4\lambda_0 \rho \theta(1-\theta)n^2}\right), \quad \text{if SOI modification is $\mathsf{add}$}.
\]
\end{theorem}
\begin{proof}
We begin the proof by quoting a result from \citep{bach2014adaptivity}. Note that the statement in \citep{bach2014adaptivity} is stated in terms of $\norm{\cdot}_2$, but the extension to the norm induced by $V$ is straightforward.
\begin{lemma}[{\citealp[Lemma~9]{bach2014adaptivity}}] \label{lem:eigopt-logistic}
Let $\trainset$ be a training set and $\ww^\star = \argmin \losslog(\ww,\trainset)$. Let $\lambda_{\text{min}}^\star$ be the minimum non-zero eigenvalue of $\nabla^2\losslog(\ww^\star,\trainset)$. Then, for every $\ww$ such that $\norm{\nabla \losslog(w,\trainset)}\leq \frac{3}{4}\lambda_{\text{min}}^\star$, we have
\begin{equation*}
    \losslog(w,\trainset) - \losslog(\ww^\star,\trainset) \leq 2 \frac{\norm{\nabla \losslog(w,\trainset)}_V^2}{\lambda_{\text{min}}^\star}.
\end{equation*}
\end{lemma}
Define the $\EE_t\left[ \cdot \right]$ as the conditional expectation conditioned on the history up to time $t$, i.e., $\{\ww_{0},\dots,\ww_{t}\}$. We can write by \cref{lem:quad-ub}
\begin{align}\label{eq:quadratic-ub-proof}
    \losslog(\ww_{t+1}) \leq \losslog(\ww_{t}) + \inner{\nabla \losslog(\ww_{t})}{\ww_{t+1} - \ww_{t}}  + \frac{1}{2} (\ww_{t+1}-\ww_{t})^\top  H_{\text{qu},t} (\ww_{t+1}-\ww_{t}).   
\end{align}
Let $g_t = \nabla \losslog(\ww_{t})$, $\xi_1 \sim \Normal(0,I_d)$, $\xi_2 \sim \Normal(0,I_d)$. Then, by the definition of the update rule,  
\begin{align*}
    \ww_{t+1} - \ww_{t} = -\tilde{H}_t^{-1} \left(g_t + \sigma_1 \xi_1\right) + \norm{g_t + \sigma_1 \xi_1} \sigma_2 \xi_2.
\end{align*}
We use $\tilde{H}_t$ to denote $\Psi_{\lambda_0}(H_{\text{qu},t},\text{SOI modification})$. Then,
\begin{align}
\EE_t \left[ \inner{\nabla \losslog(\ww_{t})}{\ww_{t+1} - \ww_{t}} \right] &= \EE_t \left[ \inner{g_t}{-\tilde{H}_t^{-1} \left(g_t + \sigma_1 \xi_1\right) + \norm{g_t + \sigma_1 \xi_1} \sigma_2 \xi_2} \right] \nonumber\\
& = -g_t^\top \tilde{H}_t^{-1} g_t, \label{eq:qu-proof-gradterm}
\end{align}
where the last step follows from $\xi_1$ and $\xi_2$ being independent of $\ww_t$. For the third term on RHS of \cref{eq:quadratic-ub-proof}, using $\EE[\xi_1]=\EE[\xi_2]=0$ and $\xi_1 \indep \xi_2$, we can write
\begin{align*}
    &\EE_t \left[(\ww_{t+1}-\ww_{t})^\top  H_{\text{qu},t} (\ww_{t+1}-\ww_{t})\right] \\
    &= \EE_t \left[ (g_t + \sigma_1 \xi_1)^\top \tilde{H}_t^{-1} H_{\text{qu},t} \tilde{H}_t^{-1} (g_t + \sigma_1 \xi_1) \right] +\sigma_2^2 \EE_t \left[ \norm{g_t + \sigma_1 \xi_1}^2 \xi_2^\top H_{\text{qu},t} \xi_2  \right] \\ 
    &=  g_t^\top \tilde{H}_t^{-1} H_{\text{qu},t} \tilde{H}_t^{-1}g_t + \sigma_1^2 \EE_t \left[\xi_1^\top \tilde{H}_t^{-1} H_{\text{qu},t} \tilde{H}_t^{-1}\xi_1\right] \\
    &+ \left(\sigma_2^2 \norm{g_t}^2 + \sigma_1^2\sigma_2^2 \EE_t\left[\norm{\xi_1}^2\right]\right) \EE_t \left[\xi_2^\top H_{\text{qu},t} \xi_2\right].
\end{align*}

By the definition of the modification operators in \cref{def:modif} we have $\tilde{H}_t^{-1} H_{\text{qu},t} \tilde{H}_t^{-1} \preccurlyeq \tilde{H}_t^{-1}$.  Also, by the fact that for a symmetric matrix $A$ and $\xi \sim \Normal(0,I_d)$, it holds $\EE\left[\xi^\top A \xi\right] = \text{trace}(A)$, we can write
\begin{equation}
    \begin{aligned}  \label{eq:qu-proof-hessterm}
    &\frac{1}{2} \EE_t \left[(\ww_{t+1}-\ww_{t})^\top  H_{\text{qu},t} (\ww_{t+1}-\ww_{t})\right] \\
    &\leq -\frac{1}{2} g_t^\top \tilde{H}_t^{-1} g_t + \frac{1}{2}\sigma_1^2 \text{trace}(\tilde{H}_t^{-1} H_{\text{qu},t} \tilde{H}_t^{-1}) + \frac{1}{2}\sigma_2^2 \norm{g_t}^2 \text{trace}(H_{\text{qu},t}) + \frac{1}{2}\sigma_1^2 \sigma_2^2 \ d \ \text{trace}(H_{\text{qu},t})\\
    &\leq -\frac{1}{2} g_t^\top \tilde{H}_t^{-1} g_t + \frac{1}{2}\sigma_1^2 \frac{\rank}{\lambda_0} + \frac{\lambda_{\text{max},t}}{2}\sigma_2^2 \norm{g_t}^2 \rank + \frac{\lambda_{\text{max},t}}{2}\sigma_1^2 \sigma_2^2 \cdot  d \cdot \rank,
    \end{aligned}
\end{equation}
where the last line follows from $\text{trace}(\tilde{H}_t^{-1} H_{\text{qu},t} \tilde{H}_t^{-1}) \leq \frac{\rank}{\lambda_0}$ and $\text{trace}(H_{\text{qu},t}) \leq \rank \cdot \lambda_{\text{max},t} $ where the maximum eigenvalue of $H_{\text{qu},t}$ is denoted by $\lambda_\text{max,t}$. Also, 
\[
\tilde{H}_t \preccurlyeq \tilde{\lambda}_{\text{max,t}} I_d \triangleq
\begin{cases}
\displaystyle \max\{\lambda_0, \lambda_\text{max,t}\} I_d & \text{if Hessian modification is $\mathsf{clip}$}, \\
\displaystyle (\lambda_0 + \lambda_\text{max,t}) I_d  & \text{if Hessian modification is $\mathsf{add}$}.
\end{cases}
\]

Therefore, 
\begin{align}
    &\EE_t \left[ \losslog(\ww_{t+1}) - \losslog(\ww_{t}) \right] \leq  -\frac{1}{2}g_t^\top \tilde{H}_t^{-1} g_t + \frac{1}{2}\sigma_1^2 \frac{\rank}{\lambda_0} +  \frac{\lambda_{\text{max},t}}{2}\sigma_2^2 \norm{g_t}^2 \rank + \frac{\lambda_{\text{max},t}}{2}\sigma_1^2 \sigma_2^2 \cdot  d \cdot \rank  \nonumber \\
    &\leq -\frac{1}{2}\norm{g_t}^2 \left( \frac{1}{\tilde{\lambda}_{\text{max},t}} -  \sigma_2^2 \cdot \rank \cdot \lambda_{\text{max},t}\right) +  \frac{1}{2}\sigma_1^2 \frac{\rank}{\lambda_0}  + \frac{\lambda_{\text{max},t}}{2}\sigma_1^2 \sigma_2^2 \cdot  d \cdot \rank \label{eq:quad-ub-progress},
\end{align}
where the last step follows from the fact that for every $u\in \Reals^d$, $ \displaystyle u^\top \tilde{H}_t^{-1} u \geq  \frac{1}{\tilde{\lambda}_{\text{max,t}}} \norm{u}^2 $.

In the last step, we will use \cref{lem:eigopt-logistic}. Let $\lambda_{\text{min}}^\star$ be the minimum non-zero eigenvalue of $\nabla^2 \losslog(\ww^\star,\trainset)$. Since $g_t$ is a linear combination of $x_i$s (See \cref{eq:hess-logloss}, we have $\norm{g_t}_2 = \norm{g_t}_V $. Consider two cases: Case 1) $\norm{g_t}_V> \frac{3}{4}\lambda_{\text{min}}^\star$, Case 2) $\norm{g_t}_V\leq \frac{3}{4}\lambda_{\text{min}}^\star$.

For Case 1, we can simplify \cref{eq:quad-ub-progress} as follows
\[
&\EE_t \left[ \losslog(\ww_{t+1}) - \losslog(\ww^\star) \right] \leq   \losslog(\ww_{t}) - \losslog(\ww^\star)  \nonumber\\
& -\frac{9}{32}(\lambda_{\text{min}}^\star)^2 \left( \frac{1}{\tilde{\lambda}_{\text{max},t}} -  \sigma_2^2 \cdot \rank \cdot \lambda_{\text{max},t}\right) +   \frac{1}{2}\sigma_1^2 \frac{\rank}{\lambda_0}  + \frac{\lambda_{\text{max},t}}{2}\sigma_1^2 \sigma_2^2 \  d \ \rank.\nonumber
\]
For the second case, from \cref{lem:eigopt-logistic} we have 
\[
&\EE_t \left[ \losslog(\ww_{t+1}) - \losslog(\ww^\star)\right] \nonumber\\
&\leq \left[ 1-  \frac{\lambda_{\text{min}}^\star}{4\tilde{\lambda}_{\text{max},t}} + \frac{\sigma_2^2 \rank \lambda_{\text{max},t}\lambda^\star_{\text{min}}}{4}   \right] \left(\losslog(\ww_{t}) - \losslog(\ww^\star)\right)  + \sigma_1^2 \frac{\rank}{2\lambda_0}  + \frac{\lambda_{\text{max},t}}{2}\sigma_1^2 \sigma_2^2 \  d \ \rank.
\]
The stated results follow from setting $\sigma_1$ and $\sigma_2$.
\end{proof}

\section{Appendix of \cref{sec:numerical-results}}
\label{appx:num-results}
In this section, we present the details of the implementation and additional experiment results. 

\subsection{Subsampled variant of Our Algorithm}
\label{subsec:stochastic-variant-our}
In this section, we show how to extend \cref{alg:main-opt} to the minibatch version. 

Let's assume we have $m$ queries, denoted as $q_i: \dataspace^\star \to \Reals^d$, where $i \in \range{m}$, and each query has an $\ell_2$ sensitivity of one. We want to sequentially compose these queries using the Sampled Gaussian Mechanism (SGM), which combines subsampling and additive Gaussian noise \citep{mironov2019r}. To determine the appropriate noise level for achieving the desired privacy,  we assume we have an access to function  $\mathsf{get \_ noise \_ multiplier}$ which takes as input the total privacy budget, $m$, and the subsampling probability and outputs the minimum standard deviation of noise for Gaussian Mechanism to achieve the required privacy.  Such a  function can be found in various publicly available privacy libraries.

\begin{theorem}
For every training set $\trainset \in \mathcal{Z}^n$, $\lambda_0 >0$, $\theta \in (0,1)$,  privacy budget $(\varepsilon,\delta)$-DP, initialization $\ww_0$, number of iterations $T$, SOI modification $\in \{\mathsf{clip},\mathsf{add}\}$, and sampling rates $p_g, p_H \in (0,1)$ for gradient and SOI, the output of \cref{alg:main-opt-subsample}, i.e., $\ww_T$ satisfies $(\varepsilon,\delta)$-DP.
\end{theorem}
\begin{proof}
In \cref{alg:main-opt-subsample}, we have two types of SGMs, 1) gradient SGM, 2) SOI SGM. The result from \citep{lyu2022composition,vadhan2022concurrent} indicate that we can \emph{interleave} these mechanisms in a way that we have $T$ composition of only gradient SGM followed by $T$ composition of SOI SGM. Using this observation, the privacy proof is a straightforward extension of the sensitivity analysis in \cref{thm:sens-add} and \cref{thm:sens-clip}.
\end{proof}
\begin{algorithm}
\small
\caption{Newton Method with Double noise - Minibatch Version}\label{alg:main-opt-subsample}
\begin{algorithmic}
\State Inputs: training set $\trainset \in \mathcal{Z}^n$, $\lambda_0 >0$, $\theta \in (0,1)$,  privacy budget $(\varepsilon,\delta)$-DP, initialization $\ww_0$, number of iterations $T$, SOI modification $\in \{\mathsf{clip},\mathsf{add}\}$, sampling rates $p_g, p_H \in (0,1)$ for gradient and SOI. 
\State Set  $\sigma_1 = \mathsf{get \_ noise \_ multiplier}\left(\text{privacy budget}=((1-\theta)\epsilon,(1-\theta)\delta), \text{sampling rate}= p_g , \text{steps}= T\right)$
 \If {SOI modification = $\mathsf{Add}$}
 \State $ \sigma_2 = \frac{1}{(4n p_H \lambda_0^2 + \lambda_0)} \cdot  \mathsf{get \_ noise \_ multiplier}\left(\text{privacy budget}=(\theta \epsilon,\theta \delta), \text{sampling rate}= p_H , \text{steps}= T\right)$
 \ElsIf{SOI modification = $\mathsf{Clip}$} 
 \State $ \sigma_2 =  \frac{1}{(4n p_H \lambda_0^2 - \lambda_0)} \cdot \mathsf{get \_ noise \_ multiplier}\left(\text{privacy budget}=(\theta \epsilon,\theta \delta), \text{sampling rate}= p_H , \text{steps}= T\right)$
 \EndIf
\For{$t=0,\dots,T-1$}
    \State Take a Poisson subsample  $\mathcal{I}_{t,g} \subseteq \range{n}$ with sampling probability $p_g$
    \State Take a Poisson subsample  $\mathcal{I}_{t,H} \subseteq \range{n}$ with sampling probability $p_H$
    \State Query $g_t =  \frac{1}{n p_g} \sum_{i \in  \mathcal{I}_{t,g}}\nabla \logf(\ww_t,z_i)$ and $H_t = \frac{1}{n p_H} \sum_{j \in \mathcal{I}_{t,H}} H(\ww_t,z_j)$  
    \State $\tilde{H}_t = \Psi_{\lambda_0}(H_t,\text{SOI modification})$
    \State $\tilde{g}_t = g_t + \frac{1}{n p_g} \Normal(0,\sigma^2_1 I_d)$ 
    \State $\ww_{t+1} = \ww_{t} - \tilde{H}_t^{-1}\tilde{g}_t + \Normal(0,\norm{\tilde{g}_t}^2 \sigma_2^2  I_d) $
\EndFor
\State Output $\ww_{T}$.
\end{algorithmic}
\end{algorithm}

\subsection{Details of the experiments}
\cref{table:params} summarizes the hyperparmeters of \cref{alg:opt-adaptive} used for the experiments. Notice that these parameters are \emph{data-independent}.
\begin{table}[h!]
\centering
\begin{tabular}{|l|l|}
\hline
\multicolumn{1}{|c|}{Parameter } & \multicolumn{1}{c|}{Value} \\ \hline
\multicolumn{1}{|c|}{$\theta$ : fraction of the privacy budget for the search direction in \cref{alg:opt-adaptive}}                         & \multicolumn{1}{c|}{$0.3$}                         \\ \hline
                  \multicolumn{1}{|c|}{$\gamma$  : fraction of the privacy budget for computing trace in \cref{alg:opt-adaptive}}                             &                   \multicolumn{1}{|c|}{$0.1$}                            \\ \hline
                        \multicolumn{1}{|c|}{$\beta$  : the coefficient for minimum eigenvalue in \cref{alg:opt-adaptive}}                       &                 \multicolumn{1}{|c|}{$\{0.5, 1, 2\}$}    \\ \hline
                         \multicolumn{1}{|c|}{number of independent runs}                       &                 \multicolumn{1}{|c|}{$15$}    \\ \hline
\end{tabular}
\caption{Hyperparmeters  of \cref{alg:opt-adaptive}}
\label{table:params}
\end{table}
\cref{table:datasets} lists the public datasets used in our
experimental evaluation. 
\begin{table}[h!]
\centering
\begin{tabular}{|c|c|c|c|}
\hline
                        dataset name &                         number of samples & dimension & Reference \\ \hline
                         a1a     &               $30956$                           &       $134$  &  \cite{Dua:2019}  \\ \hline
                       adult         &                     $45220$                      &     $104$  &  \cite{Dua:2019} \\ \hline
                          (binary) covertype      &                    $53121$                       &    $55$ &      \cite{blackard1999comparative, Dua:2019} \\ \hline
                           synthetic     &                         $10000$                  &    $100$ &   See \cref{sec:numerical-results}   \\ \hline
                            (binary) FMNIST     &     $12000$                  &    $784$ &  \cite{xiao2017fashion} \\ \hline
                             protein      &     $50000$                  &    $74$ &  \cite{caruana2004kdd} \\ \hline
\end{tabular}
\caption{Datasets used in the experiments}
\label{table:datasets}
\end{table}

\subsection{Privacy-Utility-Run Time tradeoff}

\begin{figure}[H]
    \centering
     \begin{subfigure}[t]{0.37\textwidth}
        \centering
          \includegraphics[width=\linewidth]{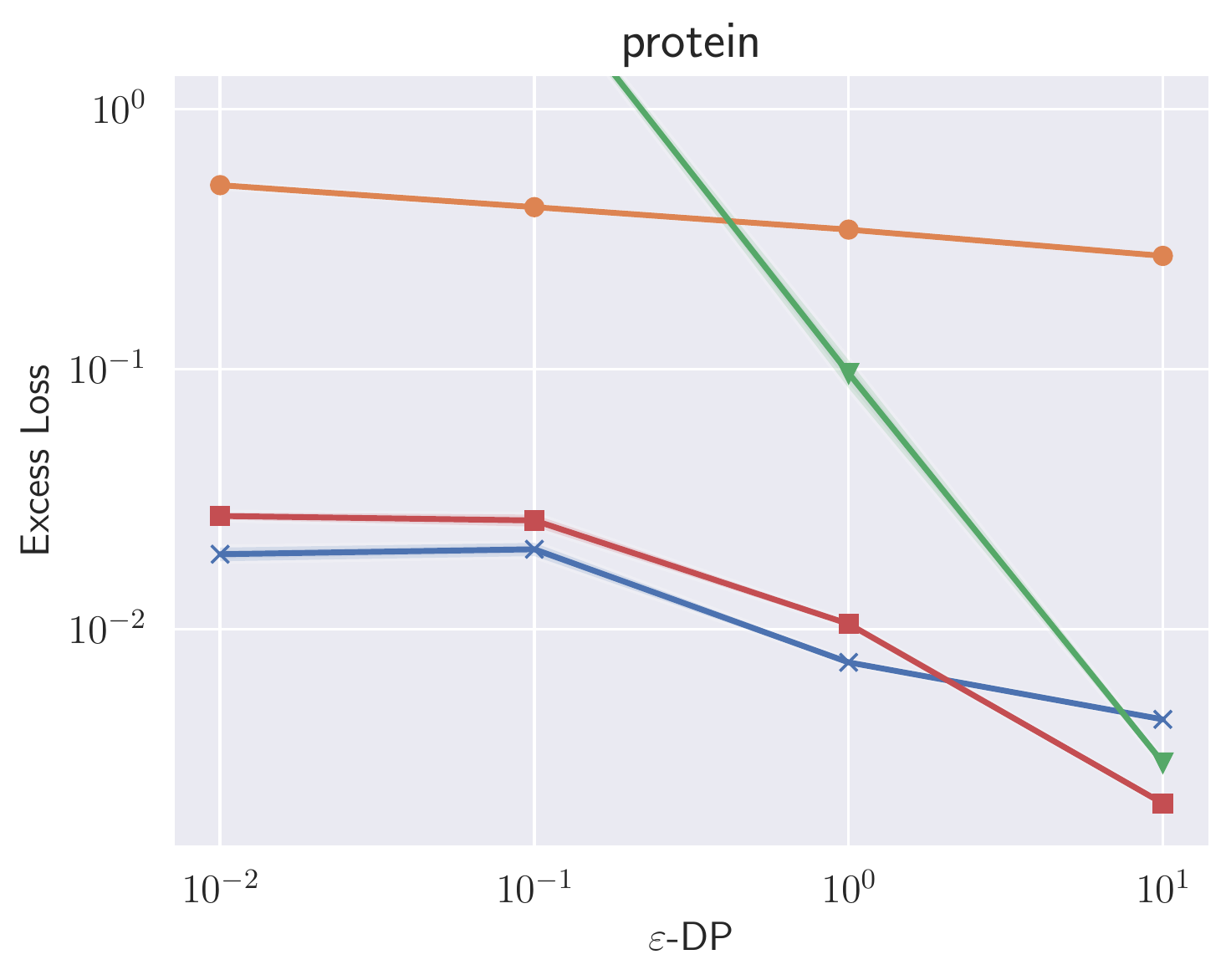}
    \end{subfigure}%
    ~ \hspace{-1.0em}
    \begin{subfigure}[t]{0.6\textwidth}
        \centering
          \includegraphics[width=\linewidth]{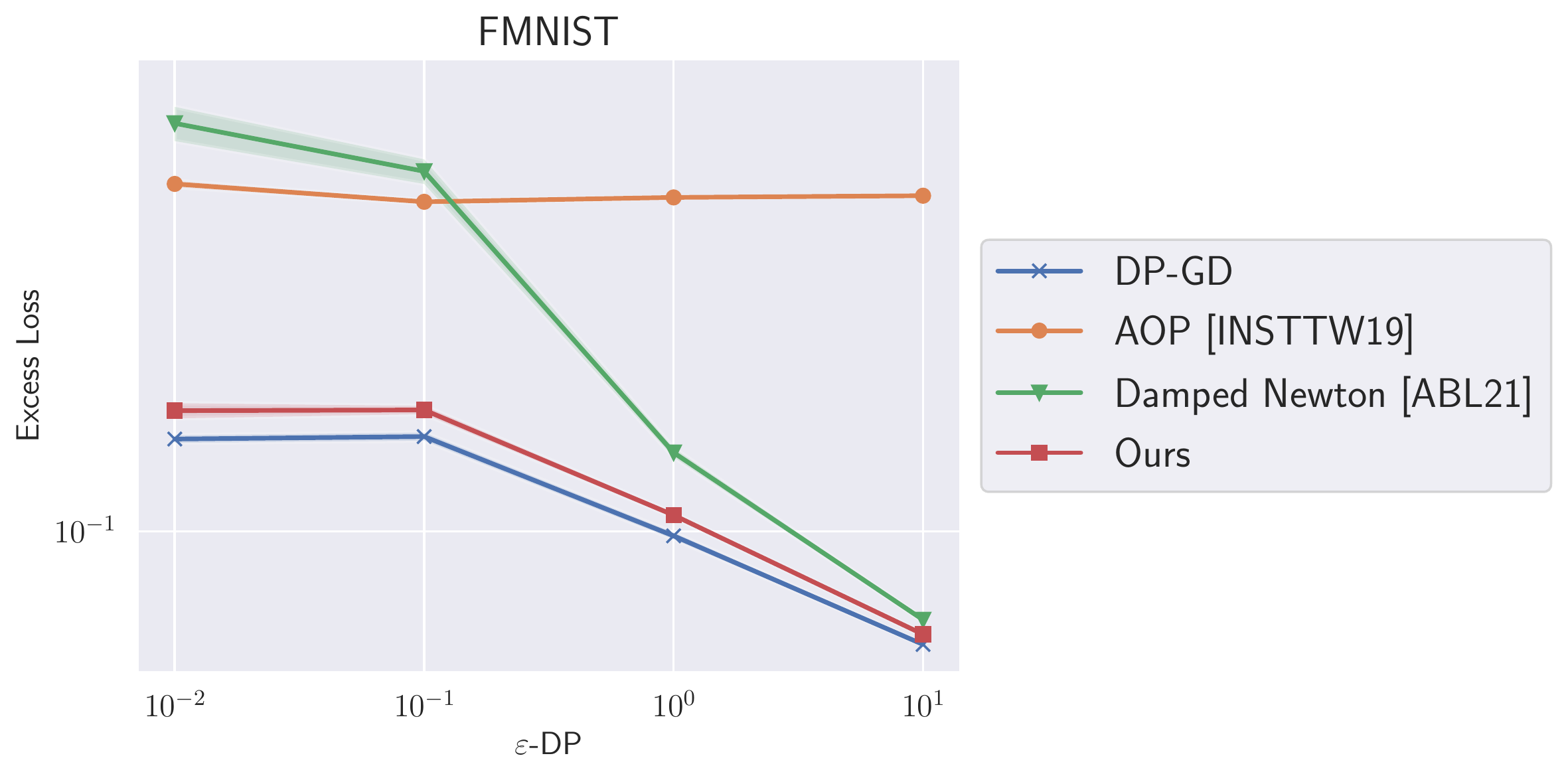}
    \end{subfigure}
    \caption{Privacy-Excess Loss Tradeoff for FMNIST and protein}
    \label{fig:fmnist-priv-utility}
\end{figure}

\begin{table}
\centering
\begin{tabular}{lcccccc}\toprule
& \multicolumn{4}{c}{$\frac{T^\star_{\text{DP--GD}}}{T^\star_{\text{ours}}}$} & \multicolumn{2}{c}{$T^\star_{\text{ours}} (\text{sec})$}
\\\cmidrule(lr){2-5}\cmidrule(lr){6-7}
           & $\varepsilon=0.01$ & $\varepsilon=0.1$ & $\varepsilon=1$ & $\varepsilon=10$ & $\min(T^\star_{\text{ours}})$ (sec.)  & $\max(T^\star_{\text{ours}})$ (sec.)\\\midrule
FMNIST    &  $3.44 \times $  &	$2.79\times$ &	$2.77\times$ 	& $8.74\times$   &      $ 11.36$    &      $ 25.61$    \\
protein & $6.65 \times $  &	$9.62\times$ &	$24.16\times$ 	& $26.99\times$   &      $ 3.99$    &      $ 4.66$ \\\bottomrule
\end{tabular}
\caption{Comparison between the run time of our algorithm and DP-GD in terms of the ratio $\displaystyle {T^\star_{\text{DP-GD}}}/{T^\star_{\text{our}}}$. The last two columns show the minimum and maximum run time of our algorithm. }
\label{table:walltime-appx}
\end{table}

\subsection{Minibatch Variant}

\begin{figure}[H]
    \centering
     \begin{subfigure}[t]{0.37\textwidth}
        \centering
          \includegraphics[width=\linewidth]{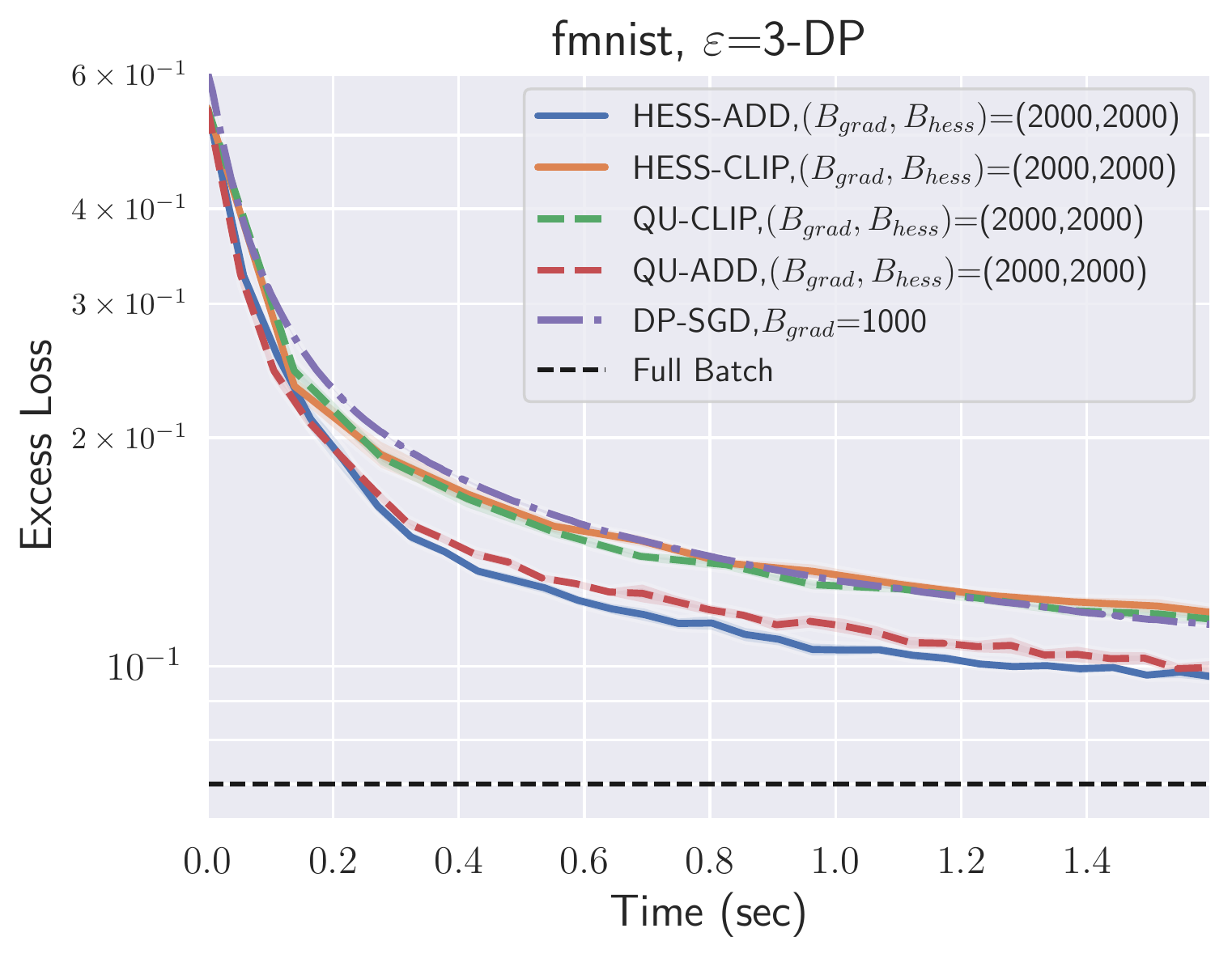}
    \end{subfigure}%
    ~ \hspace{-1.0em}
    \begin{subfigure}[t]{0.56\textwidth}
        \centering
          \includegraphics[width=\linewidth]{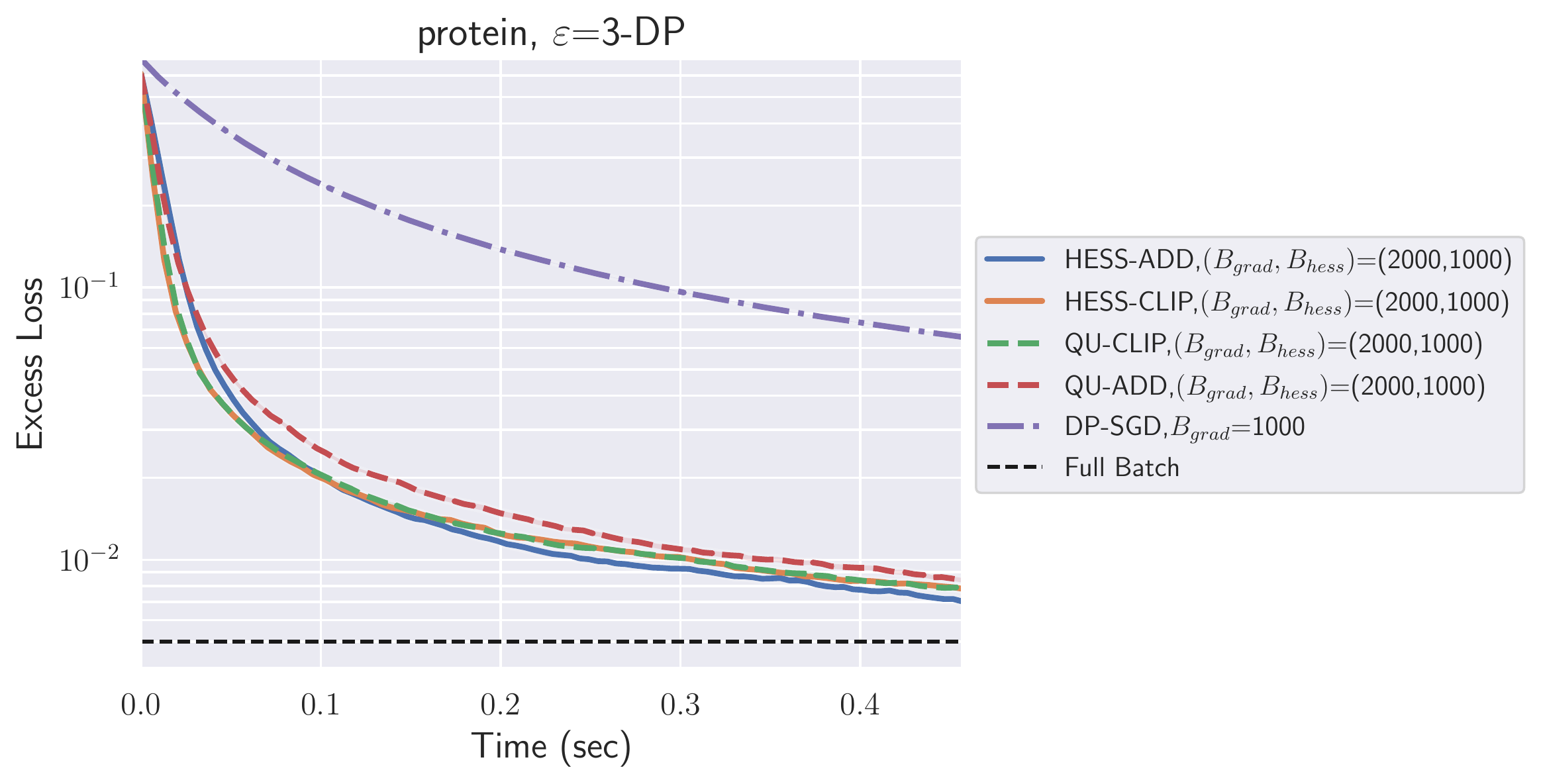}
    \end{subfigure}
    \caption{Minibatch variants for FMNIST and protein}
    \label{fig:fmnist-priv-utility}
\end{figure}

\subsection{Second Order Information vs Optimal Stepsize}

\begin{figure}[H]
    \centering
    \begin{subfigure}[t]{0.28\textwidth}
        \centering
      \includegraphics[width=\linewidth]{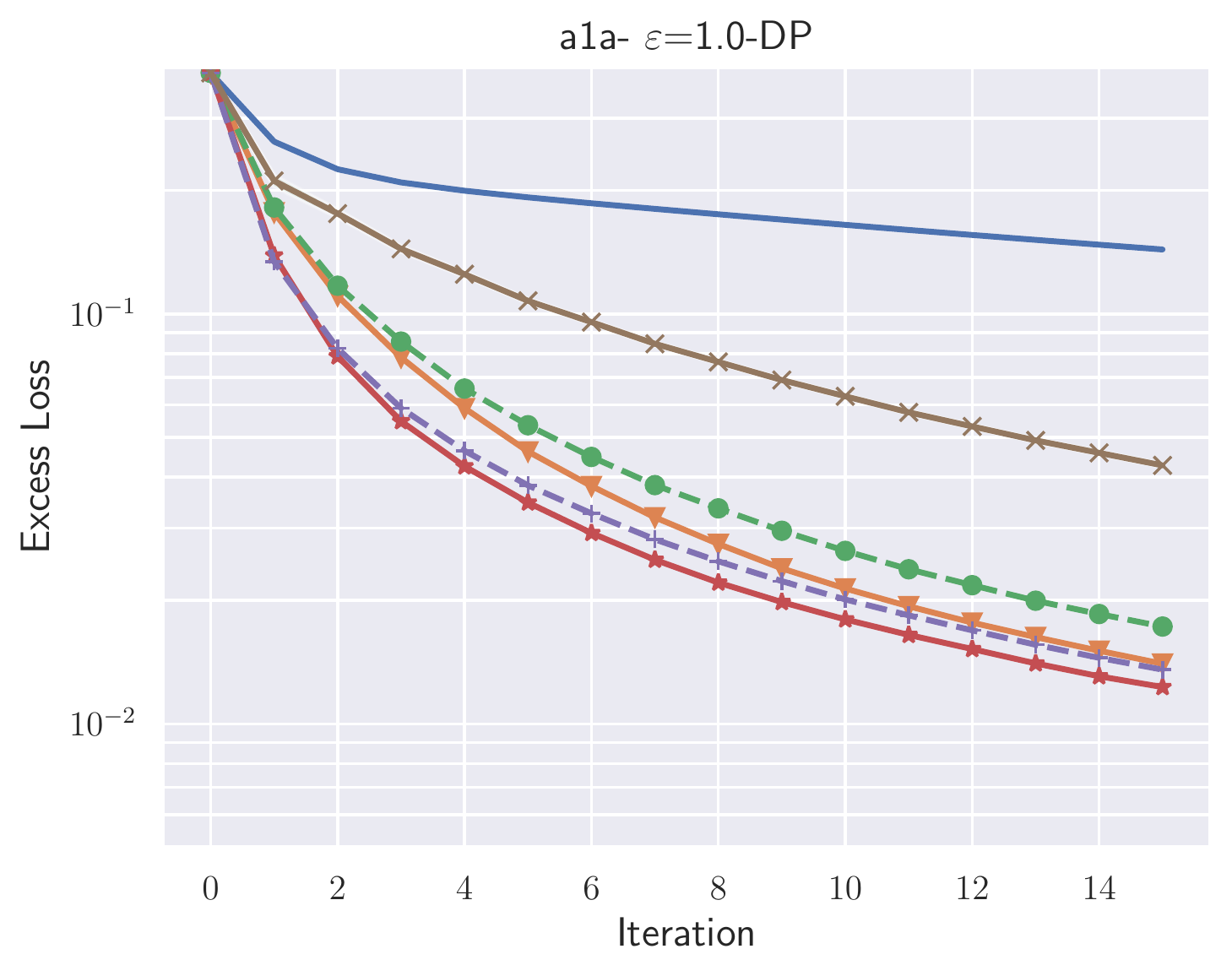}
    \end{subfigure}%
    ~ \hspace{-1.0em}
    \begin{subfigure}[t]{0.375\textwidth}
        \centering
      \includegraphics[width=\linewidth]{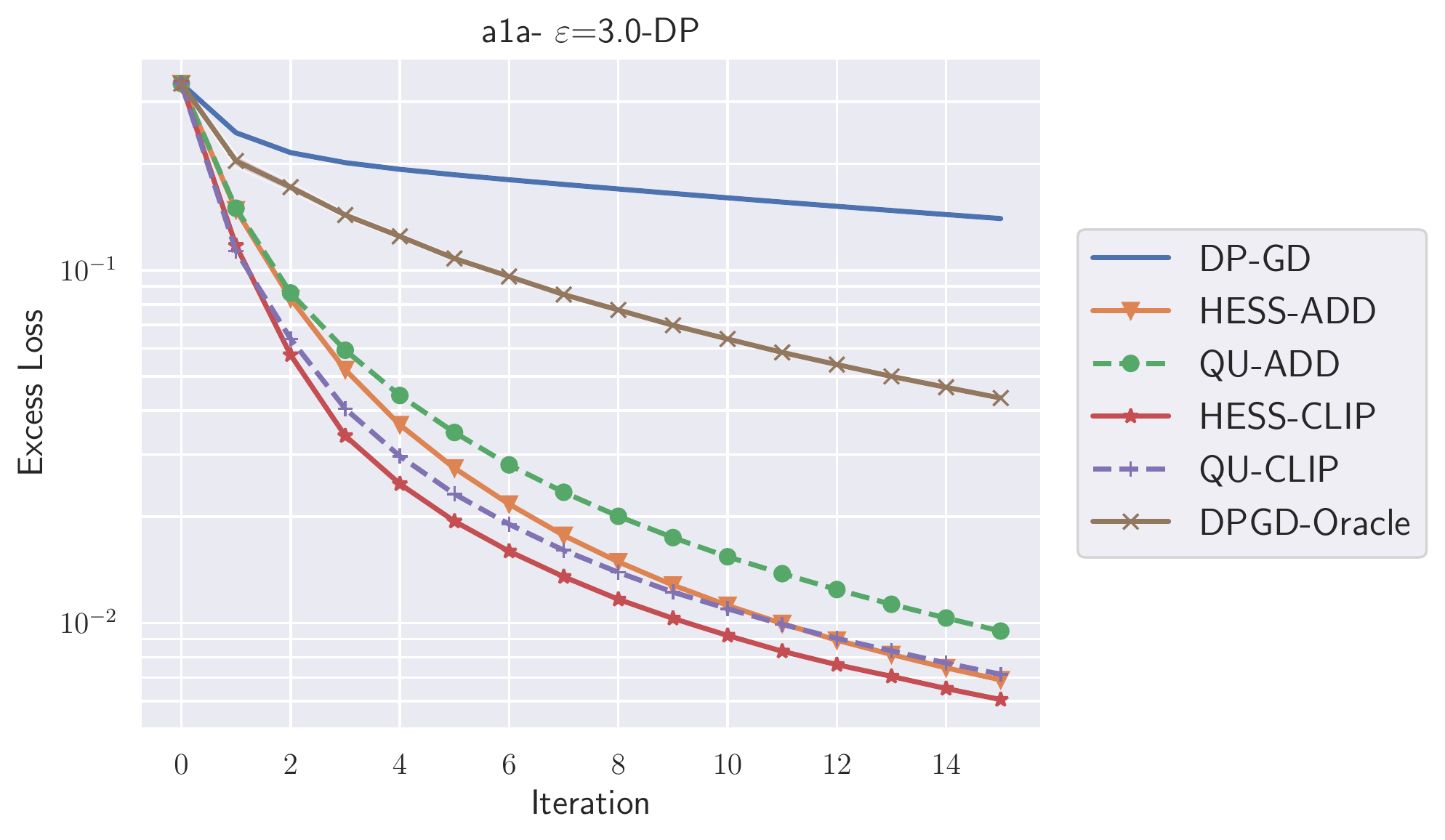}
    \end{subfigure}
    ~
         \begin{subfigure}[t]{0.28\textwidth}
        \centering
      \includegraphics[width=\linewidth]{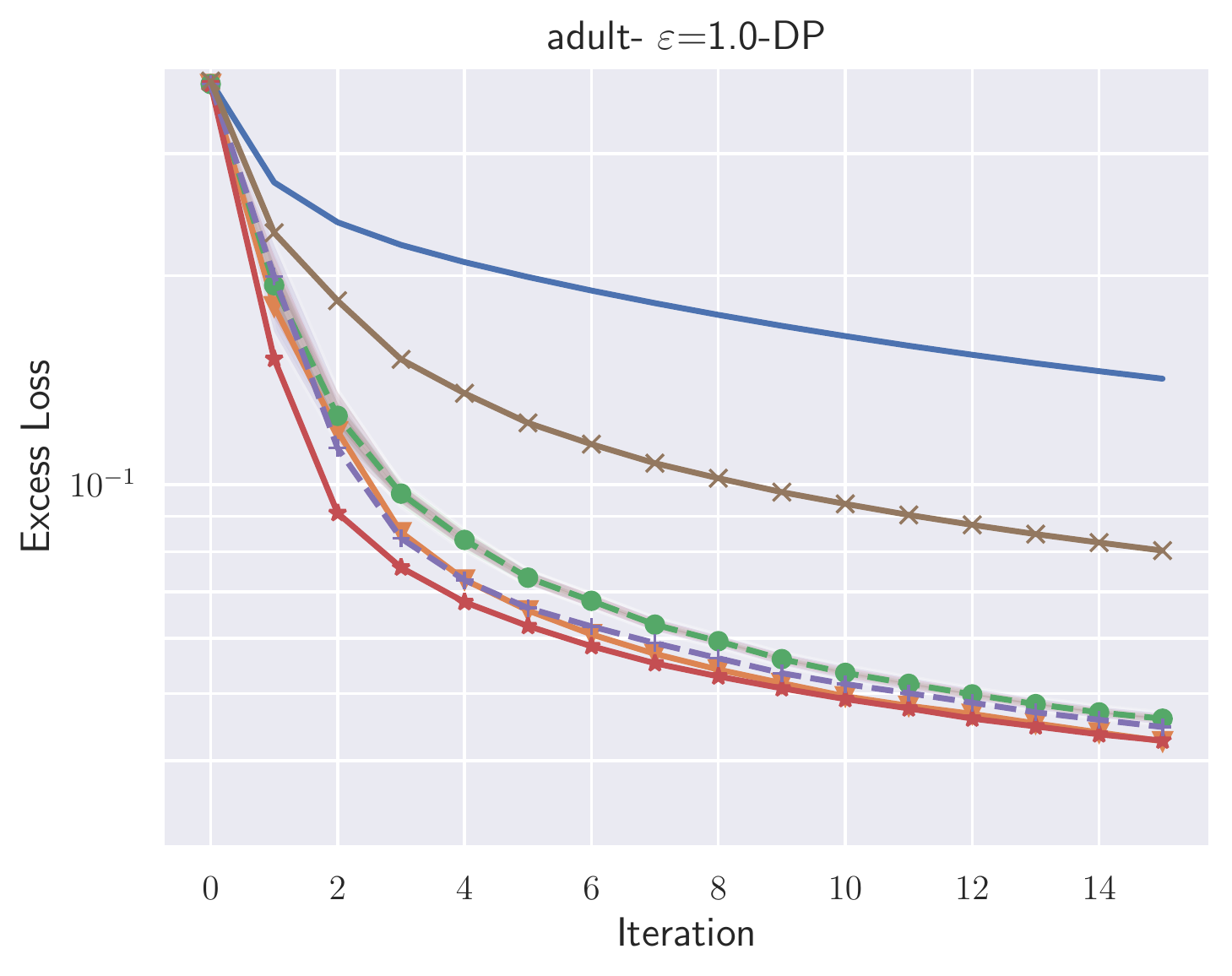}
    \end{subfigure}%
    ~ \hspace{-1.0em}
    \begin{subfigure}[t]{0.375\textwidth}
        \centering
      \includegraphics[width=\linewidth]{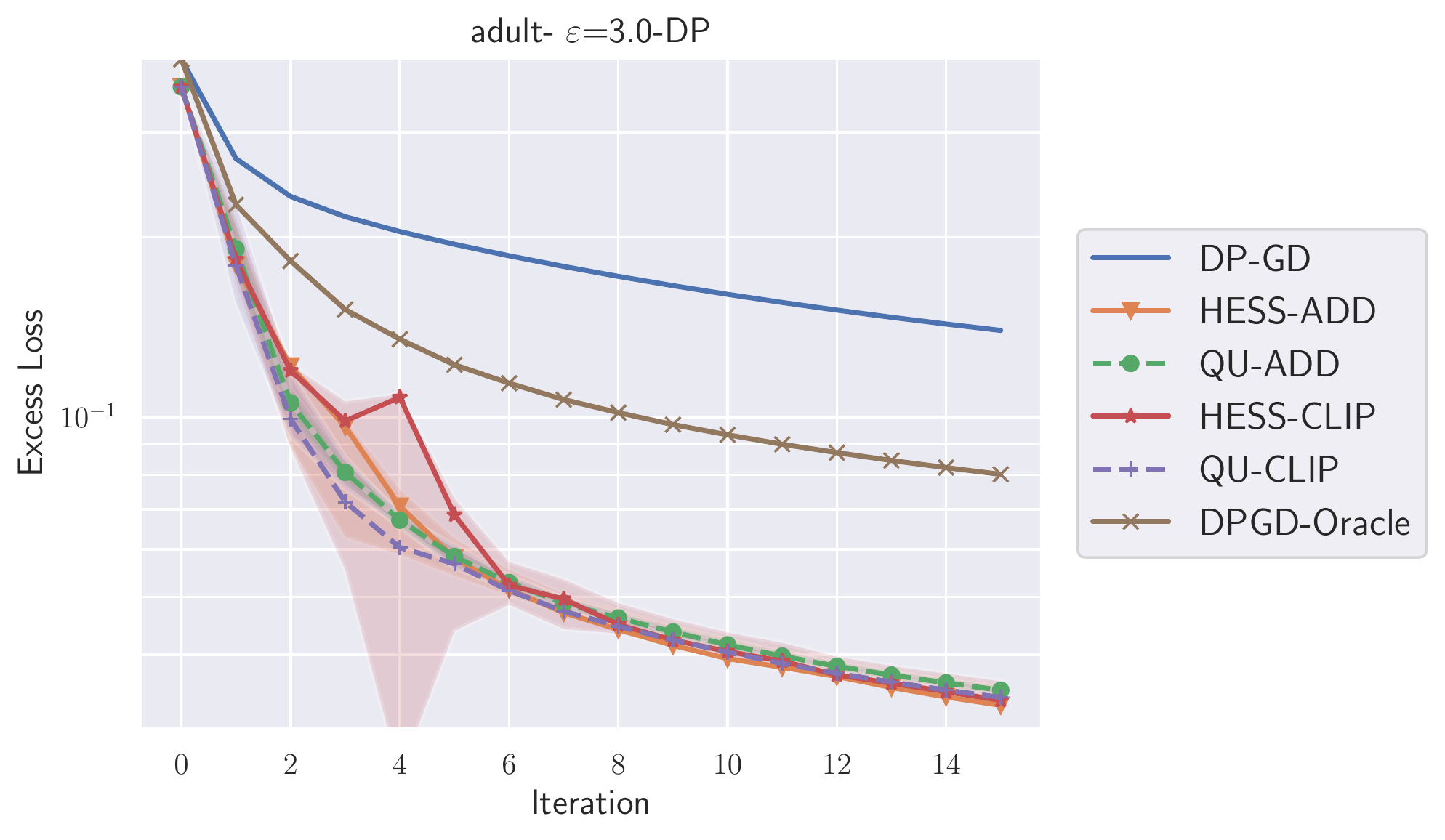}
    \end{subfigure}
          \begin{subfigure}[t]{0.28\textwidth}
        \centering
      \includegraphics[width=\linewidth]{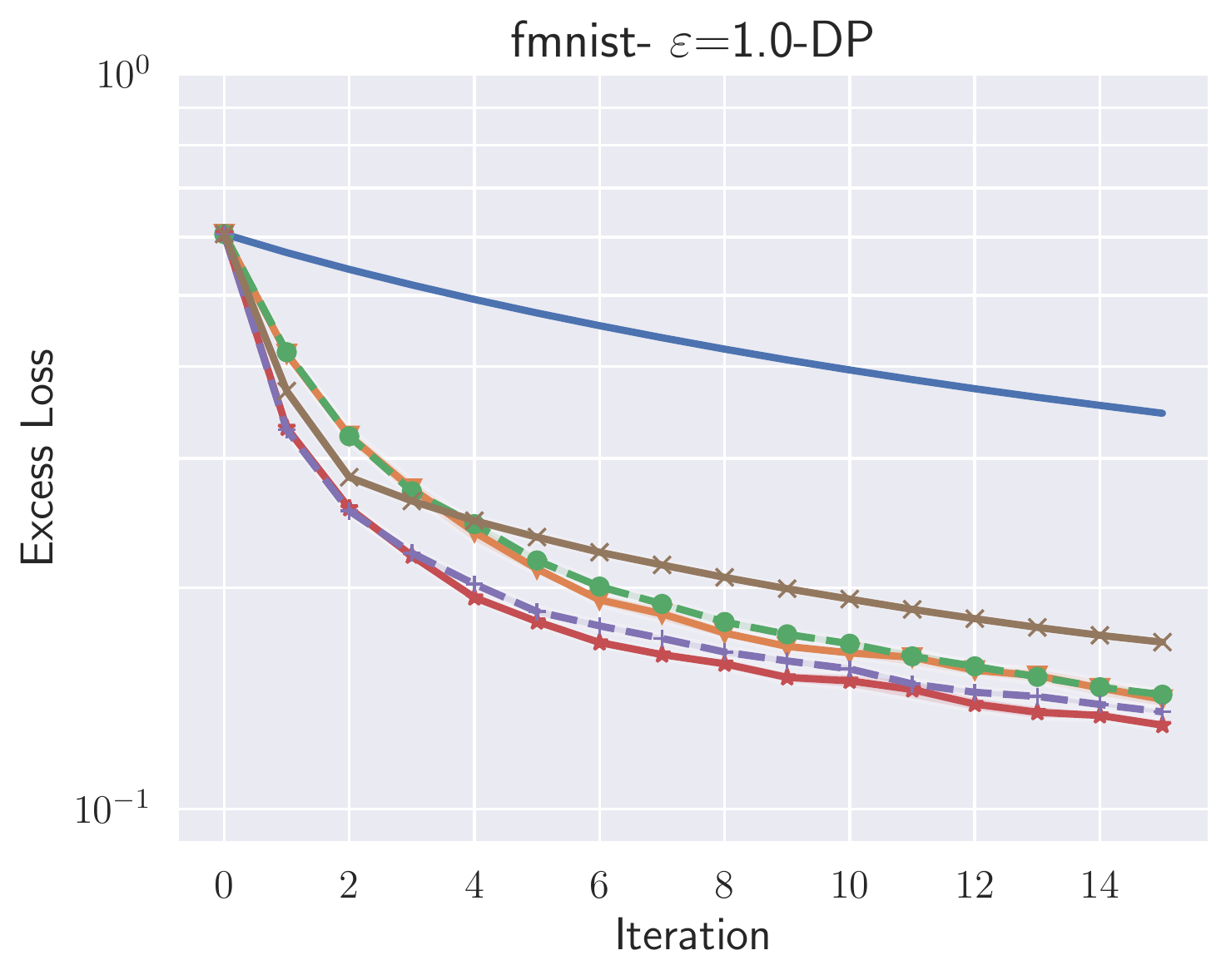}
    \end{subfigure}%
    ~ \hspace{-1.0em}
    \begin{subfigure}[t]{0.375\textwidth}
        \centering
      \includegraphics[width=\linewidth]{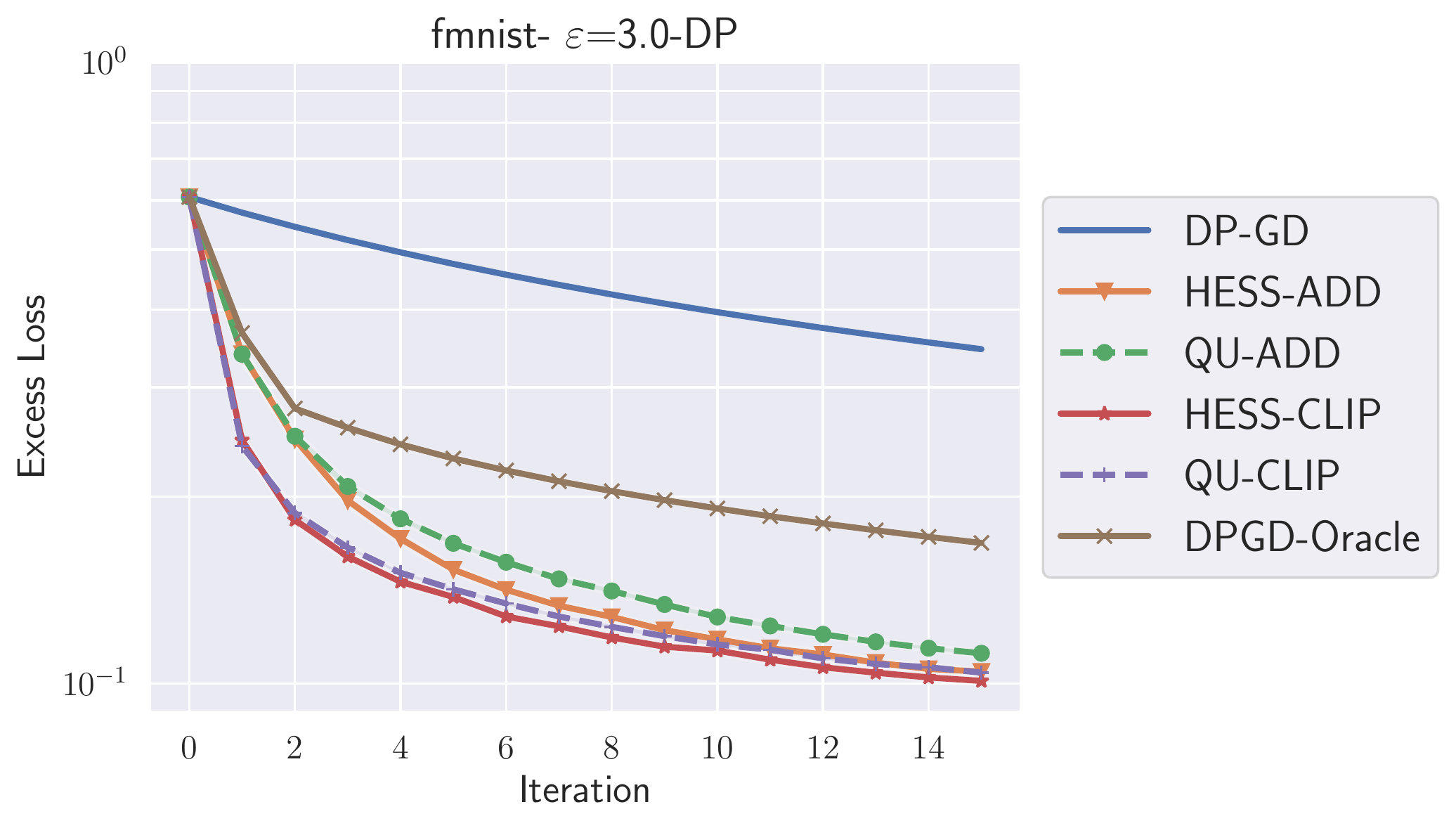}
    \end{subfigure}
          \begin{subfigure}[t]{0.28\textwidth}
        \centering
      \includegraphics[width=\linewidth]{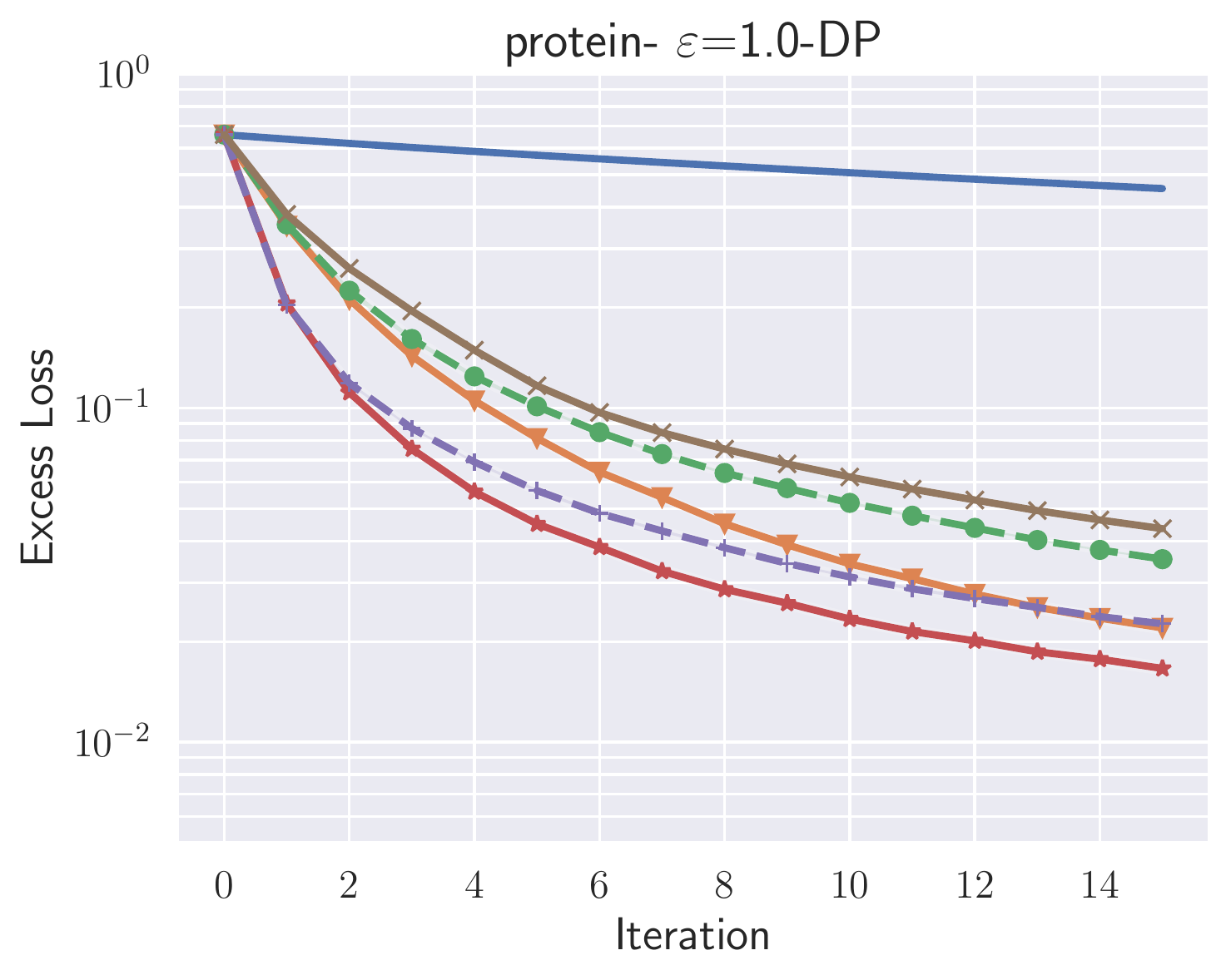}
    \end{subfigure}%
    ~ \hspace{-1.0em}
    \begin{subfigure}[t]{0.375\textwidth}
        \centering
      \includegraphics[width=\linewidth]{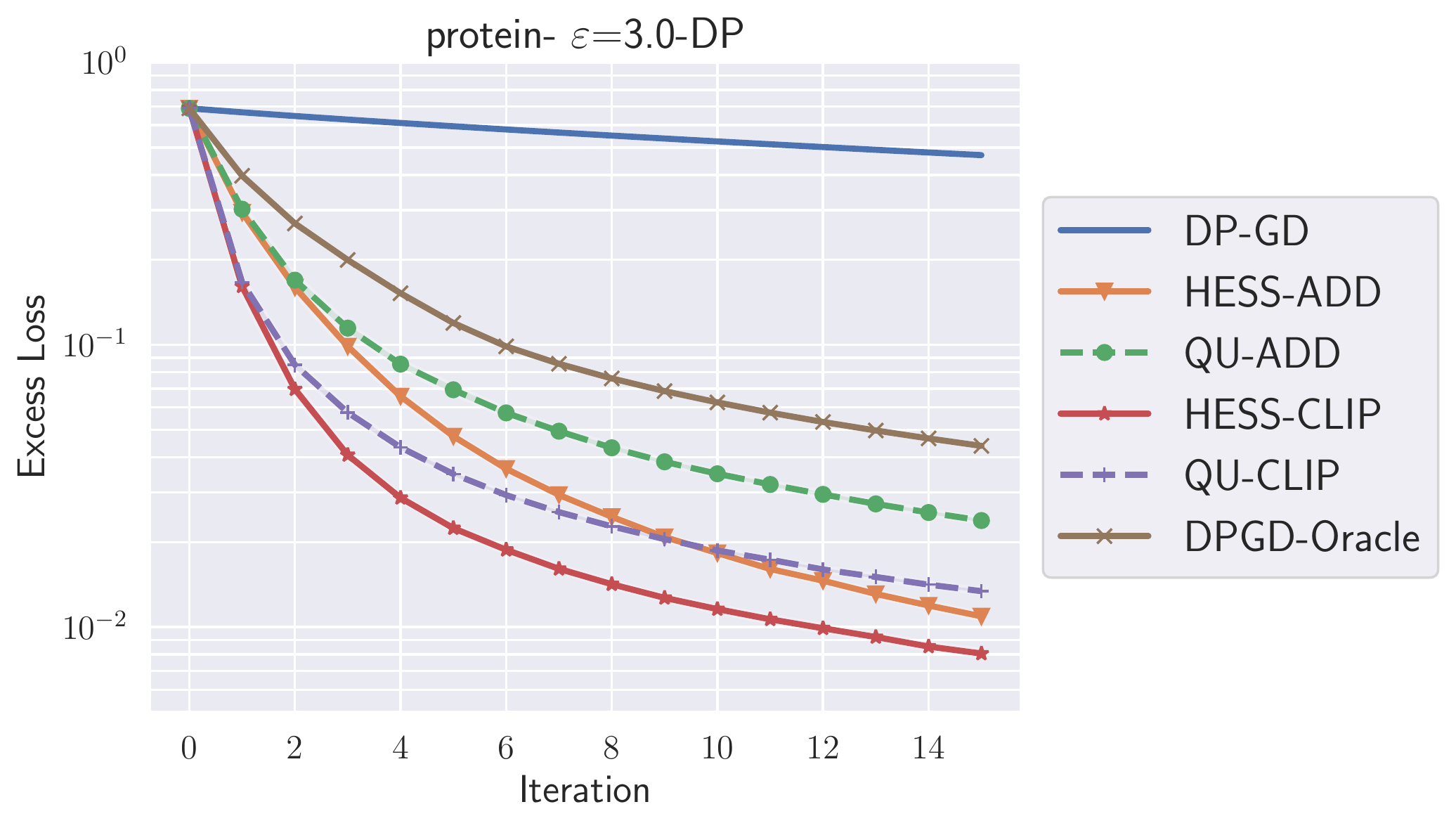}
    \end{subfigure}
     \caption{Comparison with DP-GD-Oracle (left $(\varepsilon,\delta) = (1,n^{-2})$, right = $(\varepsilon,\delta) =  (3,n^{-2})$)}
    \label{fig:convergence-2}
\end{figure}

\end{document}